\documentclass[aos]{imsart}

\usepackage{amsmath, amsthm, amsfonts,amssymb}
\usepackage{bm}
\usepackage{mathtools}
\usepackage{hyperref}
\hypersetup{colorlinks,linkcolor={blue},citecolor={blue},urlcolor={red}}  
\usepackage{fancyhdr}
\usepackage{graphicx, color}
\usepackage{epstopdf}
\usepackage{float}
\usepackage{hyperref,mathabx}
\usepackage{graphicx}
\usepackage{amsmath}
\usepackage{amsfonts}
\usepackage{amssymb}
\usepackage{amsthm}
\usepackage{bm,bbm}
\usepackage{mathrsfs}
\usepackage{dsfont}
\usepackage{booktabs}
\usepackage{tabularx}
\usepackage{subcaption}
\usepackage{epsfig,rotating,algorithm,algpseudocode}

\newcommand{\Second}{\textup{I}\!\textup{I}}

\theoremstyle{plain}
\newtheorem{thm}{Theorem}[section]

\newtheorem{lem}[thm]{Lemma}
\newtheorem{assu}[thm]{Assumption}
\newtheorem{prop}[thm]{Proposition}
\newtheorem{defn}[thm]{Definition}

\theoremstyle{remark}
\newtheorem{rem}[thm]{Remark}

\newcommand{\sfu}{\mathsf{u}}

\newcommand{\sfc}{\mathsf{c}}

\newcommand{\ub}{\mathbf{u}}

\newcommand{\OO}{\mathrm{O}}
\newcommand{\oo}{\mathrm{o}}
\newcommand{\zb}{\mathbf{z}}
\newcommand{\SNR}{\mathsf{SNR}}

\newcommand{\nb}{\mathbf{n}}
\newcommand{\wzb}{\widecheck{\mathbf{z}}}
\newcommand{\rO}{\mathrm{O}}

\newcommand{\Ab}{\mathbf{A}}

\newcommand{\fb}{\mathbf{f}}
\newcommand{\yb}{\mathbf{y}}

\newcommand{\xb}{\mathbf{x}}

\newcommand{\sw}{\mathsf{w}}

\newcommand{\sq}{\mathsf{q}}

\newcommand{\sK}{\mathsf{K}}
\newcommand{\bK}{\mathbf{K}}
\newcommand{\bA}{\mathbf{A}}

\newcommand{\si}{\mathsf{i}}

\begin{document}

\begin{frontmatter}

\title{Generalized Robust Adaptive-Bandwidth Multi-View Manifold Learning in High Dimensions with Noise}
\runtitle{GRAB-MDM}

\begin{aug}
\author[A]{\fnms{Xiucai}~\snm{Ding}\ead[label=e1]{xcading@ucdavis.edu}} \and
\author[B]{\fnms{Chao}~\snm{Shen}\ead[label=e2]{chaoshen.duke@gmail.com}} \and
\author[C]{\fnms{Hau-Tieng}~\snm{Wu}\ead[label=e3]{hw3635@nyu.edu}\orcid{0000-0002-0253-3156}}

\address[A]{Department of Statistics, University of California, Davis \printead[presep={ ,\ }]{e1}}
\address[B]{Amazon \printead[presep={ ,\ }]{e2}}
\address[C]{Courant Institute of Mathematical Sciences, New York University, New York, 10012, NY, USA \printead[presep={ ,\ }]{e3}}
\end{aug}

\begin{abstract} 
Multiview datasets are common in scientific and engineering applications, yet existing fusion methods offer limited theoretical guarantees, particularly in the presence of heterogeneous and high-dimensional noise. We propose Generalized Robust Adaptive-Bandwidth Multiview Diffusion Maps (GRAB-MDM), a new kernel-based diffusion geometry framework for integrating multiple noisy data sources. The key innovation of GRAB-MDM is a {view}-dependent bandwidth selection strategy that adapts to the geometry and noise level of each view, enabling a stable and principled construction of multiview diffusion operators. Under a common-manifold model, we establish asymptotic convergence results and show that the adaptive bandwidths lead to provably robust recovery of the shared intrinsic structure, even when noise levels and sensor dimensions differ across views. 
Numerical experiments demonstrate that GRAB-MDM significantly improves robustness and embedding quality compared with fixed-bandwidth and equal-bandwidth baselines, and usually outperform existing algorithms. The proposed framework offers a practical and theoretically grounded solution for multiview sensor fusion in high-dimensional noisy environments.  
\end{abstract}

\end{frontmatter}


\section{Introduction}

Advances in sensing technology now allow researchers to interrogate a system through heterogeneous sensors, producing multiple datasets with distinct structures and properties collected simultaneously. Following common terminology, we refer to the dataset from each sensor as a {\em view} of the system. Although it is often assumed that more data yield more information, this assumption may fail when additional data arise from heterogeneous sources. Integrating such data, variously termed sensor fusion \cite{gao2020survey,khaleghi2013multisensor,lahat2015multimodal,meng2020survey} or multi-view learning \cite{xu2013survey,yan2021deep,zhao2017multi}, remains a fundamental challenge.
To our knowledge, beyond naive concatenation, which increases dimensionality, or highly parameterized models, which risk overfitting and inefficiency, there is still no consensus or systematic framework for handling heterogeneous views, particularly under nonlinearity, high dimensionality, and noise. The central goal of sensor fusion is to combine information from multiple sensors into a unified representation that reflects the system's intrinsic latent structure. In practice, however, different views may contribute unequal or incomplete information, may exist in incompatible formats, and may be corrupted by distinct noise sources or artifacts. These issues collectively make effective sensor fusion a difficult and largely unresolved problem.

Despite these challenges, the practical importance of multi-view data integration has motivated substantial research activity. In the two-view setting, canonical correlation analysis (CCA) \cite{hotelling1936relations} and its variants, including kernel CCA (KCCA) \cite{bach2002kernel,hardoon2004ccaoverview} and NCCA \cite{michaeli2016nonparametric}, among others, have been extensively studied. These methods seek maximally correlated linear or nonlinear projections across two families of measurable functions, with the main differences lying in how these function classes and their associated correlations are defined. Notably, NCCA is closely related to the alternating diffusion map (ADM) \cite{lederman2018geometry,talmon2019latent}, a kernel-based method derived from diffusion maps \cite{coifman2006diffusion} to uncover shared latent manifold structure across two views. ADM, NCCA, and KCCA are all kernel-based approaches. 
For more than two views, several empirical generalizations of single-view algorithms have been proposed.
 \cite{lindenbaum2018_multiview_seismic, LINDENBAUM2020127,katz2019alternating,roman-messina2023_multiview_spectral_clustering} proposed multiview versions of the diffusion map, called multiview diffusion maps (MDM), which extends the original diffusion map (DM) algorithm \cite{coifman2006diffusion}. ADM is applied to fuse multiple views by concatenating all binary pairs \cite{katz2019alternating}. \cite{ji2022_multiview_lle_hyperspectral,rodosthenous2024multi, shen2013multiview_lle, zong2017_multi_manifold_nmf} developed multiview versions of locally linear embedding. Furthermore, \cite{do2021_generalization_tsne_umap, kanaan-izquierdo2019_multiview_package, rodosthenous2024multi, xie2011_m-sne} introduced multiview variants of t-Stochastic Neighbourhood Embedding, while \cite{do2021_generalization_tsne_umap} proposed a multiview extension of UMAP \cite{McInnes2018} and \cite{rodosthenous2024multi} developed a multiview version of Isomap. Comprehensive reviews are available in \cite{rodosthenous2024multi, xu2013survey}. 
A parallel line of work develops deep-learning-based architectures, and readers can find tons of them in the review articles \cite{yan2021deep, Stahlschmidt2022MultimodalReview}. 
While powerful, these models are typically applied in an ad hoc, black-box manner and lack rigorous theoretical guarantees, especially in high-dimensional or noisy regimes, raising concerns about interpretability and reproducibility.
For non-deep-learning approaches, theory remains limited once the number of views exceeds two. Aside from the analysis supporting ADM \cite{katz2019alternating}, rigorous results for multi-view fusion are scarce. Although the MDM framework \cite{LINDENBAUM2020127} is designed for general multi-view settings, existing theoretical work is restricted to the noise-free, two-view case.

In this paper, motivated by the practical challenges of analyzing high-dimensional, noisy multiview data and by the limited theoretical understanding of multiview fusion under noise, we introduce a new kernel-based algorithm, the \emph{Generalized Robust Adaptive-Bandwidth Multiview Diffusion Maps} (GRAB-MDM). Building on the multiview diffusion maps (MDM) framework, our method constructs a diffusion process using a novel data-driven bandwidth selection that propagates information across views to achieve robust sensor fusion. Specifically, given $\sK \geq 2$ views, GRAB-MDM constructs a $\sK \times \sK$ block kernel affinity matrix $\mathcal{K}$, where each block is of size $n \times n$ (see \eqref{eq_largekernelmatrix} below), designed to extract common information across views through a convolution-type procedure (see \eqref{eq_Kmultiplication} below). Based on this affinity matrix, we construct a normalized transition matrix (see  \eqref{eq_finaloperator} below).
Since the resulting matrix is of size $\sK n \times \sK n$, it admits a natural eigendecomposition that yields a joint spectral embedding of all samples across views. This construction also provides a flexible framework for defining different embedding. For example, in tasks such as simultaneous clustering, one may average embeddings across views to better exploit cross-view information rather than relying on view-specific embeddings alone (see \eqref{eq_useaverger}).

A critical component of GRAB-MDM is the choice of bandwidth, which must adapt to the geometry and noise level of each view. To address this, we propose a theoretically grounded, sensor-dependent bandwidth selection strategy (see Algorithm~\ref{al_bandwidthselection} below) that enables a principled multiview diffusion construction under noise. The procedure consists of two stages: first, a view-specific quantity is selected to capture signal characteristics unique to each sensor; second, a global scaling parameter is determined by leveraging cross-view information.

GRAB-MDM is theoretically sound and geometrically interpretable. Mathematically, we consider that each view is represented by a high-dimensional point cloud $\mathcal{Y}_\ell:=\{\yb_i^{\ell}\}_{i=1}^{n} \subset \mathbb{R}^{p_{\ell}}, 1 \leq \ell \leq \sK$, so that $n$ samples are observed through $\sK$ sensors of possibly different dimensions. We refer to $\yb_i^\ell$ as the {\em $i$-th sample recorded by the $\ell$-th sensor}. In practice, each observation is corrupted by noise, 
\begin{equation}\label{sec_datasetupsignalplusnoise}
\yb_i^{\ell}=\xb_i^\ell+\bm{\xi}_i^\ell, \ 1 \leq \ell \leq \sK, \ 1 \leq i \leq n\,,
\end{equation}
where $\xb_i^\ell$ denotes the clean signal and $\bm{\xi}_i^\ell$ the corresponding noise. Since sensors often have disparate physical characteristics, there is generally no simple relationship between $\bm{\xi}_i^\ell$ and $\bm{\xi}_i^{\ell'}$ when $\ell\neq \ell'$; both the noise level and its statistical structure may vary substantially across views. Our goal is to recover the features jointly shared by $\{\xb_i^\ell\}_{i=1}^n$ across all views, using only the noisy observations $\mathcal{Y}_\ell, 1 \leq \ell \leq \sK$.

To render the problem theoretically tractable, we assume that the clean signals $\{\xb_i^\ell\}_{i=1}^n$ lie on a common low-dimensional, smooth, closed latent manifold $\mathcal{M}$, and that each sensor observes data on a manifold diffeomorphic to $\mathcal{M}$. This setting is formalized by the generalized common manifold model (see Assumption~\ref{sec_modelassumption} below). Under this assumption, we show that GRAB-MDM converges to a limiting operator that involves a mixture of Laplace-Beltrami operators and lower order terms (see Definition~\ref{defn_VCDM} below) associated with view-specific embedded submanifolds. Bias and variance analyses are provided in Theorems~\ref{thm_bias_vmvl} and~\ref{thm_vmvlvarianceanalysis}, respectively.
In Section~\ref{sec_robustnessofresults} we further show that when the signal-to-noise ratio is properly defined and reasonably high, GRAB-MDM is robust to noise. As a corollary, our analyses provide a rigorous foundation for the MDM algorithm of \cite{LINDENBAUM2020127} when $\sK>2$.
Finally, while the richness of the limiting operators highlights the ability of GRAB-MDM to capture complex geometric structure under the common manifold model, these operators are not simple weighted Laplace-Beltrami operators. A full understanding of their spectral behavior would require substantially new analytical tools beyond existing work \cite{MR4279237,singer2017spectral,van2023weak,von2008consistency,shen2022scalability}. We leave a detailed investigation of this direction for future research.

This paper is organized as follows. Section \ref{sec_proposedlalgo} presents the GRAB-MDM algorithm, with emphasis on the sensor-dependent bandwidth selection. The manifold model underlying our analysis is introduced in Section \ref{Section common manifold model}. Sections \ref{section asymptotic analysis} and \ref{sec_robustnessofresults} provide the theoretical guarantees for GRAB-MDM under clean and noisy datasets, respectively. Numerical experiments are reported in Section \ref{sec_numericalresult}. The proofs and additional simulations are given in an online supplementary file.

We summarize the notation used in this paper. We use the notation  $\oslash$ for entrywise division.  For deterministic positive sequences $\{a_n\}$ and $\{b_n\},$ we write $a_n=\OO(b_n)$ if $a_n \leq C b_n$ for some positive constant $C>0.$ If both $a_n=\OO(b_n)$ and $b_n=\OO(a_n),$ we write $a_n \asymp b_n.$ We write $a_n=\oo(b_n)$ if $a_n \leq c_n b_n$ for some positive sequence $c_n \to 0.$ For a vector $\mathbf{a} = (a_1,\ldots,a_n)^\top \in \mathbb{R}^{n}$, we define its $\ell_p$ norm as $\| \mathbf{a} \|_p = \big(\sum_{i=1}^n |a_i|^p\big)^{1/p}$.  For a matrix $ \mathbf{A}=(a_{ij})\in \mathbb{R}^{n\times n}$,  we define its Frobenius norm as $\| \bold{A}\|_F = \sqrt{ \sum_{i=1}^{n}\sum_{j=1}^{n} a^2_{ij}}$, and its operator norm as $\| \bold{A} \| =\sup_{\|\bold{x}\|_2\le 1}\|\bold{A}\bold{x}\|_2 $. For any integer $n>0$, we denote the set $[n]=\{1,2,\ldots,n\}$. A random vector $\mathbf{g}$ is said to be sub-Gaussian if $\mathbb{E} \exp(\mathbf{a}^\top \mathbf{g}) \leq \exp\left( \| \mathbf{a} \|_2^2/2 \right)$, for any deterministic vector $\mathbf{a}.$ Throughout, $C,C_1,C_2,\ldots$ are universal constants independent of $n$, and can vary from line to line. Moreover, we denote $\mathbb{R}^+$ as the set of positive real numbers, $C(\iota(\mathcal{M}))$ as the space of continuously differentiable functions on $\iota(\mathcal{M})$, $L^{\infty}(\iota(\mathcal{M}))$ as the space of essentially bounded functions on $\iota(\mathcal{M})$, and $\mathrm{L}_2(\mu)$ as the space of square-integrable functions with respect to the measure $\mu$.
To simplify our probabilistic statements, we adopt the notion of stochastic domination, widely used in modern random matrix theory \cite{MR3699468} to concisely express bounds of the form ``$\mathsf{X}^{(n)}$ is bounded with high probability by $\mathsf{Y}^{(n)}$ up to small powers of $n$." Let
	$\mathsf{X}=\big\{\mathsf{X}^{(n)}(u):  n \in \mathbb{N}, \ u \in \mathsf{U}^{(n)}\big\}$ and  $\mathsf{Y}=\big\{\mathsf{Y}^{(n)}(u):  n \in \mathbb{N}, \ u \in \mathsf{U}^{(n)}\big\}$
	be two families of nonnegative random variables indexed by $n\in \mathbb{N}$ and a possibly $n$-dependent parameter set $\mathsf{U}^{(n)}$. We say that $\mathsf{X}$ is {\em stochastically dominated} by $\mathsf{Y}$, uniformly in $u$, if for every small $\upsilon>0$ and large $ D>0$, there exists $n_0(\upsilon, D)\in \mathbb{N}$ such that $\sup_{u \in \mathsf{U}^{(n)}} \mathbb{P} \Big( \mathsf{X}^{(n)}(u)>n^{\upsilon}\mathsf{Y}^{(n)}(u) \Big) \leq n^{- D}$, for all $n \geq  n_0(\upsilon, D)$. An $n$-dependent event $\Omega \equiv \Omega(n)$ is said to hold {\em with high probability}, if for any large $D>1$, there exists $n_0=n_0(D)>0$ so that $\mathbb{P}(\Omega) \geq 1-n^{-D}$ for all $n \geq n_0.$  When no ambiguity arises, we use the standard notation $\mathsf{X}=\OO_{\prec}(\mathsf{Y})$, $\mathsf{X} \prec \mathsf{Y}$, or $\mathsf{Y}\succ \mathsf{X}$ to denote stochastically domination.

\section{The proposed algorithm and bandwidth selection procedure}\label{sec_proposedlalgo}

Motivated by MDM \cite{LINDENBAUM2020127} and by the need to handle noisy data in practical applications, we now introduce our proposed algorithm, the {\em Generalized Robust Adaptive-Bandwidth Multiview Diffusion Maps} (GRAB-MDM). Algorithmically, GRAB-MDM differs from MDM primarily in its principled selection of sensor-dependent bandwidths and its embedding strategy for multiview datasets.

\subsection{The proposed GRAB-MDM algorithm} \label{al_VCDM}
For each view $1 \leq \ell \leq \sK$ and  some  bandwidth $\epsilon_\ell$ selected from Algorithm \ref{al_bandwidthselection} below in Section \ref{sec_bandwidthselection}, we define $\bK^{\ell} \in \mathbb{R}^{n \times n}$ so that for $1 \leq i, j \leq n$
\begin{equation}\label{eq_Kmultiplicationentry}
\mathbf{K}^{\ell}(i,j)=K \left(\frac{\| \yb^{\ell}_i-\yb^{\ell}_j \|_2^2}{\epsilon_{\ell}} \right), \ 1 \leq i,j \leq n\,,
\end{equation}
where $K:[0,\infty)\to \mathbb{R}_+\cup\{0\}$ is a non-degenerate kernel with sufficient regularity and decay satisfying some mild regularity conditions (see Assumption \ref{assum_main} below for more details); for example, a Gaussian kernel. 
Based on (\ref{eq_Kmultiplicationentry}), for $1 \leq  \ell_1,  \ell_2 \leq \sK,$ we define 
\begin{equation}\label{eq_Kmultiplication}
\mathbf{K}^{\ell_1, \ell_2}=\mathbf{K}^{\ell_1} \mathbf{K}^{\ell_2}\,.
\end{equation}
Using (\ref{eq_Kmultiplication}), we can further define the block-wise asymmetric kernel affinity matrix $\mathcal{K} \in \mathbb{R}^{n \sK \times n\sK }$ as follows
\begin{equation}\label{eq_largekernelmatrix}
\mathcal{K}:=
\begin{pmatrix}
\mathbf{0}_{n \times n} & \bK^{1,2} & \bK^{1,3} & \cdots & \bK^{1, \sK} \\
\bK^{2,1} & \mathbf{0}_{n \times n} & \bK^{2,3} & \cdots & \bK^{2, \sK}\\ 
\bK^{3,1} & \bK^{3,2} & \mathbf{0}_{n \times n} &  \cdots & \bK^{3, \sK}\\
\vdots & \vdots & \vdots & \cdots & \vdots \\
\bK^{\sK, 1} & \bK^{\sK, 2} & \bK^{\sK, 3} & \cdots &  \mathbf{0}_{n \times n}
\end{pmatrix}.
\end{equation}
Denote the diagonal degree matrix $\mathcal{D}\in \mathbb{R}^{n \sK \times n \sK}$ that
\begin{equation}\label{eq_diagonalffinity}
\mathcal{D}(i,i)=\sum_{j=1}^{n \sK} \mathcal{K}(i,j)\,, 
\end{equation}
which is nondegenerate by the construction of $\mathcal{K}$.
Then define the transition matrix as 
\begin{equation}\label{eq_finaloperator}
\mathcal{A}=\mathcal{D}^{-1} \mathcal{K}.
\end{equation} 
Since $\mathcal{A}$ is similar to a symmetric matrix $\mathcal{D}^{-1/2} \mathcal{K}\mathcal{D}^{-1/2}$, we have eigenvalue decomposition of $\mathcal{A}$. Denote $\mathcal{A}$'s eigenvalues and right eigenvectors as $\{\eta_i\}_{i=1}^{n \sK}$ and $\{\ub_i\}_{i=1}^{n \sK}$, where $1=\eta_1 \geq \eta_2 \geq \cdots \geq \eta_{n \sK}\geq -1$. The eigenvalue bound is due to the normalization and the spectral norm controlled by the $\ell^\infty$ norm of $\mathcal{A}$.  
Fix $1 \leq m<n\sK$, $t\geq 0$ and construct  
\begin{equation}\label{al_VCDM1}
Q_{m,t}= \texttt{diag}(\eta^t_2,\ldots,\eta^t_{m+1}) [\ub_2,\ldots,\ub_{m+1}]^\top \in \mathbb{R}^{ m \times n \sK}\,.
\end{equation} 
Note that all entries of $\ub_1$ are the same due to the construction of $\mathcal{A}$. 

In practice, the data embedding is constructed from (\ref{al_VCDM1}) depending on the downstream tasks. If one aims to obtain distinct embeddings for each view, a natural approach is to embed 
\[
\Phi^{(\ell)}_{m,t}:\,\yb^{\ell}_j\to Q_{m,t}e_{j+(\ell-1)n}\in \mathbb{R}^{m}
\] 
for the $\ell$-th view, where $e_j\in \mathbb{R}^{n\sK}$ is the unit vector with the $j$-th entry equal to $1$. When $m=n$, this embedding is reduced to \cite[(21)]{LINDENBAUM2020127}. 
When only a single embedding for each sample from all across views is needed, one can consider a joint embedding by mapping 
\[
\Phi_{m,t}^{(\texttt{joint})}:\,\{\yb^{\ell}_j\}_{\ell=1}^{\sK}\to [Q_{m,t}e_{j}; Q_{m,t}e_{j+n};\ldots;Q_{m,t}e_{j+(\sK-1)n}]\in \mathbb{R}^{m\sK}\,.
\] 
Similarly, when $m=n$, this embedding is reduced to \cite[(23)]{LINDENBAUM2020127}. See Sections \ref{section asymptotic analysis} and \ref{sec_robustnessofresults} for a theoretical justification of these embeddings when $\sK\geq 2$.
In this paper, motivated by the theoretical analysis in Sections \ref{section asymptotic analysis} and \ref{sec_robustnessofresults}, we propose to also consider 
\begin{align}\label{eq_useaverger}
\Phi_{m,t}:\,\{\yb^{\ell}_j\}_{\ell=1}^{\sK}\to \frac{1}{\sK}\sum_{\ell=1}^{\sK} Q_{m,t}e_{j+(\ell-1)n}\in \mathbb{R}^{m},
\end{align}
which averages the embeddings across all views. 
With the chosen embedding, we can proceed to downstream tasks such as clustering and visualization. These tasks typically rely on the Euclidean distance between embedded points, which in the literature is often referred to as the {\em diffusion distance}. This terminology arises because the eigenvalues and eigenvectors used to construct the embedding are derived from a diffusion process defined on the dataset. To formalize this relationship, denote the associated diffusion distances as 
\begin{align*}
d_{m,t}^{(\ell)}(\yb^{\ell}_i,\,\yb^{\ell}_j) &\,:= \left\|\Phi_{m,t}^{(\ell)}(\yb^{\ell}_i)-\Phi_{m,t}^{(\ell)}(\yb^{\ell}_j)\right\|_{\mathbb{R}^{m}},\ \ \ell=1,\ldots,\sK\,,\\
d_{m,t}^{(\texttt{joint})}(\{\yb^{\ell}_i\}_{\ell=1}^{\sK},\,\{\yb^{\ell}_j\}_{\ell=1}^{\sK}) &\,:= \left\|\Phi_{m,t}^{(\texttt{joint})}(\{\yb^{\ell}_i\}_{\ell=1}^{\sK})-\Phi_{m,t}^{(\texttt{joint})}(\{\yb^{\ell}_j\}_{\ell=1}^{\sK})\right\|_{\mathbb{R}^{m\sK}}\,,\\
d_{m,t}(\{\yb^{\ell}_i\}_{\ell=1}^{\sK},\,\{\yb^{\ell}_j\}_{\ell=1}^{\sK}) &\, := \left\|\Phi_{m,t}(\{\yb^{\ell}_i\}_{\ell=1}^{\sK})-\Phi_{m,t}(\{\yb^{\ell}_j\}_{\ell=1}^{\sK})\right\|_{\mathbb{R}^m}\,,
\end{align*}
respectively, where $d_{m,t}^{(\ell)}$ is view-dependent. Note that, $d_{n-1,t}^{(\texttt{joint})}(\{\yb^{\ell}_i\}_{\ell=1}^{\sK},\,\{\yb^{\ell}_j\}_{\ell=1}^{\sK}) = \|e_i^\top \mathcal{A}^t-e_j^\top \mathcal{A}^t\|_2$, and similar expressions hold for the other distances. Here, $\mathcal{A}$ is a row-stochastic matrix describing a diffusion process on the dataset. At a high-level, $\mathcal{A}$ can be interpreted as the transition matrix of a Markov chain defined on the union of all views such that for each observation $\yb_i^{\ell}, 1 \leq i \leq n, 1 \leq \ell \leq \sK,$ the one-step transition probability is given by $
\mathbb{P}(\yb_i^{\ell_1}, \yb_j^{\ell_2})= \frac{\bK^{\ell_1 \ell_2}(i,j)}{\mathcal{D}(\si,\si)}$, where $\si=(\ell_1-1)n +i$ and $\ell_1 \neq \ell_2$. In this sense, the terminology diffusion distance is well justified.

In general, the choice of $t$ and $m$ in (\ref{al_VCDM1}) depends on the specific downstream task. To keep the algorithm automatic, we follow standard practice in diffusion geometry and fix $t=1$. For selecting $m$, we extend the elbow method, classically used in PCA via scree plots \cite{Jolliffe2002PCA}, by adaptively incorporating the eigen-ratios of $\{\eta_i\}$; see Appendix \ref{appendix_cchosen} of our supplement for the detailed algorithm. This strategy yields robust performance in the numerical experiments of Section \ref{sec_numericalresult}. A comprehensive study of optimal choices of $t$ and $m$, and more generally of embedding designs tailored to particular applications, is beyond the scope of this work and remains an interesting direction for future research.

\begin{rem}
In the two-view case, $\sK=2$, the ADM algorithm has been analyzed in \cite{lederman2018geometry,talmon2019latent,ding2021kernel}, while \cite{michaeli2016nonparametric} studied NCCA. Although both methods superficially resemble GRAB-MDM when $\sK=2$, they differ in an essential way: rather than normalizing after multiplying the kernel affinity matrices as in (\ref{eq_Kmultiplicationentry}), ADM and NCCA normalize each kernel separately before multiplication, leading to distinct operators and interpretations.
Beyond the two-view setting, canonical extensions of these approaches remain unclear. With the exception of the recent effort in \cite{katz2019alternating}, there is still no principled framework for generalizing these algorithms to settings with more than two views.
\end{rem}

\begin{rem}
Heuristically, the construction of $\mathcal{A}$ forces the Markov chain to transition only across different views, prohibiting within-view moves. This design mirrors the strategy in \cite{MR3449771}, where suppressing lazy walks improves robustness to high-dimensional noise in diffusion maps. As shown in Section \ref{sec_robustnessofresults}, the same principle applies here and significantly enhances noise robustness.

On the other hand, if one inserts the blocks $\bK^{\ell,\ell}$ back into the diagonal of $\mathcal{K},$ it would yield a matrix $\mathcal{K}^\circ$, with associated degree matrix $\mathcal{D}^\circ$ and diffusion matrix $\mathcal{A}^\circ$. The rank of $\mathcal{K}^\circ$ is at most $n$ and admits the outer-product representation $\begin{bmatrix}\bK^1\\\vdots\\\bK^{\sK}\end{bmatrix}\in \mathbb{R}^{n\sK\times n}$. Consequently, the eigenvalue decomposition of $\mathcal{A}^\circ$ can be simplified by computing the singular value decomposition (SVD) of $\mathcal{D}^{\circ -1/2} \begin{bmatrix}\bK^1\\ \vdots\\ \bK^{\sK}\end{bmatrix} =: \begin{bmatrix}\bA^1\\ \vdots\\ \bA^{\sK}\end{bmatrix}$, which is computationally more efficient. Furthermore, this SVD reduces to the SVD of each $n \times n$ block $\bA^\ell$. Specifically, if the SVD of $\bA^\ell$ is $U_\ell D_\ell V_\ell^\top$, then the SVD of $\begin{bmatrix}\bA^1; \ldots; \bA^{\sK}\end{bmatrix}$ can be written as
 $\begin{bmatrix} U_1 & 0 & \ldots & 0\\ 0 & U_2 &\ldots & 0\\ \vdots & \vdots & \vdots & \vdots \\ 0 & 0 & \ldots & U_{\sK} \end{bmatrix}\begin{bmatrix} D_1 & 0 & \ldots & 0\\ 0 & D_2 &\ldots & 0\\ \vdots & \vdots & \vdots & \vdots \\ 0 & 0 & \ldots & D_{\sK} \end{bmatrix} \begin{bmatrix}V_1^\top \\ \vdots \\ V_{\sK}^\top \end{bmatrix}$. For each view, the SVD of $\bA^\ell$ differs from the eigenvalue decomposition of the random walk matrix typically used in DM. Because the shared information is primarily captured in $\mathcal{D}^\circ$, this relationship suggests that inserting the diagonal blocks, or equivalently, allowing view-level lazy walks, may diminish the effectiveness of information fusion, even for clean datasets. A systematic quantification of this phenomenon lies beyond the scope of the present paper.
\end{rem}

\begin{rem}
For notional simplicity, we use the same kernel function $K(\cdot)$ in (\ref{eq_Kmultiplicationentry}) for all views $1 \leq \ell \leq \sK.$ Our results and algorithms extend straightforwardly to view-specific kernels $K_\ell$, provided they each satisfy Assumption \ref{assum_main}. Although such a generalization may offer empirical flexibility, it adds no conceptual insight to the analysis, and we omit it for brevity.
\end{rem}

\subsection{The bandwidth selection procedure in GRAB-MDM}\label{sec_bandwidthselection}

A key component of GRAB-MDM is the bandwidth selection procedure in Algorithm \ref{al_bandwidthselection}, designed to recover the underlying geometric structure of the clean samples $\{\xb_i^\ell\}$ from the noisy observations $\{\yb_i^\ell\}$. A theoretical justification is provided in Section \ref{sec_badnwidthselctionalgorithm}. The resulting bandwidth, $\epsilon_\ell := \sfc h_\ell$ in (\ref{eq_finalbandwidthform}), decomposes into a view-specific term $h_\ell$ and a global scaling factor $\sfc$.

The construction of $h_\ell$ follows \cite{9927456}, where it is introduced as a data-adaptive local scale parameter. Specifically, $h_\ell$ is defined as the $\omega_\ell$-percentile of the empirical cumulative distribution of the pairwise squared distances in the $\ell$-th view, with $\omega_\ell$ a tunable hyperparameter. Theoretically, any fixed $\omega_\ell\in(0,1)$ yields the same asymptotic behavior. Practically, to reduce tuning and improve stability, we recommend the resampling strategy of \cite{9927456}, detailed in Appendix \ref{appendix_cchosen} of our supplement. In this way, $h_\ell$ captures the intrinsic geometry of each noisy point cloud $\mathcal{Y}_\ell$, $1 \leq \ell \leq \sK$.

The global scaling factor $\sfc$, in contrast, reflects geometry shared across all views and its order is theoretically linked to the underlying geometry. As shown in Section \ref{sec_robustnessofalgorithms}, $\sfc$ must lie in an appropriate range: excessively large or small values destabilize the algorithm in Section \ref{al_VCDM}, particularly the spectral properties of the transition matrix in (\ref{eq_finaloperator}). When $\sfc$ is appropriately selected, however, the spectrum is stable. Algorithm \ref{al_bandwidthselection} determines $\sfc$ via a grid search using the spectral distance as in (\ref{eq_spectraldistance}) \cite{gu2015spectral,jovanovic2012spectral,lin2022graph}, which is small when the corresponding $\sfc$ values are close. A hyperparameter $\delta$ controls the stability threshold needs to be chosen. For implementation, we recommend the calibration procedure of \cite{9779233}; see Appendix \ref{appendix_cchosen} for details.

Section \ref{sec_numericalresult} presents extensive numerical experiments. Across all scenarios, particularly in high-noise regimes, GRAB-MDM consistently outperforms ten competing methods, demonstrating substantial robustness.

\begin{algorithm}[!ht]
	\caption{The bandwidth selection procedure} \label{al_bandwidthselection}
	\begin{algorithmic}
		\State {\bf Input:} Observed samples $\mathcal{Y}_\ell$ and some sequence of percentiles $\omega_\ell \in (0,1),  1 \leq  \ell  \leq \sK,$ chosen according to Appendix \ref{appendix_cchosen} of the supplement.  
	\vspace{5pt}	
	
\State {\bf View-specific signal and structure extraction}		
		\State  For $1 \leq \ell \leq \sK,$ we choose $h_\ell$ according to 
\begin{equation}\label{VMVL_bandwidth}
\int_0^{h_\ell} \mathrm{d} \mu_\ell(x)=\omega_\ell,
\end{equation}
where $\mu_{\ell}(x)$ is the empirical cumulative distribution function (ECDF) of $\{\| \yb^{\ell}_i-\yb^{\ell}_j \|_2^2\}_{1 \leq i \neq j \leq n}.$ 	
	\vspace{5pt}	
	
\State {\bf Global scaling factor construction}	
		\State 1. Given some large integer $N$ and a sequence of some small values of grids $0<\sfc_1<\sfc_2 < \cdots < \sfc_N,$  for each view $1 \leq \ell \leq \sK,$ we set the bandwidth $\epsilon_\ell(i)=\sfc_i h_\ell.$ Using the above bandwidths, we construct a sequence of transition matrices $\mathcal{A}_i, 1 \leq i \leq N,$ as in (\ref{eq_finaloperator}).  
		\State 2. For $1 \leq i \leq N,$ let the eigenvalues of  $\mathcal{A}_i$ be $\{\lambda_k(i)\}.$ For $1 \leq i \neq j \leq N,$ we construct the distance measurement between $\mathcal{A}(i)$ and $\mathcal{A}(j)$ as follows  
		\begin{equation}\label{eq_spectraldistance} 
\mathsf{d}_{ij}=\sum_{k=1}^{n} (\lambda_k(i)-\lambda_k(j))^2. 
		\end{equation}
\State 3. For each $1 \leq i \leq N$ and some small threshold $\delta$, we denote the index set
\begin{equation*}
S(i) \equiv S(i,\delta):=\left\{ 1 \leq j \neq i \leq N| \mathsf{d}_{ij}<\delta \right\}, 
\end{equation*}
and its cardinality as $|S(i)|.$  Let $\alpha:=\max\{\arg \max_{1 \leq i \leq N} |S(i)|\},$ and choose 
\begin{equation}\label{eq_cchosen}
\sfc=\sfc_\alpha. 
\end{equation}
		\State {\bf Output:}  The bandwidths $\epsilon_\ell$ for $1 \leq  \ell  \leq \sK$ are chosen as follows 
		\begin{equation}\label{eq_finalbandwidthform}
		\epsilon_\ell=\sfc h_\ell. 
		\end{equation}
	\end{algorithmic}
\end{algorithm}

\begin{rem}
In the MDM algorithm \cite{LINDENBAUM2020127}, which operates within the Gaussian-kernel framework and assumes clean data, the recommended parameter choices are selected from a predefined set to accentuate the bell-shape behavior of the kernel, extending the heuristic of \cite{singer2009detecting}. To our knowledge, this strategy lacks theoretical justification. Moreover, because MDM does not account for noise in its design, it is not inherently robust in high-dimensional noisy settings and may overemphasize pairwise view interactions while failing to capture more informative multiview relationships.
\end{rem}

\begin{rem}
Bandwidth selection is a critical challenge in general statistical tasks. To our knowledge, in the single-view case ($\sK = 1$), most existing approaches rely on ad hoc choices and implicitly assume noise-free data, making them ill-suited for high-dimensional or noisy environments, apart from a few recent advances \cite{9927456}. The pairwise-distance quantile strategy of \cite{9927456} provides a principled solution in this regime, enabling DM-based methods to operate reliably under a signal-plus-noise model with theoretical guarantees.  
\end{rem}

\section{The common manifold model and limiting operators}\label{Section common manifold model}

In this section, we introduce a mathematical model for the clean signals $\{\xb_i^\ell\}$ in (\ref{sec_datasetupsignalplusnoise}), which serves as the basis for analyzing GRAB-MDM under noise. For notional simplicity, we denote the $\sK$ clean point clouds as  $\mathcal{X}_\ell:=\{\xb_i^{\ell}\}_{i=1}^{n}, 1 \leq \ell \leq \sK.$ Our objective is to construct simultaneous embeddings across all views by exploiting their shared structure. To this end, we adopt a common manifold model for the clean views, formalized in Assumption \ref{assu_model}. Under this framework, the proposed algorithm corresponds, in the large-sample limit, to certain differential operators defined on the underlying manifold.

\subsection{The model and basic definitions}\label{sec_modelassumption}

We generalize the common manifold model \cite{talmon2019latent,ding2021kernel} to facilitate our analysis. Consider the following assumption.

\begin{assu}[Generalized common manifold model]\label{assu_model} 
Suppose $\mathcal{M}$ is a $d$-dimensional closed and connected Riemannian manifold with metric $\mathrm{g}$. 
Take a random vector $X: (\Omega, \mathcal{F}, \mathsf{P}) \rightarrow \mathcal{M}$. Fix $\sK\in \mathbb{N}$. For $\ell=1,\ldots, \sK$, consider $\iota_\ell:\mathcal{M}\rightarrow \mathbb{R}^{p_\ell}$ so that $\iota_\ell$ is diffeomorphic from $\mathcal{M}$ to its range. Endow $\iota_\ell(\mathcal{M})$ with an induced metric, denoted as $g_\ell$, from the ambient space.
Denote the random vector $X_\ell=\iota_\ell\circ X\in \mathbb{R}^{p_\ell}$ so that its range is supported on $\iota_\ell(\mathcal{M})$. 
Let $\mathbb{P}:=X_*\mathsf{P}$ and $\mathbb{P}_\ell:=(X_{\ell})_*\mathsf{P}$, $1 \leq \ell \leq \sK$ be the induced measures of $X_\ell$ on $\iota_\ell(\mathcal{M})$. Assume $X_*\mathsf{P}$ is absolutely continuous with related to the Riemannian density, or volume measure, $\mathrm{d} V$ on $\mathcal{M}$. Therefore we can find a density function $\varrho:\mathcal{M}\to \mathbb{R}_+$ so that 
\begin{equation}\label{ew_densitydefinition}
\frac{\mathrm{d} \mathbb{P}}{\mathrm{d} V}=\varrho,  
\end{equation}
almost everywhere. Assume there exists some small universal constant $0<\delta_0<1$ so that for all $z \in \mathcal{M},$ 
\begin{equation}\label{eq_densityproperty}
\delta_0 \leq \varrho(z) \leq \delta_0^{-1}. 
\end{equation}
When $\varrho$ is constant, we call $X$ uniform; otherwise non-uniform.
Then, i.i.d. sample $n$ points from $X$ to construct the point clouds $\mathcal{X}_\ell$, $\ell=1,\ldots,\sK$; that is, $\xb_i^\ell:=\iota_\ell(x_i)$, where $x_i\in \mathcal{M}$ is i.i.d. sampled from $X$.
\end{assu} 

See Figure~\ref{fig:inclu} for an illustration of this common manifold model, where simultaneous acquisition implies that observations across views are coupled through  the maps $\vartheta_{ji}:=\iota_j\circ \iota_i^{-1}$, $1 \leq i \neq j \leq \sK,$ that are diffeomorphisma satisfying
  \begin{equation}\label{eq_commonmode2}
X_j=\vartheta_{ji}(X_i), \  1 \leq i \neq j \leq \sK.
\end{equation}
Thus, we can view the manifold $\mathcal{M}$ as a {\em latent factor} shared by all views, in line with the perspective of~\cite{talmon2019latent}.
This mild structural assumption captures cross-view dependence and enables recovery of shared geometric information.

For each $\ell=1,\ldots,\sK$, the diffeomorphism $\iota_\ell$ together with the absolute continuity of $X_*\mathsf{P}$ ensures the existence of a unique density function $\varrho_\ell:\iota_\ell(\mathcal{M})\to \mathbb{R}_+$ such that 
\begin{equation}\label{ew_densitydefinition}
\frac{\mathrm{d} \mathbb{P}_\ell}{\mathrm{d} V_\ell}=\varrho_\ell,  
\end{equation} 
where $\mathrm{d} V_\ell$ denotes the Riemannian volume density associated with $\varrho_\ell$.
Since $\mathbb{P}_\ell={\iota_\ell}_*\mathbb{P}$, with the coordinate $x^1,\ldots,x^d$ and $y^1,\ldots,y^d$ on $\mathcal{M}$ and $\iota_\ell(\mathcal{M})$ respectively, by a direct calculation, for $y=\iota_\ell(x)$, we have 
\begin{align}
\varrho_\ell(y)=\varrho(\iota_\ell^{-1}(y))\left|\det(D\iota^{-1}_\ell(y))\right|\frac{\sqrt{\det G(\iota_\ell^{-1}(y))}}{\sqrt{\det G_{\ell}(y)}}\,, \label{relationship between sampling functions}
\end{align}
where $G=[g_{ij}]_{i,j=1}^d$ and $G_\ell=[g_{\ell,ij}]_{i,j=1}^d$ are the components of $g$ and $g_\ell$ respectively, and $D\iota^{-1}_\ell$ is associated with the coordinates. Therefore, there exists $0<\delta^*\leq \delta_0$ such that for all $1\leq\ell\leq\sK$ and all $z \in \mathcal{M},$ 
\begin{equation}\label{eq_densityproperty}
\delta^* \leq \varrho_\ell(\iota_\ell(z)) \leq (\delta^*)^{-1}. 
\end{equation}
Note that in general $\varrho_i(x)\neq \varrho_j(\vartheta_{ji}(x))$, and even if $X$ is uniform, in general the sampling of each $X_\ell$ is not uniform. 
Moreover, using the relation~\eqref{eq_commonmode2}, for any sufficiently regular function
$g : \iota_i(\mathcal{M}) \times \iota_j(\mathcal{M}) \to \mathbb{R}$, 
\begin{align} \label{prop_jointdistribution}
\mathbb{E}_i(g(X_i, X_j))& =\mathbb{E}_i(g(X_i, \vartheta_{ji}(X_i)))=\int_{\iota_i(\mathcal{M})} g(x, \vartheta_{ji}(x)) \mathrm{d} \mathbb{P}_i(x)\,,
\end{align}
where $\mathbb{E}_i$ denotes expectation under the law of $X_i$.

\begin{figure}[H]
\centering
    \includegraphics[width=10cm]{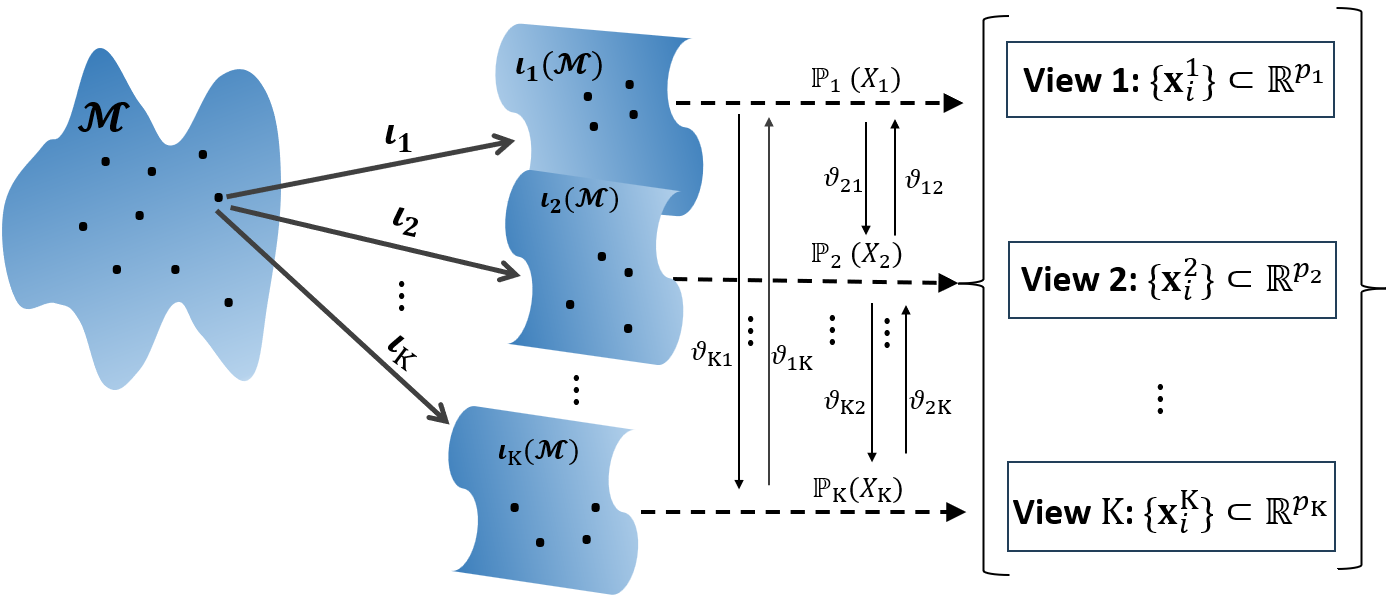}
    \caption{Illustration of the latent common manifold model. The data generating process starts with a common (latent) manifold $\mathcal{M}.$ For each point $z \in \mathcal{M},$ it can be simultaneously embedded into $\sK$ different submanifolds $\iota_\ell(\mathcal{M}), 1 \leq \ell \leq \sK,$ via $\sK$ different embeddings. The observed point clouds are collections of data points sampled jointly according to (\ref{ew_densitydefinition}) and (\ref{eq_commonmode2}) from those embedded submanifolds. In other words, for $1 \leq i \neq j \leq \sK$ and $1 \leq \alpha \leq n,$ $\xb_\alpha^j=\vartheta_{ji}(\xb_\alpha^i).$} \label{fig:inclu}
\end{figure}

\begin{rem}
For concreteness, and motivated by our target applications, we focus on the setting where all datasets arise from a single underlying Riemannian manifold equipped with a Riemannian metric, and each view is generated through a different embedding of this manifold. In principle, the framework can be extended to richer geometric structures, including multiple latent spaces with hierarchical relationships, a direction we leave for future investigation.
\end{rem}

\subsection{The limiting operators}
Based on the common manifold model, we introduce the limiting operators associated with our algorithm proposed in Section \ref{sec_proposedlalgo}. To this end, observe that for a $\sK\times 1$ block vector $v=[v_1^\top,\ldots,v_{\sK}^\top]^\top \in \mathbb{R}^{n\sK}$ with blocks of size $n\times 1$, where $v_l\in \mathbb{R}^n$ is the $l$-th block for $l=1,\ldots,\sK$, we have
\[
(\mathcal{D}^{-1} \mathcal{K}v)_i(l)=\frac{\sum_{j\neq i}\sum_{\alpha=1}^n\big(\sum_{\beta=1}^n  \mathbf{K}^{i}(l,\beta)\mathbf{K}^{j}(\beta,\alpha)\big)v_j(\alpha) }{\sum_{j\neq i}\sum_{\alpha=1}^n\big(\sum_{\beta=1}^n  \mathbf{K}^{i}(l,\beta)\mathbf{K}^{j}(\beta,\alpha)\big)}\,.
\]
At the purely algebraic level, this operator is a combination of matrix-matrix and matrix-vector multiplications.
However, its asymptotic interpretation under the manifold model requires care because each kernel matrix lives on a potentially different embedded manifold, and summation indices implicitly mix coordinates across views.
We therefore need to make explicit the geometric meaning of each component.

First, for the product $\mathbf{K}^{i}(l,\beta)\mathbf{K}^{j}(\beta,\alpha)$, note that $\mathbf{K}^{i}(l,\beta)=K \left(\frac{\| \yb^{i}_l-\yb^{i}_\beta \|_2^2}{\epsilon_{i}} \right)$ and $\mathbf{K}^{j}(\beta,\alpha)=K \left(\frac{\| \yb^{j}_\beta -\yb^{j}_\alpha\|_2^2}{\epsilon_{j}} \right)$ are kernels defined on different embedded manifolds, $\iota_i(\mathcal M)$ and $\iota_j(\mathcal M)$, respectively. To interpret the summation indexed by $\beta$ involving $\yb^{i}_\beta$ and $\yb^{j}_\beta$,
we introduce a {\em transfer kernel} $\widetilde{K}_{ij}:\iota_i(\mathcal{M})\times \iota_j(\mathcal{M})\to \mathbb{R}$ defined for $x \in \iota_i(\mathcal{M})$ and $y \in \iota_j(\mathcal{M})$ by 
\begin{align}\label{eq_widetildekij}
\widetilde{K}_{ij}(x,y)&=\mathbb{E}_i [K_{\epsilon_i} (x, X_i) K_{\epsilon_j}(\vartheta_{ji}(X_i), y)] \nonumber \\ & =\int_{\iota_i(\mathcal{M})} K_{\epsilon_i} (x, w) K_{\epsilon_j}(\vartheta_{ji}(w), y)  \mathrm{d}\mathbb{P}_i(w),
\end{align}
where
\begin{equation}\label{eq_kernelexplicitform}
K_{\epsilon_i}(x,w):=K\left(\|x-w\|_{\mathbb{R}^{p_i}_2}^2/\epsilon_i\right).  
\end{equation}
By sampling $n$ points, $\xb^i_1,\ldots,\xb^i_n$, on $\iota_i(\mathcal{M})$ following $\mathbb{P}_i$ and set $x$ and $y$ to be $\xb^i_l$ and $\xb^j_\alpha=\vartheta_{ji}(\xb^i_\alpha)$ respectively, \eqref{eq_widetildekij} is discretized into
\[
\widetilde{K}_{ij}(\xb^i_l,\xb^j_\alpha)\approx \frac{1}{n}\sum_{\beta=1}^n K_{\epsilon_i} (\xb^i_l, \xb^i_\beta) K_{\epsilon_j}(\vartheta_{ji}(\xb^i_\beta), \xb^j_\alpha) =\frac{1}{n}\mathbf{K}^{i,j}(l,\alpha)\,,
\]
where $\approx$ will be precisely quantified below in the proof.

The tricky part is the matrix-vector multiplication $\sum_{\alpha=1}^n \mathbf{K}^{i,j}(l,\alpha)v_j(\alpha)$, since we need to specify the quantity where $v_j$ is discretized from. Because the index $\alpha$ corresponds to the $j$-th view, it is natural to interpret $v_j$ as a discretized function defined on $\iota_j(\mathcal{M})$; that is, $v_j(\alpha):=f_j(\iota_j(z_\alpha))$, where $f_j\in C(\iota_j(\mathcal{M}))$.
Define an integral operator $\omega_{ij}:C(\iota_j(\mathcal{M}))\to C(\iota_i(\mathcal{M}))$ so that for all $f_j \in C(\iota_j(\mathcal{M}))$ and $x\in \iota_i(\mathcal{M})$, 
\begin{equation}\label{eq_omegaij}
\omega_{ij} f_j(x)= \frac{\mathbb{E}_j[\widetilde{K}_{ij}(x,X_j) f_j(X_j)]}{\sum_{j \neq i}\mathbb{E}_j[\widetilde{K}_{ij}(x,X_j)]}= \frac{\int_{\iota_j(\mathcal{M})}\widetilde{K}_{ij}(x,z) f_j(z)  \mathrm{d}\mathbb{P}_j(z)}{\sum_{j \neq i}\int_{\iota_j(\mathcal{M})}\widetilde{K}_{ij}(x,z)  \mathrm{d}\mathbb{P}_j(z)}\,.
\end{equation}
By sampling $n$ points, $z_1,\ldots,z_n\in\mathcal{M}$ following $\mathbb{P}$ and set $x=\xb^i_l$ and $v_j\in \mathbb{R}^n$ so that $v_j(\alpha)=f_j(\xb^j_\alpha)$, where $f_j\in C(\iota_j(\mathcal{M}))$, the numerator of \eqref{eq_omegaij} is discretized into
\[
\int_{\iota_j(\mathcal{M})}\widetilde{K}_{ij}(\xb^i_l,z) f_j(z) \mathrm{d}\mathbb{P}_j(z)\approx \frac{1}{n}\sum_{\alpha=1}^n \widetilde{K}_{ij}(\xb^i_l,\xb^j_\alpha) f_j(\xb^j_\alpha) \approx \frac{1}{n}[\mathbf{K}^{i,j}v_j](l)\,,
\]
and similarly for the denominator, where $\approx$ will be made precise in subsequent analysis. {Below, denote $\fb=(f_1, \cdots, f_\sK)^\top \in C(\iota_1(\mathcal{M}))\times \ldots \times C(\iota_{\sK}(\mathcal{M}))$ to be a vector of functions defined on different embedded manifolds with the evaluation $\fb(x)= (f_1(\iota_1(x)), \cdots, f_\sK(\iota_{\sK}(x)))^\top\in \mathbb{R}^{\sK}$ when $x\in \mathcal{M}$, or $\fb(x')= (f_1(\vartheta_{1i}(x')), \cdots, f_\sK(\vartheta_{\sK i}(x')))^\top\in \mathbb{R}^{\sK}$ when $x'\in \iota_i(\mathcal{M})$ for any $i=1,\ldots,\sK$. These evaluations are all equivalent, and we choose one or another depending on the context. In practice, we cannot access $\mathcal{M}$ but can only access $\iota_i(\mathcal{M})$, so the second evaluation is more natural for the algorithm analysis, while the first evaluation helps explore the latent relationship.}
We now introduce the limiting operator. 
\begin{defn}\label{defn_VCDM} 
Define an operator $\Omega:  C(\iota_1(\mathcal{M}))\times \ldots \times C(\iota_{\sK}(\mathcal{M})) \rightarrow C(\iota_1(\mathcal{M}))\times \ldots \times C(\iota_{\sK}(\mathcal{M}))$ such that for $\fb\in C(\iota_1(\mathcal{M}))\times \ldots \times C(\iota_{\sK}(\mathcal{M}))$, $\Omega \fb$ satisfies 
\begin{equation*}
[\Omega \fb]_i :=\sum_{j \neq i}{\omega_{ij} (f_j\circ \vartheta_{ji})}\in C(\iota_i(\mathcal{M}))\,,
\end{equation*}
where $[\Omega \fb]_i$ is the $i$-th component of $\Omega \fb$.
In other words, {for $x \in \mathcal{M}$, we have $[\Omega \fb]_i(\iota_i (x))=
\sum_{j \neq i}\omega_{ij} f_j(\iota_j(x))\in \mathbb{R}$.}
\end{defn}

\begin{rem}
Viewing $\fb$ as a $\sK$-dim vector with functional entries, we may write
$\Omega \fb=\Ab \fb$, where $ \Ab$ is the functional analogue of $\mathcal{K}$, a $\sK \times \sK$ matrix whose entries are operators,
\begin{equation}\label{eq_defnAb}
\Ab:=
\begin{pmatrix}
0 & \omega_{12} & \omega_{13} &\cdots& \omega_{1 \sK} \\
\omega_{21} & 0 & \omega_{23} & \cdots & \omega_{2 \sK} \\
\vdots &  \vdots & \vdots&  \vdots & \vdots \\
\omega_{\sK 1} &  \omega_{\sK 2}&  \omega_{\sK 3} & \cdots& 0
\end{pmatrix}\,.
\end{equation}
Its structure mirrors that of the corresponding matrix in the discretized setting.
\end{rem}

\section{Asymptotic analysis}\label{section asymptotic analysis} In this section, we provide the asymptotic analysis of our algorithm for the clean signals $\{\xb_i^\ell\}_{i=1}^n$. A robust analysis extending these results to the noisy observations described in (\ref{sec_datasetupsignalplusnoise}) is deferred to Section~\ref{sec_robustnessofresults}. 
\subsection{Connection with the Laplace-Beltrami operator: Bias analysis}
In this section, we analyze the bias of the sequence of operators introduced in Definition~\ref{defn_VCDM} and clarify its connection to, as well as its distinction from, the classical Laplace-Beltrami operator \cite{singer2006graph}.
We begin by stating our assumptions on the kernel function $K. $ For $r,l,k=0,1,2,\ldots,$ define
$$\mu_{r,l}^{(k)}:=\int_{\mathbb{R}^{d}}\|x\|^{l}\partial_{k}K^{r}(\|x\|)\,\mathrm{d}x.$$ 

\begin{assu}\label{assum_main}
Assume that the kernel function $K \in C^3(\mathbb{R}^+)$ is bounded, $K(0), K(1)$ are away from $0$ by some constant, and $K$ is normalized so that $\mu_{1,0}^{0}=1$.
\end{assu}

We begin by introducing notations related to the geometry of the manifolds.
For the common manifold $\mathcal{M}$ and its embedded images $\iota_k(\mathcal{M}), 1 \leq k \leq \sK,$ denote for each $z \in \mathcal{M}$ the tangent space $T_z \mathcal{M}$ and the embedded tangent space $T_{\iota_k(z)}\iota_k(\mathcal{M})=(\iota_k)_* T_z \mathcal{M},$. Let $\nabla^{(\ell)}$ and $\Delta^{(\ell)}$ denote, respectively, the covariant derivative and Laplace-Beltrami operator on $\iota_\ell(\mathcal{M})$. For $x=\iota_k(z), 1 \leq k \leq \sK,$ write $\Second_{x,k}$ for the second fundamental form of $\iota_k(\mathcal{M})$ at $x$, and let $s_k(x)$ denote the scalar curvature. Define 
\begin{equation}\label{eq_wxdefinition}
\sw_k(x):=\frac{1}{3}s_k(x)-\frac{d}{12|\mathsf{S}^{d-1}|} \frac{\mu_{1,3}^{(1)}}{\mu_{1,2}^{(0)}} \int_{\mathsf{S}^{d-1}} \Second_{x,k}^2(\theta, \theta)\mathrm{d} \theta,
\end{equation}
where $|\mathsf{S}^{d-1}|$ is the volume of the canonical $(d-1)$-dim sphere. 
The main result of this section is summarized in the following theorem.  

{
\begin{thm}\label{thm_bias_vmvl}   
Suppose Assumptions \ref{assu_model} and \ref{assum_main} hold. Fix $1 \leq i \leq \sK$. Denote $\fb=(f_1,\ldots,f_{\sK}) \in C^4(\iota_1(\mathcal{M}))\times \ldots \times C^4(\iota_{\sK}(\mathcal{M}))$. For $i=1,\ldots, \sK$, denote $\bar{f}\in C^4(\iota_i(\mathcal{M})))$ as 
\[
\bar{f}_i(x):=\frac{\sum_{j \neq i} \epsilon_j^d \varrho_j(\vartheta_{ji}(x)) f_j(\vartheta_{ji}(x))}{\sum_{j \neq i}\epsilon_j^d \varrho_j(\vartheta_{ji}(x)) }\,.
\] 
When $\epsilon_1,\ldots,\epsilon_{\sK}$ are sufficiently small and satisfy $\epsilon_i \asymp \epsilon_j$, for each $i=1,\ldots, \sK$, we have for $x \in \iota_i(\mathcal{M})$, 
{\small
\begin{align}\label{eq_VCDM}
&[\Omega \fb]_i(x) =\,\bar{f}_i(x) +
\frac{\mu_{1,2}^{(0)} }{2d}\left[\frac{\sum_{j \neq i} \epsilon_j^{d+1}
\big([\Delta^{(j)} (f_j\varrho_j)  - \sw_j \varrho_j f_j] \big |_{\vartheta_{ji}(x)} - 
[\Delta^{(j)} \varrho_j - \sw_j \varrho_j ]\big|_{\vartheta_{ji}(x)}
\bar{f}_i(x) \big) 
 }{\sum_{j \neq i} \epsilon_j^d\varrho_j(\vartheta_{ji}(x))} \right. \\ 
 &\left.+ {\epsilon_i \frac{\sum_{j \neq i}\epsilon_j^d\left( \left[\Delta^{(i)}(\varrho_i \varrho_j\circ \vartheta_{ji} f_j\circ\vartheta_{ji})-\sw_i \varrho_i \varrho_j \circ \vartheta_{ji} f_j \circ \vartheta_{ji}
 \right]-\bar{f}_i \left[ \Delta^{(i)}(\varrho_i\varrho_j\circ\vartheta_{ji})-\sw_i \varrho_i \varrho_j \circ \vartheta_{ji} \right] \right)(x)}
 {\sum_{j \neq i} \epsilon_j^d\varrho_i (x)\varrho_j(\vartheta_{ji}(x))} } \right]  \nonumber\\
 &+ \mathrm{O} \Big(  \sum_{j \neq i}  (\epsilon^2_i+ \epsilon^2_j ) \Big)\,.  \nonumber 
\end{align} 
} 
\end{thm}

From the asymptotic expansion in (\ref{eq_VCDM}), we observe that the $i$-th component of $\Omega \fb$ involves nontrivial interactions among {\em all views}. The first first-order term contains operators acting on the non-$i$-th manifolds, and the second first-order term involves those on the $i$-th manifold. Since there is no simple invariance formula for the Laplace-Beltrami operator under a general diffeomorphism, one generally has  $\Delta^{(i)}(f\circ \vartheta_{ji})(x)\neq \Delta^{(j)}f(\vartheta_{ji}(x))$ when $f\in C^\infty(\iota_j(\mathcal{M}))$ and $x\in \iota_i(\mathcal{M})$, unless $\vartheta_{ji}$ is an isometry. Consequently, in general \eqref{eq_VCDM} cannot be further simplified, and it is not transparent from \eqref{eq_VCDM} alone how information is exchanged across views or how each view benefits from the others.

To parse \eqref{eq_VCDM}, we start with assuming $\sK=2$. In this case, when $i=1$, we have $\bar{f}_1 = f_2 \circ \vartheta_{21}$, and  the terms in \eqref{eq_wxdefinition} vanish so that
\begin{align*}
[\Omega \fb&]_1(x) =f_2\circ \vartheta_{21}(x) +
 \epsilon_1 \frac{\mu_{1,2}^{(0)} }{2d}\left(\Delta^{(1)} (f_2\circ \vartheta_{21})(x) +2\nabla^{(1)}\log(\varrho_1 \varrho_2\circ \vartheta_{21})(x)\cdot \nabla^{(1)}(f_2\circ\vartheta_{21}))(x)\right)\\ 
& + \epsilon_2 \frac{\mu_{1,2}^{(0)} }{2d}\left((\Delta^{(2)} f_2)\circ \vartheta_{21}(x)
+2(\nabla^{(2)}\log(\varrho_2)\cdot \nabla^{(2)}f_2)\circ \vartheta_{21}(x)\right) + \mathrm{O} (\epsilon^2_1+ \epsilon^2_2 )\,.\nonumber
\end{align*} 
In this setup, $\frac{[\Omega \fb]_1-f_2\circ \vartheta_{21}}{\epsilon_1}$ approximates a mixed Laplacian  $\Delta^{(1)} (f_2\circ \vartheta_{21})+ \frac{\epsilon_2}{\epsilon_1} (\Delta^{(2)} f_2)\circ \vartheta_{21}$ that involves a first order differentiation term $2 \nabla^{(1)}\log(\varrho_1 \varrho_2\circ \vartheta_{21})\cdot \nabla^{(1)}(f_2\circ\vartheta_{21}))
+2\frac{\epsilon_2}{\epsilon_1}(\nabla^{(2)}\log(\varrho_2)\cdot \nabla^{(2)}f_2)\circ \vartheta_{21}$, which depends on the sampling scheme. Note that the mixed Laplacian depends on $\frac{\epsilon_2}{\epsilon_1}$, which is asymptotically constant under the assumption $\epsilon_1\asymp \epsilon_2$. Consequently, the choice of $\epsilon_i$ plays a role in enhancing the desired Laplacian.

However, the situation becomes complicated when $\sK\geq 3$. In this case, $\bar{f}_i$ is in general  no longer equal to $f_l\circ \vartheta_{li}$ for any $l\neq i$. The resulting mixed Laplacian involves not only first order differentiation terms depending on the sampling scheme, but also a zeroth order term, $(\sw_j \varrho_j)|_{\vartheta_{ji}(x)}
\bar{f}_i(x)- (\sw_j \varrho_j f_j)|_{\vartheta_{ji}(x)}$. which involves curvature and is generically nonzero. Note that even if $\fb$ is chosen so that $f_i=f_j\circ \vartheta_{ji}$, we still cannot guarantee $(\sw_j \varrho_j)|_{\vartheta_{ji}(x)}
\bar{f}_i(x)- (\sw_j \varrho_j f_j)|_{\vartheta_{ji}(x)}=0$, since the density functions are not comparable \eqref{relationship between sampling functions}; that is, $\varrho_j\neq \varrho_i\circ \vartheta_{ij}$. This suggests that the presence of a single view with large curvature can impact the embedding through its weight $\sw_i$.

To gain further insight when $\sK>2$, consider a special case that $\iota_k=c_k\iota$ for some $c_k>0, 1 \leq k \leq \sK,$ where $\iota:\mathcal{M}\to \mathbb{R}^p$ is a fixed embedding. In other words, all views observe the same latent manifold but with possibly different dilation factors. In this setup, all quantities scale nicely. Specifically, $c_j^d\varrho_j(\vartheta_{ji}(x))=c_j^d\varrho_j(c_jx/c_i)=c_i^d\varrho_i(x)$ for $i\neq j$ since $\vartheta_{ji}(x)=c_jx/c_i$ when $x\in \iota_i(M)$. Therefore, $\sum_{j \neq i}\varrho_j(\vartheta_{ji}(x))= c_i^d\varrho_i(x)\sum_{j\neq i} \frac{1}{c_j^d}$ and $\bar{f}_i(x)=\frac{\sum_{j\neq i}c_j^{-d} f_j(\vartheta_{ji}(x))}{\sum_{j\neq i} c_j^{-d}}$. 
The metrics satisfy $\vartheta_{ji}^*g^{(j)}=\frac{c_j^2}{c_i^2} g^{(i)}$, so the curvature dilates squared inversely like $\frac{c_i^2}{c_j^2}$; that is, $\sw_j(\vartheta_{ji}(x))=\sw_j(c_jx/c_i)=\frac{c_i^2}{c_j^2}\sw_i(x)$. The gradient dilates inversely like $(\nabla^{(j)}f)(\vartheta_{ji}(x))=(\Delta^{(j)}f)(c_jx/c_i)=\frac{c_i}{c_j}\nabla^{(i)}(f(c_jx/c_i))=\frac{c_i}{c_j}\nabla^{(i)}(f\circ\vartheta_{ji})(x)$, where $f\in C^2(\iota_j(\mathcal{M}))$ and $x\in \iota_i(\mathcal{M})$. Moreover, Laplace-Beltrami operator dilates accordingly; that is, $(\Delta^{(j)}f)(\vartheta_{ji}(x))=(\Delta^{(j)}f)(c_jx/c_i)=\frac{c_i^2}{c_j^2}\Delta^{(i)}(f(c_jx/c_i))=\frac{c_i^2}{c_j^2}\Delta^{(i)}(f\circ\vartheta_{ji})(x)$, where $f\in C^2(\iota_j(\mathcal{M}))$ and $x\in \iota_i(\mathcal{M})$. Putting all these together, if we further assume that $\epsilon_1=\ldots=\epsilon_{\sK}=\epsilon$, \eqref{eq_VCDM} is reduced to 
{\small
\begin{align}
[\Omega \fb&]_i(x) =\bar{f}_i(x) +
\epsilon\frac{\mu_{1,2}^{(0)} }{2d}\frac{\sum_{j \neq i} c_i^{2}c_j^{-(d+2)}
\big[\Delta^{(i)} ((f_j\circ \vartheta_{ji})\varrho_i)  - 
\bar{f}_i(x) \Delta^{(i)} \varrho_i 
- (\sw_i \varrho_i (f_j\circ\vartheta_{ji}-\bar{f}_i)) \big] (x)
 }{\varrho_i (x)\sum_{j \neq i}c_j^{-d}}
  \nonumber \\
& +\epsilon\frac{\mu_{1,2}^{(0)} }{2d}\frac{\sum_{j \neq i}c_j^{-d}\big[  (\Delta^{(i)}(\varrho_i^2  f_j\circ\vartheta_{ji})-\sw_i \varrho_i^2 f_j \circ \vartheta_{ji}) -\bar{f}_i(\Delta^{(i)}(\varrho_i^2)-\sw_i \varrho_i^2) \big](x)}
{\varrho_i ^2(x)\sum_{j \neq i}c_j^{-d}} + \mathrm{O} (\epsilon^2)\,,\nonumber
\end{align} }
which shows that $\frac{([\Omega \fb]_i-\bar{f}_i)(x)}{\epsilon}$ approximates a mixed weighted Laplaciain with mixing weights $c_j^{-d}$ and $c_i^2c_j^{-(d+2)}$, combined with first order and zeroth order differentiation terms. If we further assume $f_j(\vartheta_{ji}(x))=f_i(x)$ so that $\bar{f}_i(x)=f_i(x)$, it is reduced to
\begin{align*}
[\Omega \fb&]_i(x) ={f}_i(x) +
\epsilon\frac{\mu_{1,2}^{(0)} }{2d}\frac{\sum_{j \neq i} c_i^{2}c_j^{-(d+2)} }{\sum_{j \neq i}c_j^{-d}}\big[\Delta^{(i)} f_i  - 
\nabla^{(i)} \log(\varrho_i) \cdot \nabla^{(i)}{f}_i
\big] (x) \\
& +\epsilon \frac{\mu_{1,2}^{(0)} }{2d}\big[  \Delta^{(i)}f_i+ 4\nabla^{(i)}\log(\varrho_i)\cdot \nabla^{(i)}f_i) \big](x) + \mathrm{O} (\epsilon^2)\,,
\end{align*}
If we further assume $c_1=\ldots=c_{\sK}=1$ so that $\varrho_1=\ldots=\varrho_{\sK}$,  and  choose $f_1=\ldots=f_{\sK}=f$, then we further simply (\ref{eq_VCDM}) and obtain
\begin{align*}
[\Omega \fb(x)]_i=f(x)  +\frac{\mu_{1,2}^{(0)}}{d} \Delta f(x) +\frac{3\mu_{1,2}^{(0)}}{d}  \nabla \log(\varrho(x))
 \cdot \nabla f(x)
+ \mathrm{O} \left(\epsilon^2  \right)
\end{align*}  
for any $i=1,\ldots,\sK$, yielding an expression structurally similar to the standard DM framework \cite{coifman2006diffusion}, except for the appearance of the constant $3$ rather than $1$. As shown in the proof, this factor of $3$ arises from three distinct contributions: the gradients associated with the two different views and an additional term generated by applying the product rule to their interaction. 
}

\begin{rem}\label{rem_biasanalysis}
In the DM literature and its extension \cite{coifman2006diffusion,MR4279237,singer2017spectral} for the single-view setting, a common practice to recover the Laplace-Beltrami operator without non-uniform sampling impact is the {\em $\alpha$-normalization}, with parameter $\alpha\geq 0$. In particular, choosing $\alpha = 1$ removes the influence of the sampling density from the Laplacian term, yielding an operator independent of the underlying distribution. 
In the multiview setting, when the goal is to recover the Laplace-Beltrami operator on $\mathcal{M}$, one might attempt to apply the same idea by normalizing the sampling density for each dataset before constructing $\mathcal{K}$. Alternatively, one could normalize the kernel in \eqref{eq_widetildekij} after constructing $\mathcal{K}$. However, neither approach produces the desired effect.
The key obstruction is the {\em joint normalization} step, $\mathcal{D}$, which mixes the geometries of different views that may differ by a diffeomorphism. Consequently, even if each view is individually normalized so that its kernel compensates for sampling density, the joint normalization reintroduces unwanted cross-view terms. Due to the page limitation, this important issue is made explicit in Theorems \ref{thm_alphanormalizationone} and \ref{thm_alphanormalizationtwo} of the supplement. In short, the similar $\alpha$ normalization, either sensor-wisely or jointly, leads to more complicated mixing of different views that limited insight can be directly caught.
\end{rem}

\subsection{Variance analysis}\label{sec_ptcvarianceanalsyis}
In this section, we discuss how our proposed algorithm connects to the limiting operators.
We begin by introducing notations. Recall that
$\fb=(f_1,\ldots,f_{\sK}) \in C^4(\iota_1(\mathcal{M}))\times \ldots \times C^4(\iota_{\sK}(\mathcal{M}))$ and that the samples are given by $\xb_i^k:=\iota_k(x_i), \ 1 \leq i \leq n, \ 1 \leq k \leq \sK,$ where $x_i\in \mathcal{M}$ are i.i.d. sampled from $X$.
Discretize $\mathbf{f}$ into a block vector  $\bm{f} \in \mathbb{R}^{n\sK}$ with $\sK$ blocks of size $n$: 
\begin{equation}\label{eq_overallf}
\bm{f}=(\bm{f}_1, \cdots, \bm{f}_{\sK})^\top, \ \ \bm{f}_\ell=(f_\ell(\xb_1^\ell), \cdots, f_\ell(\xb_n^\ell))^\top \in \mathbb{R}^n\quad\mbox{for }\ 1 \leq   \ell \leq \sK\,.
\end{equation}
For $1 \leq \ell \leq \sK,$ define the clean kernel matrix $\widecheck{\bK}^\ell\in \mathbb{R}^{n\times n}$ analogously to (\ref{eq_Kmultiplicationentry}), using the clean point cloud $\{\xb_i^{\ell}\}$: 
\begin{equation}\label{definition clean matrix from clean data}
\widecheck{\bK}^{\ell}(i,j)=K\left(-\frac{\|\xb_i^{\ell}-\xb_j^{\ell} \|_2^2}{\epsilon_{\ell}} \right), \ 1 \leq i,j \leq n. 
\end{equation}
For $1 \leq \ell_1 \neq \ell_2 \leq \sK,$ set $\widecheck{\bK}^{\ell_1, \ell_2}=\widecheck{\bK}^{\ell_1} \widecheck{\bK}^{\ell_2}.$ Construct $\widecheck{\mathcal{K}}$ as in (\ref{eq_largekernelmatrix}) using $\widecheck{\bK}^{\ell_1, \ell_2}$, and define the corresponding transition matrix $\widecheck{\mathcal A}=\widecheck{\mathcal{D}}^{-1} \widecheck{\mathcal{K}}.$ For $1 \leq s \leq n\sK,$ denote by $[\widecheck{\mathcal{A}}\bm{f}](s)$ the $s$th entry of the vector $\widecheck{\mathcal{A}}\bm{f} \in \mathbb{R}^{n \sK}.$  The variance analysis of GRAB-MDM is stated below.

\begin{thm}\label{thm_vmvlvarianceanalysis} Suppose the assumptions of Theorem \ref{thm_bias_vmvl} hold. Moreover,  for $1 \leq i \leq \sK$, assume  
\begin{equation}\label{eq_bandwidthassumptionvariance}
\epsilon_i=\oo_\prec(1)\ \mbox{ and }\ \frac{\sqrt{\log n}}{n^{1/2}\epsilon_{i}^{d/4} \epsilon_j^{d/4}} = \oo_\prec(1),
\end{equation}
when $n\to \infty$. Then, for all $1\leq s \leq n \sK$, asymptotically when $n\to \infty,$ we have
    \begin{equation}\label{eq_errorratepointwise}
        [\widecheck{\mathcal{A}}\textbf{\textit{f}}](s) = [\Omega_k \mathbf{f} ]_{\ell}(\xb_{s'}^\ell)
        + \mathrm{O}_{\prec} \left( \sum_{ j \neq \ell} \frac{1}{\sqrt{n} \epsilon_\ell^{d/4} \epsilon_j^{d/2}} \right)\,,
    \end{equation}
    where $s=1,\ldots,n\sK$, $\ell=\lfloor (s-1)/n\rfloor+1$ and $s'=s-(\ell-1) n$.  
    \end{thm}

Assumption~\eqref{eq_bandwidthassumptionvariance} imposes upper and lower bounds on the bandwidths used in both the algorithm and the associated operators, consistent with the single-view conditions in \cite{coifman2006diffusion,singer2006graph}.
Likewise, the convergence rate in \eqref{eq_errorratepointwise} matches the single-view results when $\epsilon_\ell=\epsilon_j$. More generally, our analysis shows that the contribution of view $j\neq \ell$ to the $\ell$th view scales as $\epsilon_j^{-d/2}$, which asymptotically dominates the $\epsilon_\ell^{-d/4}$ term. 
This discrepancy originates from the structure of \eqref{eq_widetildekij}.
As detailed in \eqref{eq_expansionexpansion111} of the supplement, the $\ell$-th view undergoes both row-wise (additive) and column-wise (multiplicative) operations, whereas other views experience only the column-wise operation, leading to the stronger $\epsilon_j^{-d/2}$ dependence. This insight suggests bandwidth-design strategies across views. This fact echos the practical application of (\ref{al_VCDM1}). Empirically, one does not use all $n\sK$ entries of an eigenvector for embedding but instead chooses a length-$n$ subvector. As demonstrated in Section~\ref{sec_numericalresult}, it is often advantageous to select the subvector corresponding to the cleanest view, which can accommodate a smaller bandwidth due to its higher signal-to-noise ratio (SNR), and thus achieves superior convergence performance.

 \begin{rem}
In our algorithm, we select the bandwidth according to \eqref{eq_finalbandwidthform}, where $\epsilon_i = \sfc h_i$. For a fixed underlying manifold, $h_i$ can  be assumed to satisfy $h_i \asymp 1$. Moreover, to ensure that condition \eqref{eq_bandwidthassumptionvariance} holds and that the error term in \eqref{eq_errorratepointwise} vanishes, the global scaling factor $\sfc$ must satisfy $n^{-2d/3} \ll \sfc \ll 1$. This condition indicates that $\sfc$ should be small, but chosen within an appropriate range for practical implementation. We refer to the second paragraph following Proposition~\ref{thm_bandwidth} for further discussion of practical considerations regarding the choice of $\sfc$, as implemented in Algorithm~\ref{al_bandwidthselection}.      
\end{rem}

With Theorems \ref{thm_bias_vmvl} and \ref{thm_vmvlvarianceanalysis}, we obtain more insights to the complication and structure of GRAB-MDM (and the original MDM). We introduce some additional notations. Let $\mathsf A$, $\mathsf B_1$, $\mathsf C_1$, $\mathsf B_2$, and $\mathsf C_2$ be $\sK \times \sK$ block matrices with block size $n \times n$, whose diagonal blocks are zero. For $\alpha \neq \beta$, the off-diagonal $(\alpha,\beta)$-th blocks are $n \times n$ diagonal matrices defined as
{\small
\begin{align*}
& \mathsf{A}(\alpha,\beta):=\texttt{diag}\left(
\frac{ \epsilon_\beta^d \varrho_\beta(\vartheta_{\beta \alpha}(\xb^\alpha_k))}{\sum_{l \neq \alpha}\epsilon_l^d\varrho_l(\vartheta_{l \alpha}(\xb^\alpha_k)) } \right)_{k=1}^n, \
\mathsf{B}_1(\alpha,\beta):= \texttt{diag}\left(
\frac{ 2 \epsilon_\beta^d}{\sum_{l \neq \alpha} \epsilon_l^d\varrho_l(\vartheta_{l \alpha}(\xb^\alpha_k)) } \right)_{k=1}^n, \\
&\mathsf{C}_1(\alpha,\beta):=\texttt{diag}\left(
\frac{ \epsilon_\beta^d
(\Delta^{(\beta)}\varrho_\beta+\sw_\beta\varrho_\beta)}{\sum_{l \neq \alpha}\epsilon_l^d\varrho_l}\Big|_{\vartheta_{\beta \alpha}(\xb^\alpha_k)} \right)_{k=1}^n, \  
\mathsf{B}_2(\alpha,\beta):= \texttt{diag}\left(
\frac{
2 \epsilon_\beta^d}{\sum_{l \neq \alpha}\epsilon_l^d(\varrho_\alpha\varrho_l\circ\vartheta_{l \alpha})(\xb^\alpha_k))} \right)_{k=1}^n,  \\
& \mathsf{C}_2(\alpha,\beta):=\texttt{diag}\left(
\frac{ \epsilon_\beta^d
[\Delta^{(\alpha)}(\varrho_\alpha\varrho_\beta\circ \vartheta_{\beta \alpha})-\sw_\alpha \varrho_\alpha \varrho_\beta \circ \vartheta_{\beta \alpha} ]}{\sum_{l \neq \alpha}\epsilon_l^d\varrho_\alpha\varrho_l\circ \vartheta_{l \alpha}}\Big|_{\xb^\alpha_k} \right)_{k=1}^n.
\end{align*}}
For $1 \leq i \leq \sK,$ we denote several vectors $\nabla_1\bm{f}$, $\Delta_1\bm{f}$, $\nabla_2\bm{f}$, and $\Delta_2\bm{f} \in \mathbb{R}^{\sK n}$ with $\sK$ blocks of size $n$, whose $\ell$-th block, $1 \leq \ell \leq \sK,$ are
$(\nabla^{(\ell)} \varrho_\ell(\xb_k^\ell)\cdot\nabla^{(\ell)} f_\ell(\xb_k^\ell))_{k=1}^n,$ $(\Delta^{(\ell)} f_\ell(\xb_k^\ell))_{k=1}^n,$ $((\nabla^{(i)} (\varrho_i\varrho_\ell\circ \vartheta_{\ell,i})\cdot\nabla^{(i)} (f_\ell\circ \vartheta_{i\ell}))|_{\vartheta_{i,\ell}(\xb_k^\ell)})_{k=1}^n,$ and $(\Delta^{(i)} (f_\ell\circ \vartheta_{i\ell})|_{\vartheta_{i\ell}(\xb_k^\ell)})_{k=1}^n,$ respectively. Moreover, we denote
$\mathcal{E}:=\texttt{diag}(\epsilon_1\mathbf{1}_n^\top,\ldots,\epsilon_{\sK}\mathbf{1}_n^\top) \in \mathbb{R}^{\sK n \times \sK n}$ with $\mathbf{1}_n \in \mathbb{R}^n$ having entries all ones,  $\mathcal{H}:=\texttt{diag}\left(e_1\mathbf{1}_n^\top, \ldots, e_{\sK}\mathbf{1}_n^\top\right)\in \mathbb{R}^{\sK n\times \sK n}$ with $e_l:=\mathrm O\left(\sum_{j\neq l}(\epsilon^2_j+\epsilon^2_l)\right)+\mathrm{O}_{\prec} \Big( \sum_{ j \neq l} \frac{1}{\sqrt{n} \epsilon_l^{d/4} \epsilon_j^{d/2}} \Big)$, $l=1,\ldots,\sK$, which includes the higher order bias and stochastic deviation.

With these notations, the first order term in (\ref{eq_VCDM}) is discretized as
$\mathsf A \mathcal{E}\Delta_1\bm{f} + \mathsf B_1 \mathcal{E}\nabla_1\bm{f} + \mathsf{C}_1\mathcal{E}\bm{f} - \texttt{diag}(\mathsf{A}\bm{f})\mathsf C_1 \mathcal{E}\mathbf{1}_{\sK n}$, 
and the second first order term as 
$\mathcal{E}[\mathsf A\Delta_2\bm{f}+ \mathsf B_2 \nabla_2\bm{f} + \mathsf{C}_2\bm{f} -\texttt{diag}(\mathsf{A}\bm{f})\mathsf C_2\mathbf{1}_{\sK n}]$. 
Putting all together, we have
\begin{align*}
\widecheck{\mathcal{A}}\textbf{\textit{f}}
=\,&
\mathsf A\textbf{\textit{f}}+
\frac{\mu_{1,2}^{(0)} }{2d}[(\mathsf A \mathcal{E}\Delta_1\bm{f} +  \mathcal{E}\mathsf A\Delta_2\bm{f})
+(\mathsf B_1 \mathcal{E}\nabla_1\bm{f} + \mathcal{E}\mathsf B_2 \nabla_2\bm{f}) \\
&+(\mathsf{C}_1\mathcal{E}\bm{f} + \texttt{diag}(\mathsf{A}\bm{f})\mathsf C_1 \mathcal{E}\mathbf{1}_{\sK n}
+\mathcal{E}\mathsf{C}_2\bm{f} +  \mathcal{E}\texttt{diag}(\mathsf{A}\bm{f})\mathsf C_2\mathbf{1}_{\sK n})]
 + \mathcal{H} \,.
\end{align*} 
Clearly, $\mathsf A$ depends on the sampling scheme, which is in general unknown. Suppose we can estimate the sampling density and obtain an accurate $\mathsf A$. We claim that $\mathsf A$ is invertible and has a well controlled conditional number when the diffeormorphisms are ``not crazy''. To see this claim, we have
\[
\mathsf A=\mathsf A_0+E\,,
\]
where $\mathsf A_0:=(\boldsymbol{1}_{\sK}\boldsymbol{1}_{\sK}^\top -I_{\sK\times\sK})\otimes \mathsf{A}_*$, $\otimes$ is the Kronecker product, and $\mathsf{A}_*=\frac{1}{\sK(\sK-1)}\sum_{\alpha\neq \beta}\mathsf{A}(\alpha,\beta)$ is the mean, and $E$ is the deviation caused by diffeomorphisms among different sensors. Recall that the eigenvalues of $\boldsymbol{1}_{\sK}\boldsymbol{1}_{\sK}^\top -I_{\sK\times\sK}$ include $\sK-1$ and $-1$ with multiplicity $\sK-1$. The assumption that the diffeormorphism are not crazy can then be quantified by the operator norm of $E$ being sufficiently smaller than $1$. Moreover, by the lower bound assumption of the sampling density, we know that the spectrum of $\mathsf{A}_*$ is bounded from above and bounded away from $0$. Since $\mathsf A^{-1}=(\boldsymbol{1}_{\sK}\boldsymbol{1}_{\sK}^\top -I_{\sK\times\sK})^{-1}\otimes \mathsf{A}_*^{-1}$, we see that $\mathsf A$ is invertible with a well controlled conditional number. 

{
Therefore, we have  
\begin{align*}
\frac{2d}{\mu_{1,2}^{(0)}}\mathcal{E}^{-1}\,&(\mathsf A^{-1}\widecheck{\mathcal{A}}-I)\bm{f}
=
(\Delta_1\bm{f} + \mathcal{E}^{-1}\mathsf A^{-1}\mathcal{E}\mathsf A\Delta_2\bm{f})+
\mathcal{E}^{-1}\mathsf A^{-1}(\mathsf B_1 \mathcal{E}\nabla_1\bm{f} + \mathcal{E}\mathsf B_2 \nabla_2\bm{f}) \\
&+
\mathcal{E}^{-1}\mathsf A^{-1}(\mathsf{C}_1\mathcal{E}\bm{f} + \texttt{diag}(\mathsf{A}\bm{f})\mathsf C_1 \mathcal{E}\mathbf{1}_{\sK n}
+\mathcal{E}\mathsf{C}_2\bm{f} +  \mathcal{E}\texttt{diag}(\mathsf{A}\bm{f})\mathsf C_2\mathbf{1}_{\sK n})
+
\mathcal{E}^{-1}\mathsf A^{-1}\mathcal{H}\,,
\end{align*}
where except the term  $\mathcal{E}^{-1}\mathsf A^{-1}\mathcal{H}$, which converges to $0$ a.s. as $n\to \infty$ by Theorem \ref{thm_vmvlvarianceanalysis} and the Borel-Cantelli lemma, all other terms on the right-hand side are of order $1$ or smaller.

Note that $\mathcal{E}^{-1}(\mathsf A^{-1}\widecheck{\mathcal{A}}-I)$ mimics the graph Laplacian commonly used in manifold learning. We emphasize, however, that due to the nontrivial sampling density, $(\widecheck{\mathcal{A}}-I)\bm{f}$ is in general a mixture of block components of $\bm{f}$, rather than a second order differential operator. 
This result shows that, asymptotically, $\frac{2d}{\mu_{1,2}^{(0)}}\mathcal{E}^{-1}(\mathsf A^{-1}\widecheck{\mathcal{A}}-I)$ approximates a second order differential operator $\mathcal L$ on $C^4(\iota_1(\mathcal{M}))\times \ldots \times C^4(\iota_{\sK}(\mathcal{M}))$, which is a mixture of Laplace operators over different manifolds with different bandwidths.

In summary, if the objective is to recover the Laplace-Beltrami operator of either the latent or the observed manifolds for the purpose of spectral embedding, then, even with full knowledge of the sampling density, the standard approach based on eigendecomposition of a graph Laplacian may no longer be feasible. In practice, the eigenstructure of the random walk matrix $\widecheck{\mathcal{A}}$ is governed by a substantially more intricate kernel integral operator, corresponding to a mixture of Laplacians defined over multiple manifolds with heterogeneous sampling densities and bandwidths. A rigorous characterization of this eigenstructure, as well as a proof of spectral convergence for GRAB-MDM, would therefore require the development of new analytical techniques beyond those in \cite{MR4279237,van2023weak,von2008consistency,shen2022scalability}. We leave this challenging direction for future work.

\begin{rem}
If we add the diagonal back into (\ref{eq_largekernelmatrix}) using (\ref{eq_Kmultiplication}), the resulting variance includes an extra error term of order $\frac{1}{\sqrt{n}\epsilon_\ell^{d/4-1/2}}$, which is negligible asymptotically. 
\end{rem}

\section{Robustness for high-dimensional and noisy datasets}\label{sec_robustnessofresults}
In this section, we investigate the robustness of GRAB-MDM and the associated bandwidth selection procedure, focusing on the high-dimensional noisy model introduced in \eqref{sec_datasetupsignalplusnoise}. We impose the following assumption.

\begin{assu}\label{assu_noiseassumption}
Assume that for each $1 \leq i \leq \sK$, there exists a nonnegative constant $\upsilon_i>0$ such that 
\begin{equation}\label{eq_dimensiontwo}
p_i \asymp n^{\upsilon_i}. 
\end{equation}
Assume $\{\bm{\xi}_i^{\ell}\}\subset\mathbb{R}^{p_\ell}$ are centered, independent and sub-Gaussian random vectors so that
\begin{equation*}
\operatorname{Cov}(\bm{\xi}_i^\ell)=\sigma_\ell^2 \mathbf{I},
\end{equation*}  
for some constants $\sigma_{\ell}>0, 1 \leq \ell \leq \sK$.
Further, we assume that  for each $1 \leq i \leq \sK$, the embedded manifold $\iota_i(\mathcal{M})$ is centered at $0$ in the sense that $\int_{\iota_i(\mathcal{M})}x \mathrm{d}\mathbb{P}_i(x)=0$.
\end{assu}

Note that in (\ref{eq_dimensiontwo}), we adopt a more general setup, following \cite{MR3183577}, without requiring the dimensionalities of different views to be comparable to one another or to the sample size (i.e., without assuming $\upsilon_i \equiv 1$), a condition often imposed in high-dimensional statistics \cite{yao2015large}.
The centralization assumption of the embedded manifold is mild, as GRAB-MDM depends only on pairwise distances, so any global translation leaves the output unchanged. This assumption is introduced mainly for notational convenience.

Next, we explain that how the analysis of our algorithms with (\ref{sec_datasetupsignalplusnoise}) reduces to that of the spiked covariance matrix model \cite{johnstone2001distribution}. For a similar discussions for the single-view case, see \cite[Section A.1]{9927456}. For $1 \leq  \ell \leq \sK$ and all $1 \leq i \leq n,$ there exists an rotation matrix $O_\ell \in O(p_\ell)$ and $r_\ell\geq d$ so that  
\begin{equation}\label{eq_rotation}
e_j^\top O_{\ell} X_\ell=0, 
\end{equation}
for all $j>r_\ell$. Here, $r_\ell$ is the dimension of the subspace that $\iota_\ell(\mathcal{M})$ is supported in, which is assumed to be fixed as $n$ grows.
Denote the covariance matrix of the first $r_\ell$ entries of $O_{\ell} X_\ell$ as $\Sigma_\ell \in \mathbb{R}^{r_\ell \times r_\ell}$, and denote its spectral decomposition as 
$$
\Sigma_\ell=U_\ell \Lambda_\ell U_\ell^\top.
$$ 
If we further construct a matrix $\overline{U}_\ell \in \mathbb{R}^{p_\ell \times p_\ell}$ that
\begin{equation*}
\overline{U}_\ell:=
\begin{pmatrix*}
U_\ell & \mathbf{0} \\
\mathbf{0} & \mathbf{I}
\end{pmatrix*}. 
\end{equation*}
Now we apply the above matrix to (\ref{eq_rotation}) and denote 
\[
\overline{U}_\ell^\top O_\ell \xb_i^\ell=(\mathring{x}_i^\ell(1),\cdots, \mathring{x}_i^\ell(r_\ell), 0, \cdots,0).
\] 
It is clear that the covariance matrix of the first $r_\ell$ entries of $\overline{U}_\ell^\top O_\ell X_\ell$ is $\Lambda_\ell=\operatorname{diag}\left\{\lambda_{\ell,1}, \cdots, \lambda_{\ell,r_\ell} \right\}$. 
Due to the isotropic structure imposed in Assumption \ref{assu_noiseassumption}, one can see that the covariance structure of 
\begin{equation}\label{eq_reducedsamples}
\mathring{\yb}_i^\ell:= \overline{U}_\ell^\top O_\ell \yb_i^\ell
\end{equation}
is 
\begin{equation}\label{eq_covariancematrixsetting}
\Sigma_{\ell}=\operatorname{diag}\{\lambda_{\ell,1}+\sigma_\ell^2,\cdots, \lambda_{\ell,r_\ell}+\sigma_\ell^2, \sigma_\ell^2, \cdots, \sigma_{\ell}^2\} \in \mathbb{R}^{p_\ell \times p_\ell}. 
\end{equation}
Due to rotational invariance of the distance that 
$\|\mathring{\yb}_i^\ell- \mathring{\yb}_i^\ell\|_2^2=\| \yb_i^\ell-\yb_j^\ell \|_2^2,$ we can instead work with $\{\mathring{\yb}_i^\ell\}$ in our theoretical analysis whose covariance matrix follows the spiked model as in (\ref{eq_covariancematrixsetting}). We can therefore interpret $\sum_{j=1}^{r_\ell} \lambda_{\ell,j}$ as the total energy of the signal, and $p_\ell \sigma_\ell^2$ as the total energy of the noise. According to the kernel construction in \eqref{eq_Kmultiplicationentry}, the bandwidth should be chosen on the same order as the total energy. Based on (\ref{eq_covariancematrixsetting}), it is natural to define the {\em signal-to-noise ratio} (SNR) for the $\ell$-th view as
\begin{equation}\label{eq_SNRdefinition}
\SNR_\ell:=\frac{\sum_{i=1}^{r_\ell} \lambda_{\ell,i}}{p_\ell \sigma_\ell^2}\,,
\end{equation} 
where $1 \leq \ell \leq \sK$.

\subsection{Robustness of the proposed GRAB-MDM}\label{sec_robustnessofalgorithms}

We establish the robustness of GRAB-MDM to high-dimensional noise in two steps.
First, under an oracle assumption, we analyze how bandwidth and SNR interact, thereby demonstrating robustness at the population level.
Second, we show that Algorithm~\ref{al_bandwidthselection} provides effective bandwidth estimates in practice.
We focus on the commonly used high-SNR regime in statistics \cite{ding2023learning,9927456,yao2015large}; that is, $\SNR_\ell \gg 1$. In this regime, the signal components asymptotically dominate the noise, ensuring the robustness of GRAB-MDM against the high-dimensional noise.

Recall that $\widecheck{\mathcal{A}}$, defined after \eqref{definition clean matrix from clean data}, is the transition matrix constructed from the clean point clouds $\{\xb_i^\ell\}.$ Our goal is to relate the empirical matrix $\mathcal{A}$ to its clean counterpart $\widecheck{\mathcal{A}}$. 
To this end, we impose the following assumptions, which characterize the relationship among bandwidth, signal strength, overall noise level, and the kernel function $K$. For all $1 \leq \ell \leq \sK$, denote
\begin{equation}\label{eq_akeyquantity}
\mathsf{R}_\ell:= 2\frac{\sum_{i=1}^{r_\ell} \lambda_{\ell,i}+p_\ell \sigma_\ell^2}{\epsilon_\ell}. 
\end{equation} 
The numerator $\sum_{i=1}^{r_\ell} \lambda_{\ell,i}+p_\ell \sigma_\ell^2$ in $\mathsf{R}_\ell$ represents the {\em total energy}, which is asymptotically of the same order as $\|\yb_i^\ell\|^2_2$ {over the high-SNR regime}. Thus, $\mathsf{R}_\ell$ is {the ratio of squared diameter of embedded manifold and $\epsilon_\ell$.}  

For clarity and simplicity, we will adopt the following decay rate assumption for the remainder of the paper on the kernel function that for $t \geq 0$
\begin{equation}\label{eq_lowerboundkeykeykey}
\left| (\log K)'(t) \right|=\OO(1).
\end{equation}
It is easy to check that commonly used kernel functions, such as $K(t) = e^{-t}$ or $K(t) = (1+t)^{-\beta}$ for some large $\beta$, satisfy (\ref{eq_lowerboundkeykeykey}).    

{The bandwidth parameter $\epsilon_\ell$ is essential for our proposed algorithm, and we control it via the relationship between $\mathsf{R}_\ell$ and $K$. More specifically, we require the kernel $K$ and} $\epsilon_\ell$ to be chosen appropriately so that {asymptotically} 
\begin{equation}\label{eq_bandwidthassumption2}
K(\mathsf{R}_\ell) \gg n^{-1}.
\end{equation}
Note that \eqref{eq_bandwidthassumption2} can be easily achieved in practice. For instance, when $K(t)=e^{-t}$, since the underlying manifold is assumed to be fixed, under the high-SNR assumption, \eqref{eq_bandwidthassumption2} is equivalent to requiring that $\epsilon_\ell \gg (\log n)^{-1}$.

The result is stated in Theorem~\ref{thm_robustvmvl}. Bandwidths chosen under this assumption ensure that the discrepancy between $\mathcal{A}$ and $\widecheck{\mathcal{A}}$ is controlled by the SNR, thereby providing the foundation for our bandwidth selection strategy.

\begin{thm}\label{thm_robustvmvl} 
Suppose Assumptions \ref{assu_model}, \ref{assum_main}, \ref{assu_noiseassumption}, and (\ref{eq_lowerboundkeykeykey}) and (\ref{eq_bandwidthassumption2}) hold.   
Denote the error control parameter $\Psi$ as 
\begin{equation}\label{eq_Thetadefinition}
\Psi:=  \frac{\Psi_0}{\min_{\ell} K(\mathsf{R}_\ell) \sum_{\ell' \neq \ell} K(\mathsf{R}_{\ell'}) }, 
\end{equation}
where
\begin{equation}\label{eq_Theta0}
 \Psi_0:={(\sK-1)\max_\ell  \left(  \frac{1}{\SNR_\ell}+\sqrt{\frac{1}{\SNR_\ell p_{\ell}}}\right) +\sum_{1 \leq \ell' \leq \sK} \left(\frac{1}{\SNR_{\ell'}}+\sqrt{\frac{1}{\SNR_{\ell'} p_{\ell'}}} \right) }.
\end{equation}
Assume for some constant $0<\sfc \equiv \sfc(n)<\infty,$ the bandwidth satisfies
\begin{equation}\label{eq_conditiontwo}
\epsilon_\ell \asymp \sfc \sum_{i=1}^{r_\ell} \lambda_{\ell,i} \ \ \mbox{ and }\ \ \Psi= \mathrm{o}(\sfc).  
\end{equation}
Then we have that 
\begin{equation*}
\left \| \widecheck{\mathcal{A}}-\mathcal{A} \right \|=\mathrm{O}_{\prec} \left( \frac{\Psi}{\sfc} \right).
\end{equation*}
\end{thm}

Theorem \ref{thm_robustvmvl} shows that when the bandwidths satisfy $\epsilon_\ell \asymp \sfc \sum_{i=1}^{r_\ell} \lambda_{\ell,i}$, where $\sum_{i=1}^{r_\ell} \lambda_{\ell,i}$ is the signal strength of the $\ell$th view and $\sfc$ is a global scaling factor, the discrepancy between $\mathcal{A}$ and $\widecheck{\mathcal{A}}$ is controlled by $\frac{\Psi}{\sfc}$. 
We mention that $\sfc$ here plays the same role of $\sfc$ in Algorithm \ref{al_bandwidthselection}.
Consequently, the embeddings in \eqref{al_VCDM1} are robust to noise in the sense that the leading eigenvalues and eigenvectors of $\mathcal{A}$ and $\widecheck{\mathcal{A}}$ remain close under perturbation as long as \eqref{eq_conditiontwo} is satisfied. {The assumption $\Psi = \mathrm{o}(\sfc)$ forces $\Psi\to 0$ as $n\to \infty$, which combined with \eqref{eq_bandwidthassumption2} implies $\SNR_\ell \to \infty$ as $n\to \infty$}.

Before concluding this section, we provide a few remarks on the SNR. In this work, for concreteness, we focus on the commonly used high SNR regime in statistics, where the signal strength dominates the noise. When the underlying manifold is fixed, or equivalently when $\sum_{i=1}^{r_\ell} \lambda_{\ell,i}$ is fixed, this assumption reduces to requiring $\sigma_\ell = \mathrm{o}(p_\ell^{-1/2})$, a standard condition in statistical theory. In practice, one may wish to relax this assumption. A natural approach is to first reduce the ambient dimension $p_\ell$ to a much smaller dimension $\mathsf{s}_\ell \ll p_\ell$, so that it suffices to assume $\sigma_\ell = \mathrm{o}(\mathsf{s}_\ell^{-1/2})$. This reduction can be achieved via (random) sketching with sparse sensing matrices, motivated by the model-reduction step in (\ref{eq_reducedsamples}), which suggests that the informative signals resides in only a few coordinates. Concretely, one can consider $\mathring{\zb}_i^\ell=\mathrm{S}_\ell \mathring{\yb}_i^\ell \in \mathbb{R}^{\mathsf{s}_\ell},$ where $\mathrm{S}_\ell \in \mathbb{R}^{\mathsf{s}_\ell \times p_\ell}$ is a sparse sensing matrix such that, with high probability,
$\| \mathring{\zb}_i^\ell - \mathring{\zb}_j^\ell \| \approx \| \mathring{\yb}_i^\ell - \mathring{\yb}_j^\ell \|.$
As illustrations, Appendices~\ref{sec_practicalnoisereduction} and \ref{sec_additionalsimulationresults} of the supplement present numerical experiments showing that when {the signal fulfills some structure and} $\mathrm{S}_\ell$ is chosen as a Haar or orthogonal wavelet matrix, GRAB-MDM remain robust even under weaker SNR conditions. A detailed theoretical and practical investigation of how to design $\mathrm{S}_\ell$ and select $\mathsf{s}_\ell$ is deferred to future work.

\subsection{Effectiveness of the bandwidth selection algorithm}\label{sec_badnwidthselctionalgorithm}

To rigorously establish the robustness of GRAB-MDM, we must justify the bandwidth selection procedure in Algorithm~\ref{al_bandwidthselection}, since in practice the bandwidth cannot rely on the oracle information required to directly invoke Theorem \ref{thm_robustvmvl}. In practice, both $\sfc$ and the total signal energy $\sum_{i=1}^{r_\ell} \lambda_{\ell,i}$ are unknown and must be estimated. Because the underlying data structures are nonlinear, classical estimators based on outlier eigenvalues of sample covariance matrices \cite{ding2021spiked,johnstone2001distribution,paul2007asymptotics} are no longer applicable. To overcome this difficulty, we show that the proposed estimator $h_{\ell}$, chosen via (\ref{VMVL_bandwidth}) in Algorithm \ref{al_bandwidthselection}, provides a reliable approximation of the overall signal strength.

\begin{prop}\label{thm_bandwidth}
For any $\omega_\ell \equiv \omega_{\ell}(n) \in (0,1),$ under the assumptions of Theorem  \ref{thm_robustvmvl},  for  $h_\ell$ chosen according to (\ref{VMVL_bandwidth}) using $\{\yb_i^\ell\}$, we have that 
\begin{equation}\label{eq_conclusionone}
 \frac{h_\ell}{\sum_{i=1}^{r_\ell} \lambda_{\ell,i}} =1+\OO_\prec\left( \frac{1}{\SNR_\ell}+\sqrt{\frac{1}{\SNR_\ell p_{\ell}}}\right).
\end{equation}
\end{prop}

Proposition \ref{thm_bandwidth} demonstrates that once the signal dominates the noise so that the error term on the right-hand side of (\ref{eq_conclusionone}) vanishes, our bandwidth selection procedure successfully captures the overall signal strength. This occurs when the total noise energy is asymptotically dominated by the squared manifold diameter.
Theoretically, Proposition~\ref{thm_bandwidth} holds for any constant $\omega_\ell \in (0,1)$, indicating that the procedure is robust to the choice of $\omega_\ell$; this robustness is also confirmed numerically in Section~\ref{sec_numericalresult}.
To further automate the selection, we recommend the resampling approach of \cite{9927456}, described in Appendix~\ref{appendix_cchosen} of the supplement. 
Finally, note that the right-hand side of \eqref{eq_conclusionone} aligns with the high-SNR intuition that the total noise energy should be much smaller than the manifold diameter.

With Proposition \ref{thm_bandwidth} and \eqref{eq_conditiontwo}, for $1 \leq \ell \leq \sK,$ we may choose
\begin{equation}\label{eq_bandwidthtwo}
\hat{\epsilon}_{\ell}=\hat{\sfc} h_\ell,   
\end{equation}
where $\hat{\sfc}$ is determined according to (\ref{eq_cchosen}). {Here, $\hat{\sfc}$ serves as an estimate of $\sfc$. In practice, we propose to select $\sfc$ via the grid-search procedure in Algorithm~\ref{al_bandwidthselection}, choosing the value that satisfies the condition in Theorem \ref{thm_robustvmvl}. 
Observe that by assumption, $\min_{\ell} K(\mathsf{R}_\ell) \sum_{\ell' \neq \ell} K(\mathsf{R}_{\ell'})$ is bounded away from $0$ if $\epsilon_\ell$, and hence $\sfc$, is chosen not too small so that $\mathsf{R}_\ell$ is asymptotically of order $1$. In this regime, since $\Psi_0\to 0$ asymptotically, we also have $\Psi\to 0$ asymptotically. Therefore, we can allow $\sfc\to 0$ asymptotically as long as $\sfc$ is not taken so small that it falls below $\Psi$, the robustness result in Theorem \ref{thm_robustvmvl} guarantees recovery of the desired $\mathcal{A}$. Consequently, we propose to apply coordinate-wise z-score normalization to the datasets, and then simply specify an interval with small upper bound with a nonzero lower bound, over which we perform a grid search.}
For sufficiently large $n$, we  find empirically that the range $[10^{-3}, 0.5]$ is adequate; as demonstrated in Section \ref{sec_numericalresult}, this choice yields strong empirical performance in both simulations and real data.

\section{Numerical results}\label{sec_numericalresult}
In this section, we present extensive numerical simulations to demonstrate the superior performance of our proposed method and compare it with several existing approaches, particularly from the fields of manifold learning. For clarity, we focus on two tasks: spectral clustering, discussed in Section \ref{task_multiviewcluster}, and manifold learning, examined in Section \ref{sec_manifoldlearning}. Additional simulation results can be found in Appendix \ref{sec_additionalsimulationresults} of the supplement.  


\subsection{Multiview spectral clustering}\label{task_multiviewcluster}

In this section, we consider the problem of multi-view clustering, with a particular focus on spectral clustering. In this framework, embeddings obtained from multiple views are first integrated to form a unified representation, which is subsequently partitioned into clusters via the K-means algorithm. We begin by describing the simulation setup used for dataset generation, following the procedure outlined in Section~\ref{sec_datasetupsignalplusnoise}. Specifically, we examine the case of three views, i.e., we set $\sK=3$.

We begin by generating synthetic datasets with well-defined cluster structures, as described below. For $1 \leq i \leq n:=3 \mathsf{n},\, \mathsf{n}\in \mathbb{N}$ we consider that 
\begin{equation}\label{eq_truecluter} 
\mathbf{w}_i \sim \sum_{j=1}^3 \mathbf{1}_{ (j-1)\mathsf{n} +1 \leq i \leq j\mathsf{n}} \mathbf{a}_{j}, 
\end{equation}
where $\mathbf{a}_j \in \mathbb{R}^{d}, \ d=3 \mathsf{d}, \ \mathsf{d}\in \mathbb{N}, \ 1 \leq j \leq 3,$ are some independent random vectors with density function $\rho_j$ which will be specified below. For convenience, decompose each $\mathbf{a}_j$ into three blocks:
\begin{equation*}
\mathbf{a}^\top_j=(\mathbf{a}_{j1}^\top, \mathbf{a}_{j2}^\top, \mathbf{a}_{j3}^\top) \in \mathbb{R}^{3 \mathsf{d}}, \ \mathbf{a}_{j \ell} \in \mathbb{R}^{\mathsf{d}} \ \text{for} \ 1 \leq \ell \leq 3,
\end{equation*}
so that $\mathbf{w}^\top_i=(\mathbf{w}_{i1}^\top, \mathbf{w}_{i2}^\top, \mathbf{w}_{i3}^\top) \in \mathbb{R}^{3 \mathsf{d}}.$ 
We now describe how to generate the datasets in accordance with (\ref{sec_datasetupsignalplusnoise}). For some function $\mathsf{g}_\ell \in \mathbb{R}$, the clean signals are constructed for $1 \leq \ell \leq 3$ as follows
\begin{equation*}
\xb_i^\ell=
\begin{pmatrix}
\mathsf{g}_\ell \odot \mathbf{w}_{i\ell}\\
\mathbf{0}_{p_1-\mathsf{d}}
\end{pmatrix}
 \in \mathbb{R}^{p_\ell}, 
\end{equation*}
where $\odot$ means entrywise operation. In what follows, we choose $\mathsf{g}_1(x)=x^2+x, \ \mathsf{g}_2(x)=10 \log(x+2), \ \mathsf{g}_3(x)=0.8x.$
Regarding $\rho_j, 1 \leq j \leq 3,$ we consider the following two setups in our simulations. 
\begin{enumerate}
\item[(1).] $\rho_j(\bm{x})=\prod_{k=1}^d \mathbf{1}(x_k \in [a_{jk},b_{jk}]) \frac{1}{b_{jk}-a_{jk}}, \ \bm{x}=(x_1, \cdots, x_d)^{\top}$ is the density function of the uniform distribution on the Cartesian product space $\prod_{k=1}^d [a_{jk}, b_{jk}].$ In our simulations, we consider
\begin{equation*}
a_{1k}=
\begin{cases}
2.5, & 1 \leq k \leq \mathsf{d} \\
1, & \mathsf{d}+1 \leq k \leq 2\mathsf{d} \\
1, & 2\mathsf{d}+1 \leq k \leq 3\mathsf{d}
\end{cases},  \ \
b_{1k}=
\begin{cases}
3.5, & 1 \leq k \leq \mathsf{d} \\
5, & \mathsf{d}+1 \leq k \leq 2\mathsf{d} \\
3, & 2\mathsf{d}+1 \leq k \leq 3\mathsf{d}
\end{cases}, 
\end{equation*}
\begin{equation*}
a_{2k}=
\begin{cases}
2.5, & 1 \leq k \leq \mathsf{d} \\
2, & \mathsf{d}+1 \leq k \leq 2\mathsf{d} \\
1, & 2\mathsf{d}+1 \leq k \leq 3\mathsf{d}
\end{cases},  \ \
b_{2k}=
\begin{cases}
3.5, & 1 \leq k \leq \mathsf{d} \\
6, & \mathsf{d}+1 \leq k \leq 2\mathsf{d} \\
3, & 2\mathsf{d}+1 \leq k \leq 3\mathsf{d}
\end{cases}, 
\end{equation*}
\begin{equation*}
a_{3k}=
\begin{cases}
1.5, & 1 \leq k \leq \mathsf{d} \\
3, & \mathsf{d}+1 \leq k \leq 2\mathsf{d} \\
-1, & 2\mathsf{d}+1 \leq k \leq 3\mathsf{d}
\end{cases},  \ \
b_{3k}=
\begin{cases}
2.5, & 1 \leq k \leq \mathsf{d} \\
7, & \mathsf{d}+1 \leq k \leq 2\mathsf{d} \\
3, & 2\mathsf{d}+1 \leq k \leq 3\mathsf{d}
\end{cases}. 
\end{equation*}
\item[(2).] $\rho_j(\bm{x})=\prod_{k=1}^d \varphi_{jk}(x_{k}),$ where $\varphi_{jk}$ are the density functions of independent truncated Gaussian random variables with parameters $(\mu_{jk},\sigma^2,\mu_{jk}-3\sigma, \mu_{jk}+3\sigma).$ Here $\mu_{jk}$ and $\sigma^2$ are the mean and variance of the Gaussian random variables and   $\mu_{jk}-3\sigma$ and $\mu_{jk}+3\sigma$ are the truncation levels.  In our simulations, we use that $\sigma^2=0.8$ and 
\begin{equation*}
\mu_{1k}=
\begin{cases}
3, & 1 \leq k \leq \mathsf{d} \\
3, & \mathsf{d}+1 \leq k \leq 2\mathsf{d} \\
2, & 2\mathsf{d}+1 \leq k \leq 3\mathsf{d}
\end{cases},  \ \
\mu_{2k}=
\begin{cases}
3, & 1 \leq k \leq \mathsf{d} \\
5, & \mathsf{d}+1 \leq k \leq 2\mathsf{d} \\
2, & 2\mathsf{d}+1 \leq k \leq 3\mathsf{d}
\end{cases},  \ \
\mu_{3k}=
\begin{cases}
2, & 1 \leq k \leq \mathsf{d} \\
5.5, & \mathsf{d}+1 \leq k \leq 2\mathsf{d} \\
2, & 2\mathsf{d}+1 \leq k \leq 3\mathsf{d}
\end{cases}.
\end{equation*}
\end{enumerate}  
Finally, for the noise, we consider 
\begin{equation}\label{eq_noiselevel}
\bm{\xi}_i^\ell \sim \mathcal{N}(\mathbf{0}_{p_\ell}, \upsilon^2_\ell \mathbf{I}_{p_\ell})\,,
\end{equation} 
where $\upsilon_\ell>0$. In the above setup, the raw dataset contains three different clusters characterized by $\rho_j, 1 \leq j \leq 3$ and we observe three different views using three nonlinear transforms $\mathsf{g}_{\ell},$ $1 \leq \ell \leq \sK=3.$

With the above simulated datasets, we evaluate clustering performance, measured by clustering accuracy (ACC) and the Rand Index (RI) \cite{rand1971objective}. Since our data consist of three views, methods restricted to two views, such as various CCA-based approaches or alternating diffusion maps, are not applicable. For concreteness, we focus on our proposed method (GRAB-MDM) and compare it with several existing manifold-learning-based multiview methods, including multiview diffusion maps (mDM) \cite{LINDENBAUM2020127}, alternating diffusion maps applied to every pair (pADM) \cite{katz2019alternating}, multiview locally linear embedding (mLLE) \cite{rodosthenous2024multi}, multiview t-stochastic neighbor embedding (mSNE) \cite{rodosthenous2024multi}, multiview ISOMAP (mISOMAP) \cite{rodosthenous2024multi}, and multiview UMAP (mUMAP) \cite{do2021_generalization_tsne_umap}. Moreover, to highlight the advantages of learning using multiviews, we also compare the results with those obtained by treating all concatenated datasets as a single view, using standard manifold learning methods such as diffusion maps (DM) \cite{coifman2006diffusion}, locally linear embedding (LLE) \cite{roweis2000_nonlinear_lle}, t-SNE \cite{van2008tsne}, ISOMAP \cite{tenenbaum2000_global_geometric}, and UMAP \cite{McInnes2018}. We note that deep-learning-based approaches have also been proposed in the literature to address similar problems, albeit without explicitly incorporating geometric manifold structures. While we do not report detailed results for such black-box methods, we tested them in our simulations and found that they generally perform poorly, particularly in noisy settings. 

In Table \ref{tab_comparison_1}, we compare the clustering accuracy (ACC) of our method against other approaches across different noise levels for both setups (1) and (2), with $\mathsf{d}=10$, $\mathsf{n}=200$, and $p_1=p_2=p_3=100$. For our method, we select $m=1$ based on the approach as discussed in Appendix \ref{appendix_cchosen} of the supplement. Several observations can be made. First, overall, multiview-based methods outperform the simple concatenation of datasets into a single view in the noisy setup. Second, our proposed GRAB-MDM consistently outperforms all other methods, particularly in high-noise settings. In Table \ref{tab_comparison_3} of the supplement, we also report the results for the Rand index, which lead to similar conclusions. To further demonstrate the superiority of our method, Figure \ref{fig:results} compares the performance of various approaches as the noise level gradually increases. In particular, we set $\upsilon_\ell \equiv \sigma$, $1 \leq \ell \leq 3$, in (\ref{eq_noiselevel}) and use Setup (1) with $\mathsf{d}=10$, $\mathsf{n}=200$, and $p_1=p_2=p_3=100$. The results show that our method consistently outperforms all others, especially as the noise level becomes larger.
\begin{table}[!ht]
\centering
\renewcommand{\arraystretch}{1.9}
\setlength{\tabcolsep}{6pt}
\resizebox{\textwidth}{!}{
\begin{tabular}{c|c|c|c|c|c|c|c|c|c|c|c|c} \hline
\textbf{Triplet / Method} & Proposed & mDM & mSNE & mLLE & mISOMAP & mUMAP& pADM & DM & LLE & tSNE & ISOMAP & UMAP \\
\cline{1-13}
\multicolumn{13}{c}{\bf Setup (1)} \\
\cline{1-13}
$(\upsilon_1^2,\upsilon_2^2,\upsilon_3^2)=(0.3,0.1,0.3)$ & 1 (0.04) & 0.99 (0.03) & 0.9 (0.04) & 0.91 (0.04) &  0.92 (0.04) & 0.89 (0.04)&  0.98 (0.01) & 0.94 (0.07) & 0.91 (0.03) & 0.93 (0.07) & 0.91 (0.03) & 0.9 (0.02) \\
$(\upsilon_1^2,\upsilon_2^2,\upsilon_3^2)=(3,2,3)$ & 0.95 (0.01) & 0.91 (0.01) & 0.69 (0.03) & 0.69 (0.03) & 0.68 (0.02) & 0.75 (0.05) &    0.8 (0.02) &0.88 (0.02)  & 0.77 (0.07) & 0.86 (0.03) & 0.91 (0.01) & 0.85 (0.04)\\
$(\upsilon_1^2,\upsilon_2^2,\upsilon_3^2)=(5,7,5)$ & 0.89 (0.01) & 0.82 (0.01) & 0.68 (0.02) & 0.67 (0.02) & 0.67 (0.02) & 0.71 (0.04) &  0.64 (0.01)& 0.86 (0.01) & 0.68 (0.06) & 0.83 (0.03) & 0.89 (0.01) & 0.81 (0.03)\\
$(\upsilon_1^2,\upsilon_2^2,\upsilon_3^2)=(10,10,10)$ & 0.84 (0.01) & 0.81 (0.01) & 0.66 (0.01) & 0.64 (0.02) & 0.65 (0.01) & 0.67 (0.03) &  0.54 (0.01) & 0.82 (0.02) & 0.66 (0.05) & 0.70 (0.03) & 0.82 (0.03) & 0.72 (0.04) \\ 
$(\upsilon_1^2,\upsilon_2^2,\upsilon_3^2)=(20,10,45)$ & 0.81 (0.02) & 0.73 (0.02) & 0.64 (0.02) & 0.60 (0.02) & 0.61 (0.02) & 0.62 (0.02) & 0.53 (0.01) & 0.67 (0.02) & 0.37 (0.04) & 0.56 (0.03) & 0.65 (0.02) & 0.57 (0.03) \\
$(\upsilon_1^2,\upsilon_2^2,\upsilon_3^2)=(10,25,40)$ & 0.79 (0.02) & 0.72 (0.02) & 0.65 (0.02) & 0.61 (0.02) & 0.62 (0.02) & 0.66 (0.01) & 0.51 (0.03) & 0.67 (0.02) & 0.37 (0.05) & 0.58 (0.03) & 0.66 (0.01) & 0.58 (0.03) \\
$(\upsilon_1^2,\upsilon_2^2,\upsilon_3^2)=(10,30,50)$ & 0.77 (0.02) & 0.7 (0.02) & 0.65 (0.02) & 0.61 (0.02) & 0.61 (0.02) & 0.66 (0.01) & 0.49 (0.03) & 0.65 (0.01) & 0.35 (0.02) & 0.51 (0.05) & 0.64 (0.02) & 0.52 (0.04) \\ \hline
\cline{1-13}
\multicolumn{13}{c}{\bf Setup (2)} \\
\cline{1-13}
$(\upsilon_1^2,\upsilon_2^2,\upsilon_3^2)=(0.5,0,0.3)$ & 0.99 (0.03) & 0.98 (0.02) & 0.89 (0.05) & 0.91 (0.03) &  0.89 (0.05) & 0.88 (0.06) &0.84 (0.01) & 0.9 (0.05) & 0.87 (0.02) & 0.91 (0.03) & 0.9 (0.08) & 0.89 (0.03) \\
$(\upsilon_1^2,\upsilon_2^2,\upsilon_3^2)=(3,1,2)$ & 0.97 (0.01) & 0.96 (0.02) & 0.71 (0.05) & 0.64 (0.04) &  0.69 (0.04) & 0.7 (0.04)  &0.65 (0.01) & 0.78 (0.07) & 0.66 (0.03) & 0.79 (0.08) & 0.75 (0.06) & 0.85 (0.06) \\
$(\upsilon_1^2,\upsilon_2^2,\upsilon_3^2)=(9,5,4)$ & 0.92 (0.01) & 0.9 (0.07) & 0.60 (0.01) & 0.61 (0.01) &  0.61(0.03) & 0.67 (0.03) &  0.53 (0.01) & 0.74 (0.06) & 0.66 (0.01) & 0.73 (0.06) & 0.70 (0.05) & 0.76 (0.07) \\
$(\upsilon_1^2,\upsilon_2^2,\upsilon_3^2)=(10,12,8)$ & 0.85 (0.02) & 0.75 (0.02) & 0.65 (0.03) & 0.56 (0.03) & 0.59 (0.04) & 0.61 (0.03) &0.53 (0.05) & 0.69 (0.04) & 0.64 (0.04) & 0.66 (0.03) & 0.62 (0.04) & 0.67 (0.03) \\
$(\upsilon_1^2,\upsilon_2^2,\upsilon_3^2)=(30,8,50)$ & 0.83 (0.01) & 0.72 (0.01) & 0.62 (0.03) & 0.58 (0.03) & 0.56 (0.03) & 0.61 (0.02)  &0.51 (0.03) & 0.62 (0.03) & 0.38 (0.06) & 0.51 (0.03) & 0.61 (0.01) & 0.51 (0.03) \\
$(\upsilon_1^2,\upsilon_2^2,\upsilon_3^2)=(50,15,70)$ & 0.77 (0.02) & 0.67 (0.02) & 0.54 (0.03) & 0.48 (0.03) & 0.42 (0.03) & 0.56 (0.02)  &0.52 (0.01) & 0.56 (0.02) & 0.35 (0.01) & 0.43 (0.04) & 0.55 (0.03) & 0.42 (0.03) \\
$(\upsilon_1^2,\upsilon_2^2,\upsilon_3^2)=(15,30,50)$ & 0.72 (0.03) & 0.62 (0.03) & 0.58 (0.03) & 0.53 (0.03) & 0.51 (0.04) & 0.62 (0.02)  &0. 53 (0.01) & 0.63 (0.02) & 0.35 (0.01) & 0.51 (0.04) & 0.62 (0.02) & 0.50 (0.03) \\
\cline{1-13}
\end{tabular}
}
\caption{Comparison of clustering accuracy across various methods. The reported values are the means with standard deviations (in parentheses), based on 1,000 replications.}
\vspace{0.5em}
\label{tab_comparison_1}
\end{table}

\begin{figure}[!ht]
    \centering
    \includegraphics[width=10cm, height=5.5cm]{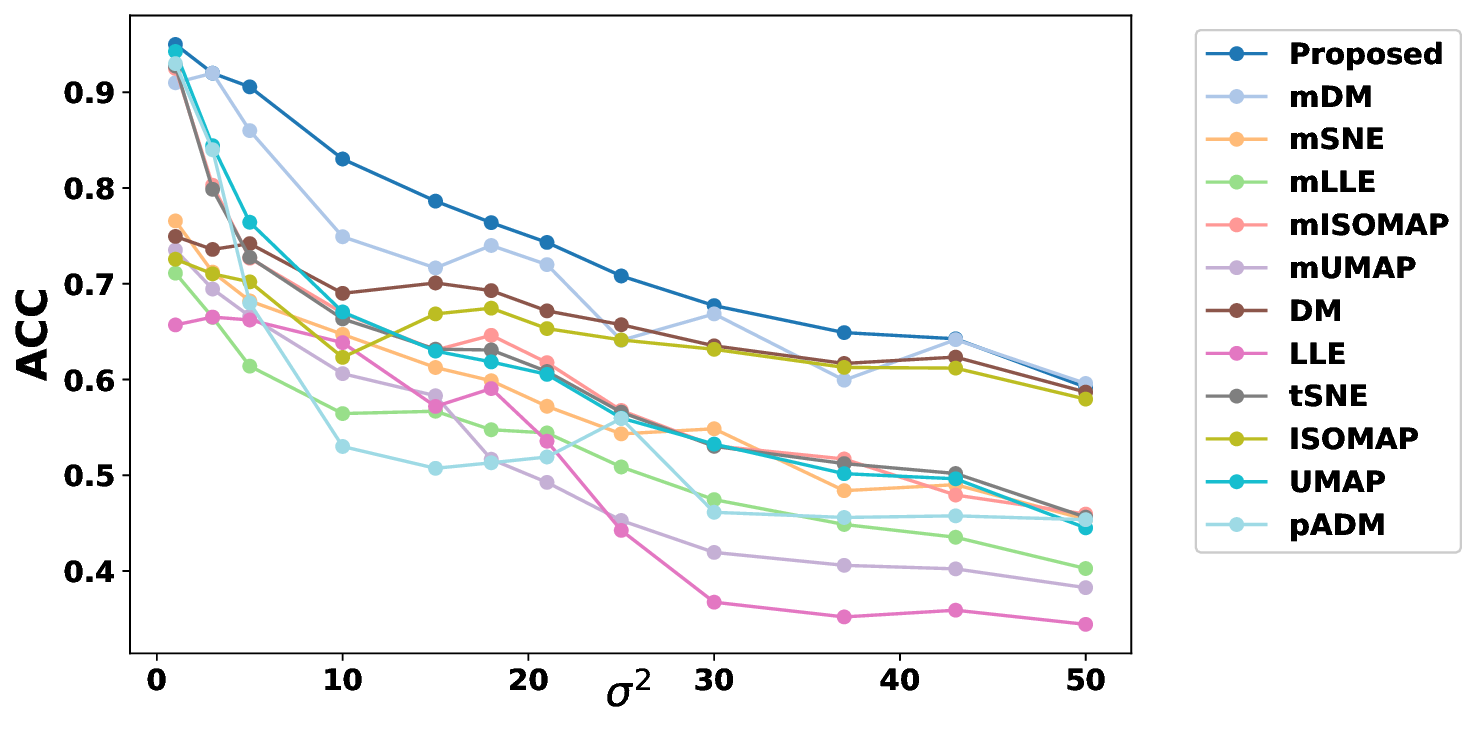}
    \caption{Comparison of the robustness of various methods. We use Setup (1), and the results are based on 1,000 repetitions.}
    \label{fig:results}
\end{figure}



\subsection{Manifold learning}\label{sec_manifoldlearning}
In this section, we compare the performance of different methods by assessing the quality of their low-dimensional embeddings in terms of neighborhood preservation, as measured by the trustworthiness metric \cite{pmlr-v5-maaten09a,venna2001neighborhood}. To facilitate comparison with existing two-view methods and streamline our discussion, we focus on the two-view setup. Specifically, for samples $\mathbf{s}_i \in \mathbb{R}^d$, $1 \leq i \leq n$, independently drawn from certain geometric objects, we construct the clean samples as described in (\ref{sec_datasetupsignalplusnoise}) so that
\begin{equation*}
\xb_i^\ell \equiv 
\begin{pmatrix}
\mathbf{s}_{i}\\
\mathbf{0}_{p-d}
\end{pmatrix}
 \in \mathbb{R}^{p}, \ \text{for} \ \ell=1,2.  
\end{equation*}   
To construct the noisy datasets, we choose
\begin{equation*}
\bm{\xi}_i^\ell \sim \mathcal{N}(\mathbf{0}_{p}, \upsilon^2_\ell \mathbf{I}_{p}), \ \upsilon_1<\upsilon_2.  
\end{equation*}
That is, view 1 corresponds to a relatively clean signal, whereas view 2 represents a noisier signal, though both share the same underlying clean structure which will be specified later. The trustworthiness is computed with respect to the underlying geometric objects using the samples $\{\mathbf{s}_{i}\}$. 
%
For simplicity and ease of comparison, we set $m=d$ in this application. Therefore, the trustworthiness  \cite{pmlr-v5-maaten09a} here is defined as 
\begin{equation*}
T(k)=1-\frac{2}{nk(2n-3k-1)} \sum_{i=1}^n \sum_{j \in U_k(i) \backslash V_k(i)} (r(i,j)-k).
\end{equation*} 
Here, $k$ denotes the neighborhood size, $U_k(i)$ is the set of $k$ nearest neighbors of point $i$ in the embedding, and $V_k(i)$ is the set of $k$ nearest neighbors of $i$ in the original data $\{\mathbf{s}_i\}$. The quantity $r(i,j)$ represents the rank of point $j$ in the ordering of neighbors of $i$ in $\{\mathbf{s}_i\}$ based on Euclidean distance. In our simulations, we set $k=5$, a commonly used choice in practice.

In what follows, for the clean geometric objects, we consider the following two setups as follows.
\begin{enumerate}
\item[(a).]  $\mathbf{s}_i, 1 \leq i \leq n,$ are independently and uniformly sampled from the Swiss roll in $\mathbb{R}^3.$
\item[(b).] $\mathbf{s}_i, \; 1 \leq i \leq n:=3\mathsf{n},$ are independently and uniformly sampled according to
\[
\mathbf{s}_i \sim 
\begin{cases}
\text{Unit sphere centered at } (0,0,0), & 1 \leq i \leq \mathsf{n}, \\[6pt]
\text{$S$-curve centered at } (0,1.5,0), & \mathsf{n}+1 \leq i \leq 2\mathsf{n}, \\[6pt]
\text{Swiss roll centered at } (1,2.5,0), & 2\mathsf{n}+1 \leq i \leq 3\mathsf{n}.
\end{cases}
\]
\end{enumerate}

For comparison, in addition to the methods reported in Table \ref{tab_comparison_2}, we also consider nonparametric canonical correlation analysis (NCCA) \cite{michaeli2016nonparametric}, kernel canonical correlation analysis (KCCA) \cite{bach2002kernel}, and alternative diffusion maps \cite{lederman2018geometry}, using the embeddings in (\ref{eq_useaverger}). Table \ref{tab_comparison_2} reports the trustworthiness of our proposed method against these alternatives across different noise levels for both Setups (a) and (b), with $n=600$ and $p=100.$ For consistency, we set $m=3$ in all simulations, as also recommended by our method described in Appendix \ref{appendix_cchosen} of the supplement. The results show that our proposed approach achieves overall better performance, particularly when the noise level is higher and the underlying geometric structure is more complex, as in Setup (b).

\begin{table}[!ht]
\centering
\renewcommand{\arraystretch}{2}
\setlength{\tabcolsep}{8pt}
\resizebox{\textwidth}{!}{
\begin{tabular}{c|c|c|c|c|c|c|c|c|c|c|c|c|c|c} \hline
\textbf{Pair / Method} & Proposed & mDM & mSNE & mLLE & mISOMAP & mUMAP& KCCA & NCCA & ADM & DM & LLE & tSNE & ISOMAP & UMAP \\
\cline{1-15}
\multicolumn{15}{c}{\bf Setup (a)} \\
\cline{1-15}
$(\upsilon_1^2,\upsilon_2^2)=(0.05,0.2)$ & 0.98 (0.001) & 0.94 (0.006) & 0.91 (0.005) & 0.89 (0.003) &  0.91(0.004) & 0.9 (0.006) & 0.91 (0.007) & 0.92 (0.003) & 0.93 (0.008) & 0.93 (0.002) & 0.93 (0.007) & 0.89 (0.005) & 0.92(0.007) & 0.9 (0.006) \\
$(\upsilon_1^2,\upsilon_2^2)=(0.1,0.3)$ &0.94(0.003) & 0.91(0.003) & 0.86(0.007) & 0.84(0.010) & 0.9(0.004) & 0.89(0.005) & 0.9(0.004) & 0.9(0.003) & 0.91(0.003) & 0.89(0.005) & 0.88(0.027) & 0.86(0.006) & 0.89(0.005) & 0.88(0.006) \\
$(\upsilon_1^2,\upsilon_2^2)=(0.2,0.6)$& 
0.88(0.006) & 0.83(0.005) & 0.75(0.011) & 0.73(0.017) & 0.83(0.006) & 0.8(0.007) & 0.82(0.006) & 0.83(0.005) & 0.83 (0.006) & 0.79(0.006) & 0.69(0.045) & 0.77(0.008) & 0.79(0.007) & 0.78(0.007) \\
$(\upsilon_1^2,\upsilon_2^2)=(0.3,0.8)$&  0.8 (0.007) & 0.72 (0.041) & 0.69(0.013) & 0.67(0.021) & 0.77(0.009) & 0.74(0.009) & 0.76(0.008) & 0.77(0.008) & 0.77(0.008) & 0.74(0.009) & 0.61(0.026) & 0.72(0.01) & 0.74(0.01) & 0.72(0.01) \\
$(\upsilon_1^2,\upsilon_2^2)=(0.45,0.5)$& 
0.78(0.008) & 0.71(0.035) & 0.68(0.008) & 0.67(0.020) & 0.74(0.008) & 0.73(0.008) & 0.7(0.007) & 0.7(0.008) & 0.7(0.009) & 0.79(0.005) & 0.64(0.041) & 0.74(0.008) & 0.79(0.006) & 0.75(0.007)  \\
\cline{1-15}
\multicolumn{15}{c}{\bf Setup (b)} \\
\cline{1-15}
$(\upsilon_1^2,\upsilon_2^2)=(0.05,0.1)$ &
0.99(0.001) & 0.96(0.012) & 0.98 (0.001) & 0.98(0.001) & 0.98(0.001) & 0.99 (0.001) & 0.99(0.001) & 0.99(0.001) & 0.9(0.023) & 0.98(0.001) & 0.98(0.001) & 0.98(0.001) & 0.99(0.000) & 0.98(0.001)\\
$(\upsilon_1^2,\upsilon_2^2)=(0.2,0.5)$ & 
0.94(0.003) & 0.92(0.018) & 0.91(0.004) & 0.91(0.005) & 0.93(0.003) & 0.93(0.004) & 0.94(0.003) & 0.94(0.003) & 0.86(0.026) & 0.94(0.002) & 0.93(0.004) & 0.91(0.004) & 0.94(0.003) & 0.92(0.004) \\
$(\upsilon_1^2,\upsilon_2^2)=(0.4,0.9)$& 
0.88(0.005) & 0.82(0.023) & 0.84(0.008) & 0.83(0.008) & 0.81(0.005) & 0.82(0.006) & 0.83(0.005) & 0.81(0.005) & 0.81(0.018) & 0.83(0.005) & 0.81(0.017) & 0.83(0.006) & 0.81(0.005) & 0.79(0.005) \\
$(\upsilon_1^2,\upsilon_2^2)=(0.8,1)$& 
0.83(0.006) & 0.77(0.025) & 0.78(0.008) & 0.79(0.010) & 0.8(0.007) & 0.8(0.006) & 0.8(0.006) & 0.8(0.006) & 0.75(0.010) & 0.77(0.005) & 0.81(0.023) & 0.8(0.007) & 0.8(0.005) & 0.78(0.007) \\
$(\upsilon_1^2,\upsilon_2^2)=(1,1)$& 
0.8(0.006) & 0.72(0.025) & 0.7(0.008) & 0.74(0.010) & 0.71(0.007) & 0.77(0.006) & 0.71(0.006) & 0.73(0.006) & 0.71(0.010) & 0.7(0.005) & 0.71(0.023) & 0.74(0.007) & 0.73(0.005) & 0.72(0.007) \\
 \hline
\end{tabular}
}
\caption{Comparison of trustworthiness of low-dimensional embedding of various methods. Reported are the mean and standard deviation (in parentheses) based on 1,000 data replications. }
\vspace{0.5em}
\label{tab_comparison_2}
\end{table}

\setcounter{equation}{0}
\numberwithin{equation}{section}

\appendix

\section{Some technical preparation}
In this appendix, we summarize some preliminary results which will be used for our technical proofs.

\subsection{Some preliminary results for calculations on the manifolds}
We present preliminary results on the asymptotic analysis of certain kernels on manifolds, which serve as essential tools for establishing the results in Section \ref{section asymptotic analysis}. The results have been established in the literature, for example, see \cite{singer2017spectral}.  For the kernel function $K$ satisfying Assumption \ref{assum_main} and some constant $\epsilon>0,$ recall the definition of $K_\epsilon(x,y)$ in \eqref{eq_kernelexplicitform} when $x,y$ are in the same space.

\begin{lem}\cite[Lemma B.1]{singer2017spectral}\label{lemma:trunc}
Suppose Assumptions \ref{assu_model} and \ref{assum_main} hold for some manifold $\mathcal{M}$ and some embedding $\iota$.	Assume $f\in C^{4}(\iota(\mathcal{M}))$   and choose some fixed constant $\gamma \in (0, 1/2)$. Then for any $x\in \iota(\mathcal{M})$, when $\epsilon$ is sufficiently small, we have that 
	$$
	\left|\int_{\mathcal{M}\backslash\widetilde{\mathcal{B}}_{\epsilon^{\gamma}}(x)}\epsilon^{-d/2}K_{\epsilon}(x,\iota(y))f(\iota(y))\,\mathrm{d}V(y)\right|=\rO(\epsilon^{2})\,,
	$$
where $\widetilde{\mathcal{B}}_{\epsilon^{\gamma}}(x)\coloneqq\iota^{-1}(\mathcal{B}_{\epsilon^{\gamma}}(x)\cap\iota(\mathcal{M}))\subseteq \mathcal{M}$, where $\mathcal{B}_{\epsilon^{\gamma}}(x)$ is the Euclidean ball with radius $\epsilon^{\gamma}$ centered at $x$.  
\end{lem}

\begin{lem}\cite[Lemma B.3]{singer2017spectral}\label{lemma:for_numerator_est}
Suppose Assumptions \ref{assu_model} and \ref{assum_main} hold for some manifold $\mathcal{M}$ and some embedding $\iota$. 	If $g\in C^{4}(\iota(\mathcal{M}))$, then for all $x\in \iota(\mathcal{M})$, we have
	$$
	\int_{\mathcal{M}}\epsilon^{-d/2}K_{\epsilon}(x,\iota(y))g(\iota(y))\,\mathrm{d}V(y)=g(x)+\frac{\epsilon\mu_{1,2}^{(0)}}{2d}(\Delta g(x)-\sw(x)g(x))+\rO(\epsilon^{2})\,,
	$$
	where $\sw(x)$ is defined in (\ref{eq_wxdefinition}) for $\iota(\mathcal{M})$. 
\end{lem}

By plugging $g(x)=f(x) \varrho(x)$ into Lemma \ref{lemma:for_numerator_est}, we obtain
\begin{align}
	\int_{\mathcal{M}}\epsilon^{-d/2}K_{\epsilon}(x,\iota(y))& f(\iota(y))\varrho(\iota(y))\,\mathrm{d}V(y) \nonumber \\
	 =&\varrho(x)f(x) +\frac{\epsilon\mu_{1,2}^{(0)}}{2d}\left(f(x)\Delta \varrho(x)+\varrho(x)\Delta f(x)
	 +2\nabla \varrho(x)\cdot\nabla f(x) -\sw(x)f(x)\varrho(x)\right) \nonumber	\\ 
	& + \rO(\epsilon^{2}).  \label{lemma:numerator_est}
	\end{align}
when $f\in C^{4}(\iota(\mathcal{M}))$.

\section{Proofs of results in Section \ref{section asymptotic analysis}} 

\subsection{Bias analysis: proof of Theorem \ref{thm_bias_vmvl}}

Throughout the proof, for notional simplicity, we will use the short-handed notations for (\ref{eq_kernelexplicitform}) that for some bandwidths $\epsilon_i$ and $\epsilon_j$ 
\begin{equation}\label{eq_shorthandnotationkikj}
K_i \equiv K_{\epsilon_i}, \ K_j \equiv K_{\epsilon_j}.
\end{equation}
Moreover, we denote
\begin{equation}\label{eq_di}
d_i(x)=\sum_{j \neq i} \int_{\iota_j(\mathcal{M})} \widetilde{K}_{ij}(x,z) \rho_j(z) \mathrm{d} V_j(z). 
\end{equation}

\begin{proof}[\bf Proof of Theorem \ref{thm_bias_vmvl}] {For notional simplicity, till the end of the proof, we denote
\begin{equation}\label{eq_uij}
\sfu_{ij}=\left( \frac{\epsilon_j}{\epsilon_i} \right)^d.
\end{equation}
By the assumption that $\epsilon_i \asymp \epsilon_j,$ we have that $\sfu_{ij} \asymp 1.$
}According to Definition \ref{defn_VCDM}, as well as (\ref{eq_widetildekij}) and (\ref{eq_omegaij}), we have that for $1 \leq i \leq \sK$ and $x \in \iota_i(\mathcal{M}),$ {
\begin{align}\label{eq_expansionvmvl}
    [\Omega \fb ]_i(x)
   & = \frac{\epsilon_i^{-d} \sum_{j \neq i} \epsilon_j^{-d} \epsilon_j^d \int_{ \iota_j(\mathcal{M})} \widetilde{K}_{ij}(x,z) f_j(z) \varrho_j(z) \mathrm{d} V_j(z)}{\epsilon_i^{-d}\sum_{j \neq i}\epsilon_j^{-d} \epsilon_j^d \int_{\iota_j(\mathcal{M})} \widetilde{K}_{ij}(x,z) \varrho_j(z) \mathrm{d} V_j(z)} \nonumber \\
   & =\frac{\sum_{j \neq i} \sfu_{ij}  \epsilon_j^{-d} \int_{ \iota_j(\mathcal{M})} \widetilde{K}_{ij}(x,z) f_j(z) \varrho_j(z) \mathrm{d} V_j(z)}{\sum_{j \neq i} \sfu_{ij} \epsilon_j^{-d} \int_{\iota_j(\mathcal{M})} \widetilde{K}_{ij}(x,z) \varrho_j(z) \mathrm{d} V_j(z)}.
\end{align} }
According to (\ref{eq_widetildekij}) and Fubini's theorem, we have that  
\begin{align}\label{eq_keyrepresentationone1111}
&\epsilon_j^{-d}\int_{\iota_j(\mathcal{M})}  \widetilde{K}_{ij}(x,z) f_j(z) \varrho_j(z) \mathrm{d} V_j(z)\nonumber\\
= \,& \epsilon_j^{-d}\int_{\iota_j(\mathcal{M})} \int_{\iota_i(\mathcal{M})} K_i(x, w) K_j(\vartheta_{ji}(w),z)  \varrho_i(w) \mathrm{d} V_i(w) f_j(z) \varrho_j(z) \mathrm{d} V_j(z) \nonumber \\
= \,& \epsilon_j^{-d/2} \int_{\iota_i(\mathcal{M})}   K_i(x, w) \mathsf{F}_j(w) \varrho_i(w) \mathrm{d} V_i(w), 
\end{align} 
where   
\begin{equation}\label{eq_keydefinitionFj}
\mathsf{F}_j(w):=\epsilon_j^{-d/2} \int_{\iota_j(\mathcal{M})} K_j(\vartheta_{ji}(w),z)   f_j(z) \varrho_j(z) \mathrm{d} V_j(z)
\end{equation}
is defined on $\iota_i(\mathcal M)$.
Apply \eqref{lemma:numerator_est} to (\ref{eq_keydefinitionFj}), we have 
\begin{align}\label{eq_Fjexpansionneedtobeused}
 \mathsf{F}_j(w)=  (\varrho_{j}f_j)|_{\vartheta_{ji}(w)}  +\epsilon_j\frac{\mu_{1,2}^{(0)} }{2d} \big[\Delta^{(j)}(f_j \varrho_j) - \sw_j \varrho_j f_j\big]|_{\vartheta_{ji}(w)} +\mathrm{O}(\epsilon_j^2). 
\end{align} 
Insert (\ref{eq_Fjexpansionneedtobeused}) back into (\ref{eq_keyrepresentationone1111}) and apply  \eqref{lemma:numerator_est} again, we obtain
\begin{align}\label{eq_expansionone}
\epsilon_j^{-d/2} \int_{ \iota_i(\mathcal{M})}& K_i(x, w) \mathsf{F}_j(w)  \varrho_i(w) \mathrm{d} V_i(w) \nonumber \\
&=\epsilon_i^{d/2}\mathcal{A}_j+\epsilon_i^{d/2+1}\mathcal{B}_j+\epsilon_i^{d/2}\epsilon_j\mathcal{C}_j(x)+\mathrm{O}(\epsilon_i^{d/2} (\epsilon_i^2 +\epsilon_j^2)),
\end{align} 
where 
\begin{align*}
& \mathcal{A}_j(x):=\,\varrho_j(\vartheta_{ji}(x))\varrho_i(x)f_j(\vartheta_{ji}(x)), \\
&\mathcal{B}_j(x):=\, \frac{\mu_{1,2}^{(0)} }{2d} 
\big[f_j(\vartheta_{ji}(x))\varrho_i(x) \Delta^{(i)}(\varrho_j\circ\vartheta_{ji})(x) 
+\varrho_j(\vartheta_{ji}(x))\varrho_i(x) \Delta^{(i)}(f_j\circ\vartheta_{ji})(x)\\
&+f_j(\vartheta_{ji}(x))\varrho_j(\vartheta_{ji}(x)) \Delta^{(i)}\varrho_i(x)  
 +2\varrho_i(x) \nabla^{(i)}(f_j\circ\vartheta_{ji})|_{x}\cdot \nabla^{(i)}(\varrho_j\circ\vartheta_{ji})|_{x}\\ 
 &+2\varrho_j(\vartheta_{ji}(x))\nabla^{(i)}\varrho_i|_{x} \cdot \nabla^{(i)}(f_j\circ\vartheta_{ji})|_{x}
 +2f_j(\vartheta_{ji}(x))\nabla^{(i)}(\varrho_j\circ \vartheta_{ji})|_{x}\cdot \nabla^{(i)}\varrho_i|_x \\
& -{ \sw_i(x) \varrho_i(x) \varrho_j(\vartheta_{ji}(x)) f_j(\vartheta_{ji})(x)}
  \big] \\
  &=\,\frac{\mu_{1,2}^{(0)} }{2d} 
\big[(f_j\circ\vartheta_{ji}) \Delta^{(i)}(\varrho_i\varrho_j\circ\vartheta_{ji}) +2\nabla^{(i)}(\varrho_i \varrho_j\circ \vartheta_{ji})\cdot \nabla^{(i)}(f_j\circ\vartheta_{ji})) \\
& +  \varrho_i \varrho_j\circ \vartheta_{ji} \Delta^{(i)}(f_j\circ\vartheta_{ji})-{ \sw_i \varrho_i \varrho_j \circ \vartheta_{ji} f_j\circ \vartheta_{ji}} \big](x) \\
& \mathcal{C}_j(x):=\,\frac{\mu_{1,2}^{(0)} }{2d} \big[\Delta^{(j)} (f_j\varrho_j)  - \sw_j \varrho_j f_j\big]|_{\vartheta_{ji}(x)}\varrho_i(x)\,.  
\end{align*}

The denominator in (\ref{eq_expansionvmvl}) is analyzed by setting $f_1(z)=\ldots=f_{\sK}(z)=1$: 
\begin{align}\label{eq_keyreduction1}
&\epsilon_j^{-d}\int_{\iota_j(\mathcal{M})}  \widetilde{K}_{ij}(x,z) \varrho_j(z) \mathrm{d} V_j(z)\nonumber\\
= \,&\epsilon_i^{d/2}\mathcal{A}^0_j+\epsilon_i^{d/2+1}\mathcal{B}^0_j+\epsilon_i^{d/2}\epsilon_j\mathcal{C}^0_j(x)+\mathrm{O}(\epsilon_i^{d/2} (\epsilon_i^2 +\epsilon_j^2))\,,  \end{align} 
where
\begin{align*}
\mathcal{A}^0_j(x):=&\,\varrho_j(\vartheta_{ji}(x))\varrho_i(x) , \\
\mathcal{B}^0_j(x):=&\, \frac{\mu_{1,2}^{(0)} }{2d} \big[ {
\Delta^{(i)}(\varrho_i\varrho_j\circ\vartheta_{ji})(x)-(\sw_i \varrho_i \varrho_j \circ \vartheta_{ji})(x)} \big], \\   
 \mathcal{C}^0_j(x):=&\,\frac{\mu_{1,2}^{(0)} }{2d} \big[\Delta^{(j)} \varrho_j - \sw_j \varrho_j \big]|_{\vartheta_{ji}(x)}\varrho_i(x)\,. 
\end{align*} 

Finally, {inserting the above into (\ref{eq_expansionvmvl})}, by the binomial expansion, since $\varrho_j(x)\geq\delta^*$ uniformly for all $j$ by Assumption \ref{assu_model}, when $\epsilon_i$ is sufficiently small, we have {
\begin{align}\label{eq_herehrehrhehrhehreh}
    [\Omega \fb ]_i(x)
   =&\,\frac{\sum_{j\neq i} \sfu_{ij} \mathcal{A}_j(x)}{\sum_{j\neq i} \sfu_{ij} \mathcal{A}^0_j(x)}+\epsilon_i\frac{\sum_{j\neq i} \sfu_{ij}\mathcal{B}_j(x)}{\sum_{j\neq i} \sfu_{ij} \mathcal{A}^0_j(x)}-\epsilon_i\frac{\sum_{j\neq i}\sfu_{ij} \mathcal{B}^0_j(x)}{\sum_{j\neq i}\sfu_{ij} \mathcal{A}^0_j(x)}\frac{\sum_{j\neq i}\sfu_{ij} \mathcal{A}_j(x)}{\sum_{j\neq i}\sfu_{ij}\mathcal{A}^0_j(x)}\nonumber
   \\
   &+\frac{\sum_{j\neq i} \sfu_{ij}\epsilon_j\mathcal{C}_j(x)}{\sum_{j\neq i}\sfu_{ij}\mathcal{A}^0_j(x)}-\frac{\sum_{j\neq i}\sfu_{ij}\epsilon_j\mathcal{C}^0_j(x)}{\sum_{j\neq i}\sfu_{ij}\mathcal{A}^0_j(x)}\frac{\sum_{j\neq i}\sfu_{ij}\mathcal{A}_j(x)}{\sum_{j\neq i}\sfu_{ij} \mathcal{A}^0_j(x)}+\rO\left(\sum_{j\neq i}(\epsilon_i^2+\epsilon_j^2)\right)\,,
\end{align} 
where we used that $\sfu_{ij} \asymp 1.$}
By a direct expansion, {
\begin{equation}\label{eq_firsttermcontrolcontrolcontrol}
\frac{\sum_{j\neq i}\sfu_{ij}\mathcal{A}_j(x)}{\sum_{j\neq i} \sfu_{ij} \mathcal{A}^0_j(x)}=\frac{\sum_{j \neq i} \sfu_{ij}  \varrho_j(\vartheta_{ji}(x)) f_j(\vartheta_{ji}(x))}{\sum_{j \neq i} \sfu_{ij}\varrho_j(\vartheta_{ji}(x)) }=: \bar{f}_i(x)\,.
\end{equation} }
By another direct expansion, we have {\scriptsize
\begin{align*}
&\frac{\sum_{j\neq i} \sfu_{ij}\mathcal{B}_j(x)}{\sum_{j\neq i} \sfu_{ij} \mathcal{A}^0_j(x)}- \frac{\sum_{j\neq i} \sfu_{ij} \mathcal{B}^0_j(x)}{\sum_{j\neq i} \sfu_{ij} \mathcal{A}^0_j(x)}\frac{\sum_{j\neq i} \sfu_{ij} \mathcal{A}_j(x)}{\sum_{j\neq i} \sfu_{ij} \mathcal{A}^0_j(x)} = \frac{\mu_{1,2}^{(0)} }{2d} \times \\
&{\frac{\sum_{j \neq i} \sfu_{ij} \left[\big(f_j\circ\vartheta_{ji}-\bar{f}_i\big) [\Delta^{(i)}(\varrho_i\varrho_j\circ\vartheta_{ji})-\sw_i \varrho_i \varrho_j \circ \vartheta_{ji}] +2\nabla^{(i)}(\varrho_i \varrho_j\circ \vartheta_{ji})\cdot \nabla^{(i)}(f_j\circ\vartheta_{ji})) + \varrho_i \varrho_j\circ \vartheta_{ji} \Delta^{(i)}(f_j\circ\vartheta_{ji}) \right](x)}{\sum_{j \neq i} \sfu_{ij}\varrho_i (x)\varrho_j(\vartheta_{ji}(x))}\,. }
\end{align*} }
Similarly, we have {
\begin{align*}
&\frac{\sum_{j\neq i} \sfu_{ij}\epsilon_j\mathcal{C}_j(x)}{\sum_{j\neq i}\sfu_{ij} \mathcal{A}^0_j(x)}-\frac{\sum_{j\neq i}\sfu_{ij}\epsilon_j\mathcal{C}^0_j(x)}{\sum_{j\neq i}\sfu_{ij}\mathcal{A}^0_j(x)}\frac{\sum_{j\neq i}\sfu_{ij}\mathcal{A}_j(x)}{\sum_{j\neq i}\sfu_{ij}\mathcal{A}^0_j(x)}\\
=\,&\frac{\mu_{1,2}^{(0)} }{2d}\frac{\sum_{j \neq i} \sfu_{ij} \epsilon_j 
\big(\big[\Delta^{(j)} (f_j\varrho_j)  - \sw_j \varrho_j f_j\big]|_{\vartheta_{ji}(x)} - 
\big[\Delta^{(j)} \varrho_j - \sw_j \varrho_j \big]|_{\vartheta_{ji}(x)}
\bar{f}_i(x) \big) 
 }{\sum_{j \neq i} \sfu_{ij}\varrho_j(\vartheta_{ji}(x))}\,,
\end{align*} }
and hence the claim.

\end{proof}

\begin{rem}
When $\sK>2$, in general {$\frac{\sum_{j\neq i} \sfu_{ij}\mathcal{A}_j(x)}{\sum_{j\neq i} \sfu_{ij}\mathcal{A}^0_j(x)}\neq f_l(\vartheta_{li}(x))$ } for any $l$, so the curvature term $\sw_j$ cannot be canceled. As a result, the result is more complicated than the usual DM framework or AD where $\sK=2$, since we have a weighted sum in MDM. 
\end{rem}

\subsection{Variance analysis: proof of Theorem  \ref{thm_vmvlvarianceanalysis}}

The following lemma about one-sided Bernstein inequality. 

\begin{lem}\cite[Proposition 2.14]{MR3967104}\label{lem_bernsteininequality}
    Let $x_i, 1 \leq i \leq n,$ be independent random variables such that $|x_i|\leq c$ for all $1 \leq i \leq n$, then for all $t \geq 0,$ we have that
    $$\mathbb{P}\left(\frac{1}{n}\sum_{i=1}^n \left( x_i - \mathbb{E} x_i\right) \geq t\right) \leq \textup{exp}\left(-\frac{nt^2}{2\left(n^{-1} \sum_{i=1}^n \mathbb{E} x_i^2+\frac{ct}{3} \right)}\right).$$
\end{lem}

\begin{proof}[\bf Proof of Theorem \ref{thm_vmvlvarianceanalysis}] Recall (\ref{eq_Kmultiplication}). According to the definition of $\mathcal{A}$ in (\ref{eq_finaloperator}) and $\bm{f}$ in (\ref{eq_overallf}), we have that for $1 \leq s \leq n \sK$, $\ell=\lfloor (s-1)/n\rfloor+1$ and $s'=s-(\ell-1) n$, 
\begin{align}\label{eq_basicexpansionone}
\left[ \widecheck{\mathcal{A}} \bm{f} \right](s) 
&=\sum_{u=1}^{n\sK} \widecheck{\mathcal{A}}(\ell, u) \bm{f}(u)=\sum_{o \neq \ell} \sum_{j=1}^n \widecheck{\bK}^{\ell o}(s',j) \bm{f}_o(j)
\nonumber\\
&=\frac{\sum_{o \neq \ell}n^{-2}  \sum_{j=1}^n(\widecheck{\mathbf{K}}^\ell \widecheck{\mathbf{K}}^{o})(s',j)f_o(\xb_j^\ell)}{\sum_{o \neq \ell}n^{-2}  \sum_{j=1}^n(\widecheck{\mathbf{K}}^\ell \widecheck{\mathbf{K}}^{o})(s',j)},
\end{align} 
where $(\widecheck{\mathbf{K}}^{\ell} \widecheck{\mathbf{K}}^{o})(s',j)$ is the $(s',j)$ entry of $\widecheck{\mathbf{K}}^{\ell} \widecheck{\mathbf{K}}^{o}.$ That is, 
\begin{equation*}
(\widecheck{\mathbf{K}}^\ell \widecheck{\mathbf{K}}^{o})(s',j)=\sum_{t=1}^n \widecheck{\bK}^\ell(s',t) \widecheck{\bK}^o(t,j). 
\end{equation*}
For $X_\ell$ and $X_o$ defined in (\ref{eq_commonmode2}), we denote $\mathsf{F}^*_{o\ell}$ as follows 
\begin{equation*}
\mathsf{F}^*_{o\ell}(j)=\epsilon_\ell^{-d/2} K_\ell(X_\ell, \xb^\ell_{s'})  K_o(\xb_j^o, X_{o}).
\end{equation*}
Conditional on the observed point clouds, we have that 
\begin{equation*}
\mathbb{E} \mathsf{F}^*_{ol}(j)=\epsilon_{\ell}^{-d/2}\int_{\iota_\ell(\mathcal{M})} \int_{\iota_o(\mathcal{M})}  K_\ell(\xb_{s'}^\ell, x) K_o(y,\xb^o_{j})\mathrm{d}\mathbb{P}_{o \ell}(x,y). 
\end{equation*}
Using the assumption of (\ref{eq_commonmode2}) and the definition in (\ref{ew_densitydefinition}), we have  
\begin{equation*}
\mathbb{E} \mathsf{F}^*_{o \ell}(j)=\epsilon_{\ell}^{-d/2} \int_{\iota_\ell(\mathcal{M})}  K_\ell(\xb^\ell_{s'},z) K_o(\vartheta_{o\ell}(z),\xb_j^o) \varrho_{\ell}(z)  \mathrm{d} V_\ell(z). 
\end{equation*}
Under Assumptions \ref{assu_model} and \ref{assum_main}, we can apply Lemma \ref{lemma:numerator_est} to further get 
\begin{align*}
\mathbb{E} \mathsf{F}^*_{o \ell}(j)=\varrho_{\ell}(\xb_{s'}^\ell) K_o(\vartheta_{o\ell}(\xb_{s'}^\ell), \xb_j^o)+\rO(\epsilon_\ell).
\end{align*}
Using the assumption of (\ref{eq_densityproperty}) and Assumption \ref{assum_main}, we readily obtain  
\begin{equation}\label{eq_meancontrolexampleone}
\mathbb{E} \mathsf{F}^*_{o \ell}(j) \asymp 1. 
\end{equation}  
Moreover, for $(\mathsf{F}^*_{o \ell}(j))^2,$ we have 
\begin{equation*}
\mathbb{E}(\mathsf{F}^*_{o \ell}(j))^2=\epsilon_{\ell}^{-d/2} \epsilon_{\ell}^{-d/2} \int_{\iota_\ell(\mathcal{M})}  K^2_\ell(\xb^\ell_{s'},z)  K_o^2(\vartheta_{o\ell}(z),\xb_j^o) \varrho_{\ell}(z)  \mathrm{d} V_\ell(z).
\end{equation*}
By a discussion similar to (\ref{eq_meancontrolexampleone}), we obtain
\begin{equation}\label{eq_meancontrolexampletwo}
\mathbb{E} (\mathsf{F}^*_{o \ell}(j))^2 \asymp \epsilon_\ell^{-d/2}. 
\end{equation}
For $1 \leq t \leq n,$ denote
\begin{equation*}
\mathrm{X}_t:=\epsilon_\ell^{-d/2} \widecheck{\bK}^\ell(s',t) \widecheck{\bK}^o(t,j). 
\end{equation*}
Under Assumption \ref{assum_main}, we have that for some constant $C>0$,
\begin{equation*}
|\mathrm{X}_t| \leq C \epsilon_{\ell}^{-d/2}.  
\end{equation*}
When conditional on the observed data point both $\xb_{s'}^\ell$ and $\xb_j^o,$ we see that $\mathrm{X}_t, 1 \leq t \leq n,$ are independent. Based on the above discussions, we can apply Lemma \ref{lem_largederiviation} that for all $\delta>0,$ 
\begin{equation}\label{eq_computationdetails}
\mathbb{P}\left( \frac{1}{n} \sum_{t=1}^n \mathrm{X}_t-\mathbb{E} \mathsf{F}^*_{o \ell}(j) \geq \delta \right) \leq \exp \left(-\frac{n \delta^2}{2C \epsilon_\ell^{-d/2}+2/3C \epsilon_{\ell}^{-d/2} \delta} \right),
\end{equation}
where $C>0$ is some constant. Note that, one can find some constant $C_1>0$ so that the following holds:
\begin{equation*}
\frac{n \delta^2}{2C \epsilon_\ell^{-d/2}+2/3C \epsilon_{\ell}^{-d/2} \delta}  \geq C_1 \frac{n \delta^2}{\epsilon_\ell^{-d/2}}. 
\end{equation*}
Now for some sufficiently large constant $D>0,$ we choose 
\begin{equation*}
\delta^*:=\sqrt{D} \frac{\sqrt{\log n}}{\sqrt{C_1} n^{1/2} \epsilon_\ell^{d/4}}.
\end{equation*}
Together with (\ref{eq_computationdetails}), we readily see   
\begin{equation*}
\mathbb{P}\left( \frac{1}{n} \sum_{t=1}^n \mathrm{X}_t-\mathbb{E} \mathsf{F}^*_{o \ell}(j) \geq \delta^* \right) \leq \exp \left(-D \log n \right)=\frac{1}{n^D}.
\end{equation*}
Similarly, we can control the other side and obtain      
\begin{equation}\label{eq_concentrationone}
\frac{1}{n}\sum_{t=1}^n \epsilon_\ell^{-d/2} \widecheck{\bK}^\ell(s',t) \widecheck{\bK}^o(t,j)=\mathbb{E} \mathsf{F}_{o \ell}^*(j)+\mathrm{O}_{\prec} \left( \frac{1}{n^{1/2} \epsilon_\ell^{d/4}} \right). 
\end{equation} 

With the above preparation, we proceed to control the  numerator and denominator of (\ref{eq_basicexpansionone}).  For notional simplicity, denote
\begin{equation*}
\mathsf{F}^{**}_{o \ell}(j)=\epsilon_o^{-d/2}\mathbb{E} \mathsf{F}^*_{o \ell}(j) f_o(X_o). 
\end{equation*}
In what follows, we study $\mathsf{F}^{**}_{o \ell}(j)$ only conditional on $\xb_{s'}^\ell.$ Consequently, $\mathbb{E} \mathsf{F}^*_{o \ell}(j)$ is still a random variable due to the randomness of $\xb_j^o.$  By a discussion similar to (\ref{eq_meancontrolexampleone}) and (\ref{eq_meancontrolexampletwo}), we have 
\begin{equation}\label{eq_F88order(1)}
\mathbb{E} \mathsf{F}^{**}_{o \ell}(j) \asymp 1,  \ \mathbb{E} (\mathsf{F}^{**}_{o \ell}(j))^2 \asymp \epsilon_o^{-d/2} \epsilon_{\ell}^{-d/2}.
\end{equation}
For simplicity, we let $\mathsf{F}_{o \ell}:=\mathbb{E}\mathsf{F}^{**}_{o \ell}(j).$ Mores explicitly, 
\begin{equation}\label{eq_defn111}
\mathsf{F}_{o \ell}=\epsilon_{\ell}^{-d/2} \epsilon_o^{-d/2}\int_{ \iota_o(\mathcal{M})} \int_{\iota_\ell(\mathcal{M})}  K_\ell(\xb^\ell_{s'}, z) K_o(\vartheta_{o\ell}(z),w) \varrho_{\ell}(z)  \mathrm{d} V_\ell(z) f_o(w) \varrho_o(w) \mathrm{d} V_o(w).
\end{equation}
Note that according to (\ref{eq_concentrationone}) and the boundedness assumption of $f_o$, we have 
\begin{align*}
\frac{\epsilon_\ell^{-d/2} \epsilon_o^{-d/2}}{n^2}  \sum_{j=1}^n (\widecheck{\bK}^\ell \widecheck{\bK}^o)(s',j) f_o(\xb_j^o)&=\frac{\epsilon_o^{-d/2}}{n} \sum_{j=1}^n \left[ \frac{\epsilon_\ell^{-d/2}}{n} \sum_{t=1}^n \widecheck{\bK}^\ell(s',t) \widecheck{\bK}^o(t,j) \right] f_o(\xb_j^o) \nonumber \\
&=\frac{\epsilon_o^{-d/2}}{n} \sum_{j=1}^n \mathbb{E} \mathsf{F}^*_{o \ell}(j) f_o(\xb_j^o)+\mathrm{O}_{\prec}\left( \frac{1}{\sqrt{n} \epsilon_{\ell}^{d/4} \epsilon_o^{d/2}} \right). 
\end{align*}
For the term $\frac{\epsilon_o^{-d/2}}{n} \sum_{j=1}^n \mathbb{E} \mathsf{F}^*_{o \ell}(j) f_o(\xb_j^o),$ by a discussion similar to (\ref{eq_concentrationone}), we have
\begin{equation*}
\frac{\epsilon_o^{-d/2}}{n} \sum_{j=1}^n \mathbb{E} \mathsf{F}^*_{o \ell}(j) f_o(\xb_j^o)= \mathsf{F}_{o \ell}+\mathrm{O}_{\prec}\left( \frac{1}{\sqrt{n} \epsilon_{\ell}^{d/4} \epsilon_o^{d/4}} \right).
\end{equation*}
Consequently, we have 
\begin{equation}\label{eq_expansionexpansion111}
\frac{\epsilon_\ell^{-d/2} \epsilon_o^{-d/2}}{n^2}  \sum_{j=1}^n (\widecheck{\bK}^\ell \widecheck{\bK}^o)(s',j) f_o(\xb_j^o)=\mathsf{F}_{o \ell}+\mathrm{O}_{\prec}\left( \frac{1}{\sqrt{n} \epsilon_{\ell}^{d/4} \epsilon_o^{d/2}} \right).
\end{equation}
This provides the necessary control for the numerator of (\ref{eq_basicexpansionone}).  
To control the denominator, denote  
\begin{equation}\label{eq_defn2222}
\mathsf{G}_{o \ell}=\epsilon_{\ell}^{-d/2} \epsilon_o^{-d/2}\int_{\iota_o(\mathcal{M})} \int_{\iota_\ell(\mathcal{M})}  K_\ell(\xb^\ell_{s'},z) K_o(\vartheta_{o\ell}(z),w) \varrho_{\ell}(z)  \mathrm{d} V_\ell(z) \varrho_o(w) \mathrm{d} V_o(w).
\end{equation}
By a discussion similar to (\ref{eq_expansionexpansion111}), we obtain 
\begin{align}\label{eq_expansionexpansion112}
\frac{\epsilon_\ell^{-d/2} \epsilon_o^{-d/2}}{n^2}  \sum_{j=1}^n (\widecheck{\bK}^\ell \widecheck{\bK}^o)(s',j) &=\frac{\epsilon_o^{-d/2}}{n} \sum_{j=1}^n \left[ \frac{\epsilon_\ell^{-d/2}}{n} \sum_{t=1}^n \widecheck{\bK}^\ell(s',t) \widecheck{\bK}^o(t,j) \right]  \nonumber \\
&=\frac{\epsilon_o^{-d/2}}{n} \sum_{j=1}^n \mathbb{E} \mathsf{F}^*_{o \ell}(j) +\mathrm{O}_{\prec}\left( \frac{1}{\sqrt{n} \epsilon_{\ell}^{d/4} \epsilon_o^{d/2}} \right), \nonumber \\
&= \mathsf{G}_{o \ell}+\mathrm{O}_{\prec}\left( \frac{1}{\sqrt{n} \epsilon_{\ell}^{d/4} \epsilon_o^{d/2}} \right). 
\end{align}
By a discussion similar to (\ref{eq_F88order(1)}), we can show that 
\begin{equation}\label{eq_Gorder(1)}
\mathbb{E}\mathsf{G}_{o \ell} \asymp 1. 
\end{equation}

Combining (\ref{eq_expansionexpansion111}) and (\ref{eq_expansionexpansion112}), we now proceed to conclude our proof. Inserting the above controls into (\ref{eq_basicexpansionone}), we have that
\begin{align}\label{ds0}
\left[ \widecheck{\mathcal{A}} \bm{f} \right](s) & =\frac{\sum_{o \neq \ell} \epsilon_\ell^{d/2} \epsilon_o^{d/2} \left[\mathsf{F}_{o \ell}+\mathrm{O}_{\prec} \left( \frac{1}{\sqrt{n} \epsilon_{\ell}^{d/4} \epsilon_o^{d/2}} \right) \right]   }{\sum_{o \neq \ell} \epsilon_\ell^{d/2} \epsilon_o^{d/2} \left[\mathsf{G}_{o \ell}+\mathrm{O}_{\prec} \left( \frac{1}{\sqrt{n} \epsilon_{\ell}^{d/4} \epsilon_o^{d/2}} \right) \right] } \nonumber \\
&=\frac{\sum_{o \neq \ell} \epsilon_\ell^{d/2} \epsilon_o^{d/2}\mathsf{F}_{o \ell}}{\sum_{o \neq \ell} \epsilon_\ell^{d/2} \epsilon_o^{d/2}\mathsf{G}_{o \ell}}+\mathrm{O}_{\prec} \left( \sum_{ o \neq \ell} \frac{1}{\sqrt{n} \epsilon_\ell^{d/4} \epsilon_o^{d/2}} \right), 
\end{align}
where in the second step we used (\ref{eq_F88order(1)}) and (\ref{eq_Gorder(1)}). 
On the other hand, by the definitions of (\ref{eq_defn111}) and (\ref{eq_widetildekij}), we see  
\begin{align*}
\epsilon_\ell^{d/2} \epsilon_o^{d/2}\mathsf{F}_{o \ell}=\int_{\iota_o(\mathcal{M})}\widetilde{K}_{\ell o}(\xb_{s'}^\ell,w) f_o(w) \varrho_o(w) \mathrm{d} V_o(w). 
\end{align*}
Similarly, by (\ref{eq_defn2222}) and the definition in (\ref{eq_di}), we have   
\begin{equation*}
\sum_{o \neq \ell }\epsilon_\ell^{d/2} \epsilon_o^{d/2}\mathsf{F}_{o \ell}=d_\ell(\xb_{s'}^\ell). 
\end{equation*}
Using (\ref{eq_omegaij}), we see   
\begin{equation}\label{ds1}
\frac{\sum_{o \neq \ell} \epsilon_\ell^{d/2} \epsilon_o^{d/2}\mathsf{F}_{o \ell}}{\sum_{o \neq \ell} \epsilon_\ell^{d/2} \epsilon_o^{d/2}\mathsf{G}_{o \ell}}=\sum_{o \neq \ell} \omega_{\ell o} f_o(\xb_{s'}^\ell)=[\Omega \mathbf{f}]_\ell(\xb_{s'}^{\ell}),
\end{equation} 
where in the last step we used Definition \ref{defn_VCDM}. We  then complete our proof using (\ref{ds0}), (\ref{ds1}). 
\end{proof}

\section{Proofs of results in Section \ref{sec_robustnessofresults}}

Under the model reduction regime as between (\ref{sec_datasetupsignalplusnoise}) and (\ref{eq_covariancematrixsetting}), the discussion of \cite{9927456,ding2021kernel} can be carried out to the current setting with some  necessary modifications. The following lemma about large deviation bound is needed in our proof. 

\begin{lem}\cite[Lemma 3.1]{MR3183577}\label{lem_largederiviation}  We assume that $\zb_i, 1 \leq i \leq n,$ are a sequence of i.i.d. $p$-dimensional sub-Gaussian random vectors that
\begin{equation*}
\mathbb{E} \zb_i=\mathbf{0}, \ \operatorname{Cov} \left( \zb_i, \zb_j \right)=\delta_{ij} \mathbf{I}_p. 
\end{equation*} 
Then we have that 
\begin{enumerate}
\item For $1 \leq i \neq j \leq n,$ 
\begin{equation*}
\zb_i^\top \zb_j \prec p^{1/2},  \ i \neq j. 
\end{equation*}
Moreover, for any $p \times p$ symmetric matrix $\Ab,$  we have that 
\begin{equation*}
\zb_i^\top \Ab \zb_j \prec \| \Ab \|_{\mathsf{F}},
\end{equation*}
where $\| \Ab \|_{\mathsf{F}}$ is the Frobenius norm of $\Ab.$
\item For all $1 \leq i \leq n$, we have that 
\begin{equation*}
\|\zb_i\|_2^2=p+\OO_{\prec}(p^{1/2}),
\end{equation*}
and
\begin{equation*}
\zb_i^\top\Ab \zb_i=\operatorname{Tr}(\Ab)+\OO_{\prec}(\| \Ab \|_{\mathsf{F}}).
\end{equation*}
\end{enumerate}
\end{lem}

\subsection{Proof of Theorem \ref{thm_robustvmvl}}

Rewrite 
\begin{equation*}
\mathcal{A}-\widecheck{\mathcal{A}}=\mathcal{D}^{-1} \left( \left[\mathcal{K}-\widecheck{\mathcal{K}} \right]+\left[ \widecheck{\mathcal{D}}-\mathcal{D} \right] \widecheck{\mathcal{A}}  \right).
\end{equation*}
Since $\| \widecheck{\mathcal{A}} \| \leq 1,$  it suffices to control that 
\begin{equation}\label{eq_keycontrol}
\left\| \widecheck{\mathcal{A}}-\mathcal{A} \right\| \leq \left\| \mathcal{D}^{-1} \right\| \left( \left\| \widecheck{\mathcal{K}}-\mathcal{K} \right\|+\left\| \widecheck{\mathcal{D}}-\mathcal{D} \right\| \right). 
\end{equation}
We first control $\left\| \widecheck{\mathcal{K}}-\mathcal{K} \right\|.$ Since $\sK$ is finite, we focus our discussion on the $(\ell_1, \ell_2)$ block. For notional simplicity, we will use the following conventions in our proof. Recall (\ref{sec_datasetupsignalplusnoise}). For $k=1,2,$
\begin{equation}\label{eq_shshshshshsshorthandnotation}
\zb_{k,ij}=\yb^{\ell_k}_i-\yb^{\ell_k}_j, \ \widecheck{\zb}_{k,ij}=\xb^{\ell_k}_i-\xb^{\ell_k}_j,  \ \nb_{k,ij}=\bm{\xi}^{\ell_k}_i-\bm{\xi}^{\ell_k}_j.  
\end{equation}  
Moreover, following (\ref{eq_Kmultiplicationentry}),  we denote that 
\begin{equation}\label{eq_shorthandnotationhahahahaha}
\widecheck{\bK} \equiv \widecheck{\bK}^{\ell_1, \ell_2}:= \widecheck{\bK}^{\ell_1} \widecheck{\bK}^{\ell_2}, \ \bK \equiv \bK^{\ell_1, \ell_2}:=\bK^{\ell_1} \bK^{\ell_2}.  
\end{equation}
For $1 \leq i,j \leq n,$ by definition, we have that 
\begin{align}\label{eq_def111111111111111}
\widecheck{\bK}_{ij}-& \bK_{ij} \nonumber  \\
&=\sum_{\beta=1}^n \Big( K(\|\wzb_{1,i\beta}\|_2^2/\epsilon_{\ell_1})K(\|\wzb_{2,\beta j}\|_2^2/\epsilon_{\ell_2})-K(\|\zb_{1,i\beta}\|_2^2/\epsilon_{\ell_1})K(\|\zb_{2,\beta j}\|_2^2/\epsilon_{\ell_2}) \Big).
\end{align}
By the mean-value theorem, we have that for some $\delta$ between $\|\wzb_{1,i\beta}\|_2^2$ and $\|\zb_{1,i\beta}\|_2^2,$ we have that 
\begin{equation}\label{eq_expansionexpansionokayterms}
\left|K(\|\wzb_{1,i\beta}\|_2^2/\epsilon_{\ell_1})-K(\|\zb_{1,i\beta}\|_2^2/\epsilon_{\ell_1}) \right| \leq \left|K'(\delta/\epsilon_{\ell_1}) \right|  \frac{2| \widetilde{\zb}_{1, i \beta}^\top \nb_{1,i \beta}|+\| \nb_{1,i\beta} \|_2^2}{\epsilon_{\ell_1}}.
\end{equation}
By part 2 of Lemma \ref{lem_largederiviation} and the assumption that $\epsilon_{\ell_1} \asymp \sfc \sum_{i=1}^{r_{\ell_1}} \lambda_{\ell_1,i}$, we have that 
\begin{equation}\label{eq_discussionone}
\frac{\| \nb_{1,i\beta}\|_2^2}{\epsilon_{\ell_1}}=\mathrm{O}_{\prec}\left(\frac{p_{\ell_1} \sigma_{\ell_1}^2}{\sfc \sum_{i=1}^{r_{\ell_1}} \lambda_{\ell_1,i} }\right)=\mathrm{O}_{\prec}\left( \frac{1}{\sfc \SNR(\ell_1)} \right),  
\end{equation}
where in the second step we used the definition (\ref{eq_SNRdefinition}). According to the discussion around (\ref{eq_rotation}) and (\ref{eq_covariancematrixsetting}), without loss of generality, we can assume that only the first $r_{\ell_1}$ entries of $\widetilde{\zb}_{1, i\beta}$ are nonzero. Since $r_{\ell_1}$ is bounded, by Cauchy-Schwartz inequality and a discussion similar to (\ref{eq_discussionone}), we have that 
\begin{equation*}
 \frac{2| \widecheck{\zb}_{1, i \beta}^\top \nb_{1,i \beta}|}{\epsilon_{\ell_1}}=\mathrm{O}_{\prec} \left( \frac{1}{\sfc} \frac{1}{\sqrt{\SNR(\ell_1)p_{\ell_1}}} \right).
\end{equation*}
Combining the above discussions with (\ref{eq_expansionexpansionokayterms}) and using Assumption \ref{assum_main}, we readily obtain 
\begin{equation*}
\left|K(\|\wzb_{1,i\beta}\|_2^2/\epsilon_{\ell_1})-K(\|\zb_{1,i\beta}\|_2^2/\epsilon_{\ell_1}) \right|=\mathrm{O}_{\prec} \left(\frac{1}{\sfc \SNR(\ell_1)}+\frac{1}{\sfc} \frac{1}{\sqrt{\SNR(\ell_1)p_{\ell_1}}}  \right). 
\end{equation*}
With (\ref{eq_def111111111111111}), using triangle inequality and the assumption that the kernel $K$ is bounded, we have that 
\begin{align}\label{eq_boundusefuloneoneone}
|\widecheck{\bK}_{ij}-\bK_{ij}| \prec \frac{n}{\sfc} \left( \frac{1}{\SNR(\ell_1)}+\frac{1}{\SNR(\ell_2)}+\sqrt{\frac{1}{p_{\ell_1}\SNR(\ell_1)}}+\sqrt{\frac{1}{p_{\ell_2}\SNR(\ell_2)}} \right).
\end{align}
Then according to Gershgorin circle theorem (see Lemma A.2 of \cite{9927456}), we find that 
\begin{equation*}
\|\widecheck{\bK}-\bK\| \prec \frac{n^2}{\sfc} \left( \frac{1}{\SNR(\ell_1)}+\frac{1}{\SNR(\ell_2)}+\sqrt{\frac{1}{p_{\ell_1}\SNR(\ell_1)}}+\sqrt{\frac{1}{p_{\ell_2}\SNR(\ell_2)}} \right).
\end{equation*} 
Similarly, since $\sK$ is finite, we have that 
\begin{equation}\label{eq_451}
\left\| \widecheck{\mathcal{K}}-\mathcal{K} \right\| \prec \frac{n^2}{\sfc} \max_\ell \left[ \sum_{1 \leq \ell' \neq \ell \leq \sK} \left(\frac{1}{\SNR(\ell')}+\sqrt{\frac{1}{\SNR{(\ell') p_{\ell'}}}} \right)+\sK \left(  \frac{1}{\SNR(\ell)}+\sqrt{\frac{1}{\SNR{(\ell) p_{\ell}}}}\right) \right]. 
\end{equation} 

Then we control $\left\| \widecheck{\mathcal{D}}-\mathcal{D} \right\|.$  Without loss of generality, we study the first entry that with the shorthand notation in (\ref{eq_shorthandnotationhahahahaha}) that
\begin{align*}
\widecheck{\mathcal{D}}(1,1)-\mathcal{D}(1,1)&=\sum_{\ell=2}^\sK \sum_{j=1}^n (\widecheck{\bK}^{1 ,\ell}_{1j}-\bK^{1 ,\ell}_{1j}) \\
& \prec \frac{n^2}{\sfc} \sum_{\ell=2}^\sK  \left( \frac{1}{\SNR(1)}+\frac{1}{\SNR(\ell)}+\sqrt{\frac{1}{p_{1}\SNR(1)}}+\sqrt{\frac{1}{p_{\ell}\SNR(\ell)}} \right),
\end{align*}
where in the second step we used (\ref{eq_def111111111111111}) and (\ref{eq_boundusefuloneoneone}).  Since both matrices are diagonal,  it is not hard to see that 
\begin{align}\label{eq_452}
\left\| \widecheck{\mathcal{D}}-\mathcal{D} \right\| \prec \frac{n^2}{\sfc} \max_\ell \left[ \sum_{1 \leq \ell' \neq \ell \leq \sK} \left(\frac{1}{\SNR(\ell')}+\sqrt{\frac{1}{\SNR{(\ell') p_{\ell'}}}} \right)+\sK \left(  \frac{1}{\SNR(\ell)}+\sqrt{\frac{1}{\SNR{(\ell) p_{\ell}}}}\right) \right].
\end{align}

Finally, we provide a control for $\left\| \mathcal{D}^{-1} \right\|.$ Without loss of generality, we again focus our discussion on the first entry. By definition, we have that 
\begin{align}\label{eq_expansion}
{\mathcal{D}}(1,1)& =\sum_{\ell=2}^\sK \sum_{j=1}^n \sum_{k=1}^n K(\|\zb_{1,1k}\|_2^2/\epsilon_1)K(\|\zb_{\ell, kj}\|_2^2/\epsilon_\ell) \\
&=\sum_{\ell=2}^\sK \sum_{j \neq k} K(\|\zb_{1,1k}\|_2^2/\epsilon_1)K(\|\zb_{\ell, kj}\|_2^2/\epsilon_\ell)+n\sK K(0)^2. \nonumber
\end{align}
In what follows, we analyze the terms for $j \neq k$ on the right-hand side of the above equation.  Recall (\ref{eq_akeyquantity}).   Note that according to the mean-value theorem, we have that for some $\upsilon$ between $\|\zb_{\ell, kj}\|_2^2$ and $\sum_{i=1}^{r_\ell} \lambda_{\ell,i}+p\sigma_\ell^2,$ we have that 
\begin{equation*}
K(\|\zb_{\ell, kj}\|_2^2/\epsilon_\ell)-K(\mathsf{R}_\ell )=K'(\upsilon/\epsilon_\ell) \left(\|\zb_{\ell, kj}\|_2^2/\epsilon_\ell-\mathsf{R}_\ell  \right). 
\end{equation*}
By definition, we have that 
\begin{align*}
\frac{\|\zb_{\ell, kj}\|_2^2}{\epsilon_\ell}=\frac{\| \xb_k^\ell-\xb_j^\ell \|_2^2+2 (\xb_k^\ell-\xb_j^\ell)^\top (\bm{\xi}_k^\ell-\bm{\xi}_j^\ell)+\| \bm{\xi}_k^\ell-\bm{\xi}_j^\ell \|_2^2}{\epsilon_\ell}.
\end{align*}
By part 2 of Lemma \ref{lem_largederiviation}, we have 
\begin{align*}
& \frac{\| \xb_k^\ell \|_2^2+\| \xb_j^\ell\|_2^2}{\epsilon_\ell}=\frac{2 \sum_{i=1}^{r_\ell} \lambda_{\ell,i}}{\epsilon_\ell}+\mathrm{O}_{\prec} \left(\frac{\sqrt{\sum_{i=1}^{r_\ell} \lambda_{\ell,i}}}{\epsilon_\ell} \right), \\
&  \frac{\| \bm{\xi}_k^\ell-\bm{\xi}_j^\ell \|_2^2}{\epsilon_\ell}=\frac{2 p_\ell \sigma_\ell^2}{\epsilon_\ell}+\mathrm{O}_{\prec} \left(\frac{\sqrt{p_\ell \sigma_\ell^2}}{\epsilon_\ell} \right).
\end{align*}
Moreover, by part 1 of Lemma \ref{lem_largederiviation}, we have  that for $j \neq k$
\begin{align*}
& \frac{(\xb_k^\ell)^\top \xb_j^\ell }{\epsilon_\ell}=\mathrm{O}_{\prec} \left( \frac{\sqrt{\sum_{i=1}^{r_\ell} \lambda_{\ell,i}}}{\epsilon_\ell} \right), \\ 
& \frac{(\xb_k^\ell-\xb_j^\ell)^\top (\bm{\xi}_k^\ell-\bm{\xi}_j^\ell)}{\epsilon_\ell}=\mathrm{O}_{\prec} \left( \frac{\sqrt{ \sigma_\ell^2 \sum_{i=1}^{r_\ell} \lambda_{\ell,i}}}{\epsilon_\ell} \right). 
\end{align*}
Combining the above discussions, according to the definition in (\ref{eq_akeyquantity}), we have that
\begin{equation*}
\frac{\|\zb_{\ell, kj}\|_2^2}{\epsilon_\ell}-\mathsf{R}_\ell=\mathrm{O}_{\prec} \left( \frac{\sqrt{ \sigma_\ell^2 \sum_{i=1}^{r_\ell} \lambda_{\ell,i}}}{\epsilon_\ell}+\frac{\sqrt{p_\ell \sigma_\ell^2}}{\epsilon_\ell}  \right). 
\end{equation*} 

Consequently, according to the assumption of (\ref{eq_lowerboundkeykeykey}), we have that 
\begin{equation*}
K(\|\zb_{\ell, kj}\|_2^2/\epsilon_\ell)=K(\mathsf{R}_\ell )(1+\mathrm{o}_{\prec}(1)). 
\end{equation*} 

Inserting the above controls back into (\ref{eq_expansion}), we have that 
\begin{equation*}
{\mathcal{D}}(1,1)=n \sK K(0)^2+n(n-1)K(\mathsf{R}_\ell) \sum_{\ell=2}^\sK  K(\mathsf{R}_\ell ) \left(1+\mathrm{o}_{\prec}(1) \right). 
\end{equation*}
Since ${\mathcal{D}}$ is a diagonal matrix, we can bound 
\begin{align}\label{eq_453}
\|{\mathcal{D}}^{-1} \| & \prec \frac{1}{n \sK K^2(0)+n(n-1) \min_{\ell} K(\mathsf{R}_\ell ) \sum_{\ell' \neq \ell} K(\mathsf{R}_{\ell'}) } \nonumber \\
& \prec \frac{1}{n^2} \frac{1}{\min_{\ell} K(\mathsf{R}_\ell ) \sum_{\ell' \neq \ell} K(\mathsf{R}_{\ell'} ) },
\end{align}
where in the second step we used the assumption of (\ref{eq_bandwidthassumption2}).

Combining (\ref{eq_keycontrol}), (\ref{eq_451}), (\ref{eq_452}), as well as (\ref{eq_453}), we conclude that 
\begin{align*}
\left\| \widecheck{\mathcal{A}}-\mathcal{A} \right\|=\mathrm{O}_{\prec}\Bigg( \frac{1}{\sfc \min_{\ell} K(\mathsf{R}_\ell ) \sum_{\ell' \neq \ell} K(\mathsf{R}_{\ell'} ) } & \max_\ell \Bigg[ \sum_{1 \leq \ell' \neq \ell \leq \sK} \left(\frac{1}{\SNR(\ell')}+\sqrt{\frac{1}{\SNR{(\ell') p_{\ell'}}}} \right) \\
&+\sK \left(  \frac{1}{\SNR(\ell)}+\sqrt{\frac{1}{\SNR{(\ell) p_{\ell}}}}\right) \Bigg] \Bigg).  
\end{align*}
This completes our proof.

\subsection{Proof of Proposition \ref{thm_bandwidth}}

We first provide some technical preparation.  Under the assumption of (\ref{sec_datasetupsignalplusnoise}), for $1 \leq \ell  \leq \sK,$ we decompose that 
\begin{equation}\label{eq_originaldecomposition}
\| \yb_i^\ell-\yb_j^\ell \|_2^2=\| \xb_i^\ell-\xb_j^\ell \|_2^2+\|\bm{\xi}_i^\ell-\bm{\xi}_j^\ell \|_2^2-2 (\xb_i^\ell-\xb_j^\ell)^\top (\bm{\xi}_i^\ell-\bm{\xi}_j^\ell).
\end{equation}  
According to  part 2 of Lemma \ref{lem_largederiviation}, similar to the discussions between (\ref{eq_def111111111111111}) and (\ref{eq_boundusefuloneoneone}), we see that 
\begin{equation} \label{cross.pd0}
\max_{1\le i \neq j\le n}|2(\xb^\ell_i-\xb^\ell_j)^\top(\bm{\xi}^\ell_i-\bm{\xi}^\ell_j)|\prec \sigma_{\ell} \bigg(\sum_{i=1}^{r_\ell}\lambda_{\ell,i}\bigg)^{1/2},
\end{equation}
\begin{align}
\big|\|\yb^\ell_i-\yb^\ell_j\|_2^2-\|\xb^\ell_i-\xb^\ell_j\|_2^2\big|
&\prec { p_\ell \sigma_\ell^2+\sigma_\ell \big(\sum_{i=1}^{r_\ell}\lambda_{\ell,i}\big)^{1/2}},\label{y.bnd}
\end{align}
and 
\begin{equation}\label{eq_errorone}
\| \yb_i^\ell-\yb_j^\ell \|_2^2=\sum_{i=1}^{r_\ell} \lambda_{\ell,i}+\OO_{\prec} \left(\sigma_\ell \left( \sum_{i=1}^{r_\ell} \lambda_{\ell,i} \right)^{1/2}+p_{\ell} \sigma_\ell^2 \right),
\end{equation}
hold for all $1 \leq i \neq j \leq n.$

Next, we show that with high probability, for any sufficiently small constant $t>0$, there exist at most $n^2/\log^t n$ pairs of $(\xb^\ell_i,\xb^\ell_j)$ so that
\begin{equation}\label{eq_eijbound}
\|\xb^\ell_i - \xb^\ell_j\|_2^2 \leq \frac{C}{\log^{4t} n} \sum_{i=1}^{r_\ell} \lambda_{\ell,i},
\end{equation}
where $C>0$ is some universal constant. To see this, we define the events
\begin{equation}
\mathcal{E}(i,j)=\bigg\{\|\xb^\ell_i-\xb^\ell_j\|_2^2\le \frac{C}{ \log^{4t}n}\sum_{i=1}^{r_\ell} \lambda_{\ell,i} \bigg\},\ \ 1\le i \neq j\le n. \nonumber
\end{equation}
The following lemma will be used in our proof. 

\begin{lem}\cite[Lemma 15]{ding2023learning} \label{event.lem}
For any given $i\in\{1,2,\ldots,n\}$,  with high probability, at most $\frac{n}{\log^{t} n}$ events in $\{ \mathcal{E}(i,j): 1\le j\le n, j\ne i\}$ are true.
\end{lem}

By Lemma \ref{event.lem}, we can apply the union bound over all $1\le i\le n$, and conclude that, with high probability, at most $\frac{n^2}{\log^t n}$ pairs of $(\mathbf{x}^\ell_i, \mathbf{x}^\ell_j)$ satisfy (\ref{eq_eijbound}). Now for any percentile parameter $\omega_\ell \in(0,1)$ that is independent of $n$, when $n$ is sufficiently large, we always have $n(n-1)\omega>\frac{n^2}{\log^t n}.$ We consider  $k=\omega_\ell n$ for some $\omega_\ell \in(0,1).$ Without loss of generality, we assume that $k$ is an integer. Let $X_{(k)}$ be the $k$-th largest element in $\{\|\xb^\ell_i-\xb^\ell_j\|^2_2:i\ne j\}$. When $n$ is sufficiently large, we have
\begin{equation} \label{Xk.bnd}
X_{(k)}\ge \frac{C}{\log^{4t}n}\sum_{i=1}^{r_\ell} \lambda_{\ell,i}.
\end{equation}
Together with (\ref{y.bnd}) and (\ref{eq_originaldecomposition}), we find that when $n$ is sufficiently large, there are at least $k$ elements in $\{\|\yb^\ell_i-\yb^\ell_j\|^2_2:i\ne j\}$ satisfies   
\begin{equation*}
{X_{(k)} - \|\yb^\ell_i-\yb^\ell_j\|_2^2 = \OO_{\prec}\bigg( p_\ell \sigma_\ell^2+\sigma_\ell \big(\sum_{i=1}^{r_\ell} \lambda_{\ell,i}\big)^{1/2}\bigg)\,, \nonumber}
\end{equation*}
which follows from the fact that the first $k$  elements $\{X_{(i)}:1\le i\le k\}$ are no smaller than $X_{(k)}$. As a result, the $k$-th largest element $ Y_{(k)}$ in $\{\|\yb^\ell_i-\yb^\ell_j\|^2_2:i\ne j\}$ satisfies
\begin{equation} \label{Yklbnd}
X_{(k)}-Y_{(k)}={ \OO_{\prec}\bigg( p_\ell \sigma_\ell^2+\sigma_\ell \big(\sum_{i=1}^{r_\ell} \lambda_{\ell,i}\big)^{1/2}\bigg).}
\end{equation}
On the other hand, we claim that there are at least $n-k+1$ elements in $\{\|\yb^\ell_i-\yb^\ell_j\|^2_2:i\ne j\}$ satisfying
\[
{\|\yb^\ell_i-\yb^\ell_j\|_2^2 -X_{(k)} ={\OO_{\prec}\bigg( p_\ell \sigma_\ell^2+\sigma_\ell\big(\sum_{i=1}^{r_\ell} \lambda_{\ell,i}\big)^{1/2}\bigg)},}
\]
which follows from Lemma \ref{event.lem} and (\ref{y.bnd}) that the last $n-k+1$ elements $\{X_{(i)}:k\le i\le n\}$ are no greater than $X_{(k)}$. As a result, we also have
\begin{equation} \label{Ykubnd}
Y_{(k)}- X_{(k)}={ \OO_{\prec}\bigg(  p_\ell \sigma_\ell^2+\sigma_{\ell}\big(\sum_{i=1}^{r_\ell} \lambda_{\ell,i}\big)^{1/2}\bigg).}
\end{equation}
Combining the above discussions, under the assumption of Theorem \ref{thm_robustvmvl}, using the definition (\ref{eq_SNRdefinition}),  we have proved that 
\begin{equation*}
\left|\frac{h_\ell}{\sum_{i=1}^{r_\ell} \lambda_{\ell,i}} - 1\right| = \OO_\prec\left( \frac{1}{\SNR(\ell)}+\sqrt{\frac{1}{\SNR{(\ell) p_{\ell}}}}\right),
\end{equation*}
which concludes our proof.

\section{Some additional discussions}

\subsection{Impact of $\alpha$-normalization}\label{additionaldiscussions_appendix}
In this section, we provide a more detailed discussion of the population $\alpha$-normalization trick in the multiview setting, as mentioned in Remark \ref{rem_biasanalysis}. In our framework, there are two possible approaches: the first is to apply normalization directly to the kernel defined in (\ref{eq_widetildekij}) (see Appendix \ref{sec_alphaone}), and the second is to normalize each kernel individually before performing the convolution (see Appendix \ref{sec_alphatwo}). 

\subsubsection{$\alpha$-normalization after convolution}\label{sec_alphaone}

\begin{defn}
    For $0 \leq \alpha \leq 1,$ corresponding to the kernel in (\ref{eq_widetildekij}), we denote  the $\alpha$-normalized kernel as
    \begin{align*}
    \widetilde{K}_{ij}^{\alpha}(x,y) = \frac{\widetilde{K}_{ij}(x,y)}{\widetilde{q}^{\alpha}_{ij}(x)\widetilde{q}^{\alpha}_{ij}(y)},
    \end{align*}
    where we denote {$$\widetilde{q}_{ij}(x) = \int_{\iota_j(\mathcal{M})} \widetilde{K}_{ij}(x,w) \varrho_j(w) \mathrm{d} V_j(w),$$ and $$\widetilde{q}_{ij}(y) = \int_{\iota_i(\mathcal{M})} \widetilde{K}_{ij}(z,y) \varrho_i(z) \mathrm{d} V_i(z) = \int_{\iota_j(\mathcal{M})} \widetilde{K}_{ji}(y,w) \varrho_j(w) \mathrm{d} V_j(w).$$}
\end{defn}
With the above notations, we define $\omega_{ij}^{\alpha}$ analogously to (\ref{eq_omegaij}) and $\Omega^\alpha$ as in Definition \ref{defn_VCDM}. The corresponding bias analysis, analogous to that in Theorem \ref{thm_bias_vmvl}, is stated as follows. To ease the statements, we focus on reporting the result for $\alpha=1.$ {For notional simplicity, we introduce the following quantities. 
For $1 \leq j \leq \sK$ and some given function $g$, we denote 
\begin{align*}
&\mathcal{A}_j(g):=\varrho_j(\vartheta_{ji}(x))\varrho_i(x)g(\vartheta_{ji}(x)), \\
& \mathcal{B}_j(g):=\frac{\mu_{1,2}^{(0)} }{2d} 
\big[(g\circ\vartheta_{ji}) \Delta^{(i)}(\varrho_i\varrho_j\circ\vartheta_{ji}) +2\nabla^{(i)}(\varrho_i \varrho_j\circ \vartheta_{ji})\cdot \nabla^{(i)}(g\circ\vartheta_{ji})) \\
& +  \varrho_i \varrho_j\circ \vartheta_{ji} \Delta^{(i)}(g\circ\vartheta_{ji})-{\sw_i \varrho_i \varrho_j \circ \vartheta_{ji} g \circ \vartheta_{ji}} \big](x), \\
& \mathcal{C}_j(g):=\,\frac{\mu_{1,2}^{(0)} }{2d} \big[\Delta^{(j)} (g\varrho_j)  - \sw_j \varrho_j g\big]|_{\vartheta_{ji}(x)}\varrho_i(x)\,.  
\end{align*} 
With the above conventions, we further denote
\begin{equation*}
\mathsf{o}_j:= \frac{1}{\mathcal{A}_j(1)},\ 
\mathsf{q}_j:=1-\frac{\epsilon_i \mathcal{B}_j(1)+\epsilon_j \mathcal{C}_j(1)}{\mathcal{A}_j(1)}.
\end{equation*}
where $1$ means $g(x) \equiv 1.$ Denote
\begin{equation*}
\bar{f}_i^*=\frac{\sum_{j \neq i} (\rho_j(\vartheta_{ji}(x)))^{-1} f_j}{\sum_{j \neq i} (\rho_j(\vartheta_{ji}(x)))^{-1}},
\end{equation*}
and
\begin{equation*}
\mathsf{R}_i=\frac{\sum_{j \neq i} \mathsf{o}_j^2 \mathsf{q}_j^2 \mathcal{A}_j(f_j)}{\sum_{j \neq i} \mathsf{o}_j \mathsf{q}_j^2}-\bar{f}_i^*.
\end{equation*}
}
\begin{thm}\label{thm_alphanormalizationone}
Under the assumptions of Theorem \ref{thm_bias_vmvl}, for $1 \leq i \neq j \leq \sK,$ we have that  
{
    \begin{align*}
        [\Omega^1 \fb]_i(x) 
        & =\bar{f}_i^*(x)+\mathsf{R}_i^*+\epsilon_i\frac{\sum_{j \neq i} \mathsf{o}_j \mathsf{q}_j \left( \mathcal{B}_j\left(f_j \mathsf{o}_j \mathsf{q}_j \right)-\mathcal{B}_j\left(\mathsf{o}_j \mathsf{q}_j \right) \bar{f}_i^* \right)}{\sum_{j \neq i} \mathsf{o}_i \mathsf{q}_j^2 } \\
&\frac{\sum_{j \neq i} \epsilon_j \mathsf{o}_j \mathsf{q}_j \left( \mathcal{C}_j\left(f_j \mathsf{o}_j \mathsf{q}_j\right)-\mathcal{C}_j\left(\mathsf{o}_j \mathsf{q}_j \right) \bar{f}_i^* \right)}{\sum_{j \neq i} \mathsf{o}_j \mathsf{q}^2_j} +\rO\left(\sum_{j \neq i} (\epsilon^2_i+\epsilon^2_j+\epsilon_i \epsilon_j)\right).
    \end{align*} }
\end{thm}

\begin{proof}
The proofs follow similar arguments to those of Theorem~\ref{thm_bias_vmvl}. We therefore only highlight the key modifications relative to the proof of Theorem~\ref{thm_bias_vmvl}. { Similar to (\ref{eq_expansionvmvl}), we have that 
\begin{equation*}
        [\Omega^1 \fb]_i(x)  = \frac{\sum_{j \neq i} \sfu_{ij} \epsilon_j^{-d}\int_{\iota_j(\mathcal{M})} \widetilde{K}_{ij}^{1}(x,w) f_j(w) \varrho_j(w) \mathrm{d} V_j(w)}{\sum_{j \neq i} \sfu_{ij} \epsilon_j^{-d}\int_{\iota_j(\mathcal{M})} \widetilde{K}_{ij}^{1}(x,w) \varrho_j(w) \mathrm{d} V_j(w)}.
\end{equation*}
}

For the numerator, by definition, we have the following {
    \begin{align}\label{eq_sampleform}
        \int_{\iota_j(\mathcal{M})} \widetilde{K}_{ij}^{1}(x,w) f_j(w) \varrho_j(w) \mathrm{d} V_j(w) & = \int_{\iota_j(\mathcal{M})} \frac{\widetilde{K}_{ij}(x,w)}{\widetilde{q}_{ij}(x)\widetilde{q}_{ij}(w)} f_j(w) \varrho_j(w) \mathrm{d} V_j(w) \nonumber \\
         &= \frac{1}{\widetilde{q}_{ij}(x)} \int_{\iota_j(\mathcal{M})} \widetilde{K}_{ij}(x,w)\frac{f_j(w)}{\widetilde{q}_{ij}(w)} \varrho_j(w) \mathrm{d} V_j(w).
    \end{align} 
Therefore, compared to the discussion in (\ref{eq_keyrepresentationone1111}), one can simply replace $f_j$ with a normalized version. For a correct scaling, we first incorporate the analysis of the $\widetilde{q}_{ij}(x)$ term. By (\ref{eq_keyreduction1}), we have that 
\begin{equation}\label{eq_bbbb1111}
\widetilde{q}_{ij}(x)=\epsilon_j^{d}\left[\epsilon_i^{d/2}\mathcal{A}^0_j+\epsilon_i^{d/2+1}\mathcal{B}^0_j+\epsilon_i^{d/2}\epsilon_j\mathcal{C}^0_j(x)+\mathrm{O}(\epsilon_i^{d/2} (\epsilon_i^2 +\epsilon_j^2)) \right].
\end{equation} 
In other words, when the bandwidths are sufficiently small, we have that $\widetilde{q}_{ij}(x) \asymp \epsilon_j^{d} \epsilon_i^{d/2}.$ Consequently, to correctly apply Lemma \ref{lemma:for_numerator_est}, we need to rescale the concerned quantities accordingly. Denote 
\begin{equation*}
\mathsf{f}_j:=\epsilon_j^d \epsilon_i^{d/2} \frac{f_j}{\widetilde{q}_{ij}}.
\end{equation*}

According to (\ref{eq_expansionone}), we have that 
\begin{align*}
\epsilon_j^{-d}\int_{\iota_j(\mathcal{M})} & \widetilde{K}_{ij}^{1}(x,w) f_j(w) \varrho_j(w) \mathrm{d} V_j(w) \\
&=\frac{\epsilon_j^{-d} \epsilon_i^{-d/2}}{\widetilde{q}_{ij}(x)} \left[\epsilon_i^{d/2}\mathcal{A}'_j+\epsilon_i^{d/2+1}\mathcal{B}'_j+\epsilon_i^{d/2}\epsilon_j\mathcal{C}'_j(x)+\mathrm{O}(\epsilon_i^{d/2} (\epsilon_i^2 +\epsilon_j^2)) \right],
\end{align*}     
where $\mathcal{A}_j', \mathcal{B}_j'$ and $\mathcal{C}_j'$ are defined similar to $\mathcal{A}_j, \mathcal{B}_j$ and $\mathcal{C}_j$  in (\ref{eq_expansionone}) by replacing $f_j$ with $\mathsf{f}_j.$    
    }
    
For the denominator, it is easy to see that  {
    \begin{align*}
        \int_{\iota_j(\mathcal{M})} \widetilde{K}_{ij}^{1}(x,w) \varrho_j(w) \mathrm{d} V_j(w) = \frac{1}{\widetilde{q}_{ij}(x)} \int_{\iota_j(\mathcal{M})} \widetilde{K}_{ij}(x,w)\frac{1}{\widetilde{q}_{ij}(w)} \varrho_j(w) \mathrm{d} V_j(w). 
    \end{align*} 
Consequently, it can be studied similarly as in (\ref{eq_sampleform}) by setting $f_j=\epsilon_j^{d} \epsilon_i^{d/2}/\widetilde{q}_{ij}$ so that 
\begin{align*}
\epsilon_j^{-d}\int_{\iota_j(\mathcal{M})} & \widetilde{K}_{ij}^{1}(x,w) \varrho_j(w) \mathrm{d} V_j(w) \\
& =\frac{\epsilon_j^{-d} \epsilon_i^{-d/2}}{\widetilde{q}_{ij}(x)} \left[\epsilon_i^{d/2}\mathcal{A}^{''}_j+\epsilon_i^{d/2+1}\mathcal{B}^{''}_j+\epsilon_i^{d/2}\epsilon_j\mathcal{C}^{''}_j(x)+\mathrm{O}(\epsilon_i^{d/2} (\epsilon_i^2 +\epsilon_j^2)) \right],
\end{align*} 
where $\mathcal{A}_j^{''}, \mathcal{B}_j^{''}$ and $\mathcal{C}_j^{''}$ are defined similar to $\mathcal{A}_j, \mathcal{B}_j$ and $\mathcal{C}_j$  in (\ref{eq_expansionone}) by replacing $f_j$ with $\epsilon_j^{d} \epsilon_i^{d/2}/\widetilde{q}_{ij}.$ Consequently, we can write 
\begin{align*}
  [\Omega^1 \fb]_i(x)&=\frac{ \sum_{j \neq i} \frac{\sfu_{ij}}{\widetilde{q}_{ij}(x)} \left[ \epsilon_i^{d/2}\mathcal{A}'_j+\epsilon_i^{d/2+1}\mathcal{B}'_j+\epsilon_i^{d/2}\epsilon_j\mathcal{C}'_j(x)+\mathrm{O}(\epsilon_i^{d/2} (\epsilon_i^2 +\epsilon_j^2)) \right]}{\sum_{j \neq i} \frac{\sfu_{ij}}{\widetilde{q}_{ij}(x)} \left[ \epsilon_i^{d/2}\mathcal{A}^{''}_j+\epsilon_i^{d/2+1}\mathcal{B}^{''}_j+\epsilon_i^{d/2}\epsilon_j\mathcal{C}^{''}_j(x)+\mathrm{O}(\epsilon_i^{d/2} (\epsilon_i^2 +\epsilon_j^2)) \right]} \\
  &=\frac{  \sum_{j \neq i} \frac{\epsilon_j^d \epsilon_i^{d/2}}{\widetilde{q}_{ij}(x)} \left[ \epsilon_i^{d/2}\mathcal{A}'_j+\epsilon_i^{d/2+1}\mathcal{B}'_j+\epsilon_i^{d/2}\epsilon_j\mathcal{C}'_j(x)+\mathrm{O}(\epsilon_i^{d/2} (\epsilon_i^2 +\epsilon_j^2)) \right]}{\sum_{j \neq i} \frac{\epsilon_j^d \epsilon_i^{d/2}}{\widetilde{q}_{ij}(x)} \left[ \epsilon_i^{d/2}\mathcal{A}^{''}_j+\epsilon_i^{d/2+1}\mathcal{B}^{''}_j+\epsilon_i^{d/2}\epsilon_j\mathcal{C}^{''}_j(x)+\mathrm{O}(\epsilon_i^{d/2} (\epsilon_i^2 +\epsilon_j^2)) \right]}.
\end{align*}  
Then the rest of the arguments follow closely as in (\ref{eq_herehrehrhehrhehreh}). For completeness, we provide the arguments for the first term as in (\ref{eq_firsttermcontrolcontrolcontrol}). We first provide some controls. According to (\ref{eq_bbbb1111}), we have that 
\begin{equation*}
\frac{\epsilon_j^d \epsilon_i^{d/2}}{\widetilde{q}_{ij}(x)}=\frac{1}{\mathcal{A}_j^0}-\frac{\epsilon_i \mathcal{B}_j^0+\epsilon_j \mathcal{C}_j^0}{(\mathcal{A}_j^0)^2}+\rO(\epsilon_i^2+\epsilon_j^2+\epsilon_i \epsilon_j). 
\end{equation*}
Consequently, we have that 
\begin{align*}
&\mathcal{A}_j'=\frac{\mathcal{A}_j}{\mathcal{A}_j^0} \left(1-\frac{\epsilon_i \mathcal{B}_j^0+\epsilon_j \mathcal{C}_j^0}{(\mathcal{A}_j^0)^2}+\rO(\epsilon_i^2+\epsilon_j^2+\epsilon_i \epsilon_j) \right), \\
& \mathcal{A}_j^{''}=1-\frac{\epsilon_i \mathcal{B}_j^0+\epsilon_j \mathcal{C}_j^0}{\mathcal{A}_j^0}+\rO(\epsilon_i^2+\epsilon_j^2+\epsilon_i \epsilon_j). 
\end{align*}
This results in 
\begin{equation*}
\frac{\sum_{j \neq i} \frac{\epsilon_j^d \epsilon_i^{d/2}}{\widetilde{q}_{ij}(x)}  \epsilon_i^{d/2}\mathcal{A}'_j}{\sum_{j \neq i} \frac{\epsilon_j^d \epsilon_i^{d/2}}{\widetilde{q}_{ij}(x)}  \epsilon_i^{d/2}\mathcal{A}^{''}_j}=\bar{f}^*(x)+\mathsf{R}_i+\rO(\sum_{j \neq i} (\epsilon_i^2+\epsilon_j^2+\epsilon_i \epsilon_j)). 
\end{equation*}
Similarly, we can control the other terms as in (\ref{eq_herehrehrhehrhehreh}). This completes our proof. }

\end{proof}
As can be seen from the above theorem, unlike in the single-view setting, the impact of sampling cannot be eliminated in the multiview setting—even after normalization—due to the inherent nature of multiview data. In fact, our algorithm is designed to leverage the diverse information from different views to uncover shared structures, so removing these effects would be counterproductive.  

\subsubsection{$\alpha$-normalization before convolution}\label{sec_alphatwo} 
In Appendix \ref{sec_alphaone}, the $\alpha$-normalization trick is applied to the kernel (\ref{eq_widetildekij}) after convolution. Alternatively, one could first normalize the kernel and then perform the convolution. We explore this possibility in the current section. Moreover specifically, we define the $\alpha$-normalized kernel as follows {
    \begin{align}\label{eq_normalizationtypetwo}
    \widetilde{K}_{ij}^{\alpha}(x,y) = \int_{\iota_i(\mathcal{M})} K_i^{\alpha} (x, z) K_j^{\alpha}(\vartheta_{ji}(z), y)  \varrho_i(z) \mathrm{d}V_i(z),
    \end{align}
    where for $1 \leq i \leq \sK$
    \begin{align*}
        K_i^{\alpha} (x, y) = \frac{K_i (x, y)}{q^\alpha_i(x)q^\alpha_i(y)}, \ \, q_i(x) = \int_{\iota_i(\mathcal{M})} K_i(x,z) \varrho_i(z) \mathrm{d} V_i(z).
    \end{align*} }

The results are stated as follows. We again focus on reporting the results for $\alpha=1.$ 
{For $1 \leq i \leq \sK,$ denote
    \begin{align}\label{eq_defnqi}
        \sq_i = \varrho_{i}\left[1+\frac{\epsilon_i \mu_{1,2}^{(0)} }{d} \left( \frac{\Delta^{(i)} \varrho_i}{2 \varrho_i}-\frac{1}{2} \sw_i \right)  \right], \  
    \end{align}
 and for a given function $g(x),$ denote 
 \begin{equation*}
 \mathcal{L}_j(g):=\varrho_j(x)g(x) +\frac{\epsilon_j\mu_{1,2}^{(0)}}{2d}\left(g(x)\Delta^{(j)} \varrho_j(x)+\varrho_j(x)\Delta^{(j)} g(x)
	 +2\nabla^{(j)} \varrho^{(j)}(x)\cdot\nabla^{(j)} g(x) -\sw_j(x)g(x)\varrho_j(x)\right).
 \end{equation*} 
Recall (\ref{eq_uij}). Moreover, we denote 
\begin{equation*}
\bar{f}_i^{**}(x):=\frac{\sum_{j \neq i} \sfu_{ij}^{-1/2} (\varrho_j(\vartheta_{ji}(x)))^{-1}   f_j(x) }{\sum_{j \neq i} \sfu_{ij}^{-1/2} (\varrho_j(\vartheta_{ji}(x)))^{-1} }.
\end{equation*}
and 
\begin{equation*}
\mathsf{H}_i:=\frac{\sum_{j \neq i} \sfu_{ij}^{-1/2} \mathcal{L}_i \left(\frac{1}{\mathsf{q}_i \mathsf{q}_j}\mathcal{L}_j(\frac{f_j}{\mathsf{q}_j})\right)}{\sum_{j \neq i} \sfu_{ij}^{-1/2} \mathcal{L}_i \left(\frac{1}{\mathsf{q}_i \mathsf{q}_j}\mathcal{L}_j(\frac{1}{\mathsf{q}_j})\right) }-\bar{f}_i^{**}(x).
\end{equation*}
}
\begin{thm}\label{thm_alphanormalizationtwo}
    Under the assumptions of Theorem \ref{thm_bias_vmvl}, for $1 \leq i \neq j \leq \sK,$ we have that {
    \begin{align*}
        [\Omega^1 \fb ]_i(x)=\bar{f}_i^{**}(x)  +\mathsf{H}_i+\rO\left(\sum_{j \neq i} (\epsilon^2_i+\epsilon_i \epsilon_j+\epsilon^2_j)\right).
    \end{align*} }    
\end{thm}
\begin{proof}
{We first prepare some important expansions. According to the definition in (\ref{eq_normalizationtypetwo}), we have that 
    \begin{align*}
        & \int_{\iota_j(\mathcal{M})} \widetilde{K}_{ij}^{1}(x,w) f_j(w) \varrho_j(w) \mathrm{d} V_j(w) \\\nonumber = & \int_{\iota_j(\mathcal{M})} \left(\int_{\iota_i(\mathcal{M})} K_i^{1} (x, z) K_j^{1}(\vartheta_{ji}(z), w)  \varrho_i(z) \mathrm{d}V_i(z) \right)f_j(w) \varrho_j(w) \mathrm{d} V_j(w)
        \\\nonumber = & \int_{\iota_j(\mathcal{M})} \left(\int_{\iota_i(\mathcal{M})} \frac{K_i (x, z)}{q_i(x)q_i(z)} \frac{K_j (\vartheta_{ji}(z), w)}{q_j(\vartheta_{ji}(z))q_j(w)}  \varrho_i(z) \mathrm{d}V_i(z) \right)f_j(w) \varrho_j(w) \mathrm{d} V_j(w)
        \\\nonumber = & \frac{1}{q_i(x)} \int_{\iota_j(\mathcal{M})} \left(\int_{\iota_i(\mathcal{M})} K_i (x, z) K_j (\vartheta_{ji}(z), y)\frac{1}{q_i(z)q_j(\vartheta_{ji}(z))} \varrho_i(z) \mathrm{d}V_i(z) \right) \frac{ f_j(w)}{q_j(w)} \varrho_j(w) \mathrm{d} V_j(w) 
        \\\nonumber = & \frac{1}{q_i(x)} \int_{\iota_i(\mathcal{M})} K_i (x, z) \left(\int_{\iota_j(\mathcal{M})}  K_j (\vartheta_{ji}(z), w) \frac{ f_j(w)}{q_j(w)} \varrho_j(w) \mathrm{d} V_j(w)   \right)  \frac{1}{q_i(z)q_j(\vartheta_{ji}(z))} \varrho_i(z) \mathrm{d}V_i(z). 
    \end{align*} 

For notional simplicity, till the end of the proof, we denote  
\begin{equation*}
\mathsf{F}_j(z):= \int_{\iota_j(\mathcal{M})}  K_j (\vartheta_{ji}(z), w) \frac{ f_j(w)}{q_j(w)} \varrho_j(w) \mathrm{d} V_j(w),
\end{equation*}
and 
\begin{equation*}
\mathsf{E}_j(z):= \int_{\iota_j(\mathcal{M})}  K_j (\vartheta_{ji}(z), w) \frac{1}{q_j(w)} \varrho_j(w) \mathrm{d} V_j(w).
\end{equation*}
Note that according to Lemma \ref{lemma:for_numerator_est}, for $1 \leq i \leq \sK,$ we have that 
\begin{equation}\label{eq_qiexpansion}
\epsilon_i^{-d/2}q_i(x)=\mathsf{q}_i+\rO(\epsilon_i^2).
\end{equation}   
Consequently, together with Lemma \ref{lemma:for_numerator_est},  we have that 
\begin{align*}
\epsilon_j^{d/2} \mathsf{F}_j(z)& =\int_{\iota_j(\mathcal{M})}  K_j (\vartheta_{ji}(z), w) f_j(\omega)(1+\rO(\epsilon_j)) \mathrm{d} V_j(w) \\
&=\epsilon_j^{d/2}f_j(\vartheta_{ji}(z))(1+\rO(\epsilon_j)). 
\end{align*}
This yields that
\begin{equation*}
\mathsf{F}_j(z)=f_j(\vartheta_{ji}(z)) \left(1+\mathrm{O}(\epsilon_j) \right).
\end{equation*}
Together with (\ref{eq_qiexpansion}), we can control that 
\begin{align*}
\epsilon_i^{d/2} \epsilon_j^{d/2} \int_{\iota_j(\mathcal{M})} & \widetilde{K}_{ij}^{1}(x,w) f_j(w) \varrho_j(w) \mathrm{d} V_j(w) \\
& = \frac{\epsilon_i^{d/2}}{q_i(x)}\int_{\iota_i(\mathcal{M})} \epsilon_i^{-d/2} K_i(x,z) f_j(\vartheta_{ji}(z)) \frac{1}{\varrho_j(\vartheta_{ji}(z))} \left(1+\rO(\epsilon_i+\epsilon_j) \right) \mathrm{d} V_i(z) \\
&=\frac{f_j(\vartheta_{ji}(x))}{\varrho_i(x)\varrho(\vartheta_{ji}(x))} \left(1+\rO(\epsilon_i+\epsilon_j) \right),
\end{align*}
where in the second step we again used Lemma \ref{lemma:for_numerator_est}.

For notional simplicity, we denote 
\begin{equation*}
\mathsf{g}_j(z):= \frac{\epsilon_i^{d/2} \epsilon_j^{d/2} \mathsf{F}_j(z)}{q_i(z) q_j(\vartheta_{ji}(z))}, \ \mathsf{g}_j^0(z):= \frac{\epsilon_i^{d/2} \epsilon_j^{d/2}\mathsf{E}_j(z)}{q_i(z) q_j(\vartheta_{ji}(z))}.
\end{equation*}
Using the above notation, we can obtain the representation that 
\begin{equation*}
\epsilon_i^{d/2} \epsilon_j^{d/2} \int_{\iota_j(\mathcal{M})}  \widetilde{K}_{ij}^{1}(x,w) f_j(w) \varrho_j(w) \mathrm{d} V_j(w)=\frac{\epsilon_i^{d/2}}{q_i(x)} \int_{\iota_i(\mathcal{M})} \epsilon_i^{-d/2} K_i(x,z) \mathsf{g}_j(z) \varrho_i(z)  \mathrm{d} V_i(z).
\end{equation*}

Now we proceed to conclude the proof. Similar to (\ref{eq_expansionvmvl}), we can write 
\begin{align*}
    [\Omega \fb ]_i(x)
   & = \frac{\epsilon_i^{d} \sum_{j \neq i} \epsilon_j^{-d/2} \epsilon_j^{d/2} \int_{ \iota_j(\mathcal{M})} \widetilde{K}_{ij}^1(x,z) f_j(z) \varrho_j(z) \mathrm{d} V_j(z)}{\epsilon_i^{d}\sum_{j \neq i}\epsilon_j^{-d/2} \epsilon_j^{d/2} \int_{\iota_j(\mathcal{M})} \widetilde{K}^1_{ij}(x,z) \varrho_j(z) \mathrm{d} V_j(z)} \nonumber \\
   & =\frac{\sum_{j \neq i} \sfu^{-1/2}_{ij}  \epsilon_i^{d/2} \epsilon_j^{d/2} \int_{ \iota_j(\mathcal{M})} \widetilde{K}^1_{ij}(x,z) f_j(z) \varrho_j(z) \mathrm{d} V_j(z)}{\sum_{j \neq i} \sfu^{-1/2}_{ij} \epsilon_i^{d/2} \epsilon_j^{d/2} \int_{\iota_j(\mathcal{M})} \widetilde{K}^1_{ij}(x,z) \varrho_j(z) \mathrm{d} V_j(z)}.
\end{align*}
Invoking the above discussions, we can further write 
\begin{align*}
  [\Omega \fb ]_i(x)=\frac{\sum_{j \neq i} \sfu_{ij}^{-1/2} \int_{\iota_i(\mathcal{M})} \epsilon_i^{-d/2} K_i(x,z) \mathsf{g}_j(z) \varrho_i(z) \mathrm{d} V_i(z)}{\sum_{j \neq i} \sfu_{ij}^{-1/2} \int_{\iota_i(\mathcal{M})} \epsilon_i^{-d/2} K_i(x,z) \mathsf{g}^0_j(z) \varrho_i(z) \mathrm{d} V_i(z)}.
\end{align*}
For $\mathsf{F}_j(z),$ we have that
\begin{align*}
\mathsf{F}_j(z)=\mathcal{L}_j(f_j/\mathsf{q}_j)+\mathrm{O}(\epsilon_i^2+\epsilon_j^2+\epsilon_i \epsilon_j). 
\end{align*}
Consequently, we have that
\begin{equation*}
\mathsf{g}_j(z)=\frac{1}{\mathsf{q}_i \mathsf{q}_j}\mathcal{L}_j(f_j/\mathsf{q}_j)+\mathrm{O}(\epsilon_i^2+\epsilon_j^2+\epsilon_i \epsilon_j). 
\end{equation*}
Consequently, by Lemma \ref{lemma:for_numerator_est}, we have that  
\begin{equation*}
\int_{\iota_i(\mathcal{M})} \epsilon_i^{-d/2} K_i(x,z) \mathsf{g}_j(z) \varrho_i(z) \mathrm{d} V_i(z)=\mathcal{L}_i(\frac{1}{\mathsf{q}_i \mathsf{q}_j}\mathcal{L}_j(f_j/\mathsf{q}_j))+\mathrm{O}(\epsilon_i^2+\epsilon_j^2+\epsilon_i \epsilon_j). 
\end{equation*}

and a binomial discussion similar to (\ref{eq_herehrehrhehrhehreh}), we can obtain the results. 
%
}
\end{proof}

Similar to the discussion at the end of Appendix \ref{sec_alphatwo}, when applying the $\alpha$-normalization trick before convolution, the impact of sampling cannot be eliminated either. This highlights that our algorithm is capable of accounting for such sampling effects and can still effectively learn the common structures from them.

{
\subsection{Further practical consideration: sparse sketching}\label{sec_practicalnoisereduction}
In this section, we provide more discussions related to the discussions in the end of Section \ref{sec_robustnessofalgorithms} on the sketching step to weaken the assumptions on SNRs. Additional simulations results will be provided in Appendix \ref{sec_additionalsimulationresults}.

According to the model reduction described in (\ref{eq_reducedsamples}) and (\ref{eq_covariancematrixsetting}), we observe that, at the data level, the signals are sparse. This implies that identifying the locations of these signal components and applying a sparse representation via sketching can substantially reduce the data dimension and, consequently, the overall noise level. As discussed therein, denote $a_{ij}=\| \mathring{\zb}_i^\ell - \mathring{\zb}_j^\ell \|,$  $b_{ij}=\| \mathring{\yb}_i^\ell - \mathring{\yb}_j^\ell \|$ and $c_{ij}=\|\xb_i^\ell-\xb_j^\ell \|,$ one can follow the discussion between (\ref{eq_expansionexpansionokayterms}) and (\ref{eq_boundusefuloneoneone}) to obtain that
\begin{equation*}
|K(a_{ij}/\epsilon_\ell)-K(c_{ij}/\epsilon_\ell)|=\mathrm{O}_{\prec} \left( \frac{|a_{ij}-b_{ij}|}{\epsilon_\ell}+ \frac{\mathsf{s}_\ell \sigma_\ell^2}{\sum_{j=1}^{r_\ell} \lambda_{\ell,j}}+\frac{\sigma_\ell}{\sqrt{\sum_{j=1}^{r_\ell} \lambda_{\ell,j}}} \right). 
\end{equation*}
Consequently, if one can design a sparse sensing matrix $\mathrm{S}_\ell \in \mathbb{R}^{\mathsf{s}_\ell \times p_\ell}$ with $\mathsf{s}_\ell \ll p_\ell$ such that $|a_{ij} - b_{ij}| = \mathrm{o}_{\prec}(\epsilon_\ell),$ the requirement on the SNR can be substantially relaxed. For instance, when the manifold is fixed and $\epsilon_\ell \asymp 1,$ it would suffice to assume $\sigma_\ell \ll \mathsf{s}_\ell^{-1/2}$ rather than the stronger original condition $\sigma_\ell \ll p_\ell^{-1/2}.$

 We illustrate in Figure \ref{fig:resultssketching} that the performance of our proposed method with sketching can be further enhanced when the data sets are noisier, highlighting the practical usefulness of the proposed ideas. In this example, we consider setup (1) in Section \ref{task_multiviewcluster} with $\mathsf{s}_1 = \mathsf{s}_2 = \mathsf{s}_3 = 12$, and take $\mathrm{S}_\ell$ to be a block of the $p_\ell \times p_\ell$ ($p_\ell=100$) Haar wavelet matrix constructed using the Python package PyWavelets. Specifically, we select the top 15 columns corresponding to the transformed components with the highest energy. A systematic investigation and theoretical analysis of the optimal selection of $\mathrm{S}_\ell$ and $\mathsf{s}_\ell$ are beyond the scope of the present paper and will be deferred to future work.
}

\section{Automated tuning parameter selection}\label{appendix_cchosen}
In this section, we discuss practical strategies for selecting various tuning parameters. To enhance the automation of our algorithm, we propose procedures for automatically determining the embedding dimension $m$ in (\ref{al_VCDM1}) and the percentiles $\omega_\ell$ used in (\ref{VMVL_bandwidth}). These choices are naturally associated with the larger eigenvalues of the corresponding matrices. In practice, our approach relies on scree plots of eigenvalues and the use of eigen-ratios, techniques that are widely employed in random matrix theory for identifying the number of spikes and signals \cite{9927456, 9779233}, consistent with the model reduction discussed in Section \ref{sec_robustnessofresults}. At a high level, these methods generalize the well-known elbow method \cite{Jolliffe2002PCA} based on scree plots in PCA.  

We first discuss how to select $m$ using the eigenvalues $\eta_2 \geq \eta_3 \geq \cdots$ of the matrix $\mathcal{A}$ in (\ref{eq_finaloperator}). Given a threshold $\mathsf{e}>0$, we define
\begin{equation*}
m := \operatorname*{arg min}_{\substack{2 \leq i \leq \sqrt{n}}} \frac{\eta_i}{\eta_{i+1}} \geq \mathsf{e}.
\end{equation*}
That is, we choose the embedding dimension so that the leading eigenvalues (corresponding to the signals) are large and well separated from the rest (corresponding to the noise). This ensures that the selected eigenvectors capture the essential information contained in the signal. The use of the argmin helps stabilize the numerical procedure. In our simulations, for practical purposes, we set
\begin{equation}\label{eq_deltadefinition}
\mathsf{e} = \frac{1}{2}\left(\operatorname*{max}_{\substack{2 \leq i \leq \sqrt{n}}} \frac{\eta_i}{\eta_{i+1}}- 1\right).
\end{equation}

Then we discuss how to choose $\omega_\ell$, following the approach in \cite{ding2023learning}. For a pre-selected sequence of percentiles $\{\omega_{\ell,i}\}_{i=1}^T,$ compute the associated $h_{\ell,i}$ according to (\ref{VMVL_bandwidth}), and denote them by $\{h_{\ell,i}\}_{i=1}^T.$ For each $1 \leq i \leq T$, calculate the eigenvalues of $\bK^\ell_{i}$, constructed using the bandwidth $h_{\ell,i}$ as in (\ref{eq_Kmultiplicationentry}). Denote these eigenvalues in decreasing order by ${\lambda_{\ell,k}^{(i)}}_{k=1}^n$. For some $\mathsf{s} > 0$ (for example, defined analogously to (\ref{eq_deltadefinition}) using the corresponding eigenvalues), define
\begin{equation*} \mathsf{k}(\omega_{\ell,i}):=\max_{1 \leq k \leq \sqrt{n}} \left\{ k \bigg \vert \frac{\lambda_{\ell,k}^{(i)}}{\lambda_{\ell,k+1}^{(i)}} \geq 1+\mathsf{s} \right\}. 
\end{equation*}
Choose the percentile $\omega$ such that
\begin{equation*}
\omega_\ell=\max_{i}[\operatorname{argmax}_{\substack{\omega_{\ell,i}}} \mathsf{k}(\omega_{\ell,i})].
 \end{equation*}
In particular, the final step can be replaced with alternative criteria depending on the user’s objectives. Here, we follow \cite{ding2023learning} and select the largest criterion to enhance robustness.

\section{Additional numerical results}\label{sec_additionalsimulationresults}

In this appendix, we present additional simulation results. Table \ref{tab_comparison_3} reports the comparison of various methods on the clustering task described in Section \ref{task_multiviewcluster}, evaluated using the Rand index. Moreover, Figure \ref{fig:eigenplots} present representative scree plots of the eigen-ratio and the associated eigenvector, illustrating the effectiveness of our procedure for selecting a meaningful value of $m$. As shown in the simulations of Section \ref{task_multiviewcluster}, our procedure (described in Appendix \ref{appendix_cchosen}) selects $m=1$, which is further confirmed in the plot. Finally, in Figure \ref{fig:resultssketching}, we demonstrate the potential benefits of incorporating sketching, as discussed in Appendix \ref{sec_practicalnoisereduction}.

\begin{table}[!htb]
\centering
\renewcommand{\arraystretch}{1.8}
\setlength{\tabcolsep}{6pt}
\resizebox{\textwidth}{!}{
\begin{tabular}{c|c|c|c|c|c|c|c|c|c|c|c|c} \hline
\textbf{Triplet / Method} & Proposed & mDM & mSNE & mLLE & mISOMAP & mUMAP & pADM & DM & LLE & tSNE & ISOMAP & UMAP \\
\cline{1-13}
\multicolumn{13}{c}{\bf Setup (1)} \\
\cline{1-13}
$(\upsilon_1^2,\upsilon_2^2,\upsilon_3^2)=(0.3,0.1,0.3)$ & 1 (0.02) & 0.97 (0.01) & 0.91 (0.02) & 0.88 (0.03) &  0.92 (0.04) & 0.9 (0.03) & 0.92 (0.03) & 0.93 (0.05) & 0.91 (0.05) & 0.9 (0.08) & 0.93 (0.02) & 0.91 (0.02) \\
$(\upsilon_1^2,\upsilon_2^2,\upsilon_3^2)=(3,2,3)$ & 0.98 (0.02) & 0.9 (0.03) & 0.71 (0.04) & 0.69 (0.01) & 0.65 (0.03) & 0.74 (0.07)  & 0.81 (0.02)  & 0.83 (0.05) & 0.78 (0.05) & 0.86 (0.02) & 0.93 (0.04) & 0.86 (0.06)\\
$(\upsilon_1^2,\upsilon_2^2,\upsilon_3^2)=(5,7,5)$ & 0.91 (0.03) & 0.87 (0.03) & 0.66 (0.05) & 0.67 (0.03) & 0.67 (0.02) & 0.73 (0.05)  & 0.75 (0.01)  & 0.83 (0.02) & 0.7 (0.05) & 0.8 (0.04) & 0.82 (0.03) & 0.84 (0.01)\\
$(\upsilon_1^2,\upsilon_2^2,\upsilon_3^2)=(10,10,10)$ & 0.85 (0.02) & 0.81 (0.03) & 0.68 (0.05) & 0.65 (0.02) & 0.65 (0.01) & 0.68 (0.02)  & 0.62 (0.01)  & 0.74 (0.01) & 0.63 (0.05) & 0.72 (0.01) & 0.8 (0.01) & 0.71 (0.05) \\ 
$(\upsilon_1^2,\upsilon_2^2,\upsilon_3^2)=(20,10,45)$ & 0.84 (0.02) & 0.72 (0.05) & 0.65 (0.03) & 0.64 (0.01) & 0.58 (0.01) & 0.63 (0.04)  & 0.53 (0.06)  & 0.66 (0.03) & 0.41 (0.07) & 0.59 (0.05) & 0.62 (0.06) & 0.6 (0.02) \\
$(\upsilon_1^2,\upsilon_2^2,\upsilon_3^2)=(10,25,40)$ & 0.79 (0.04) & 0.71 (0.08) & 0.64 (0.05) & 0.63 (0.01) & 0.65 (0.04) & 0.62 (0.03)  & 0.54 (0.04)  & 0.65 (0.06) & 0.39 (0.07) & 0.57 (0.05) & 0.68 (0.03) & 0.62 (0.07) \\
$(\upsilon_1^2,\upsilon_2^2,\upsilon_3^2)=(10,30,50)$ & 0.78 (0.04) & 0.71 (0.05) & 0.66 (0.05) & 0.62 (0.01) & 0.6 (0.05) & 0.65 (0.03)  & 0.52 (0.03)  & 0.64 (0.03) & 0.36 (0.01) & 0.57 (0.04) & 0.61 (0.05) & 0.51 (0.05) \\ \hline
\cline{1-13}
\multicolumn{13}{c}{\bf Setup (2)} \\
\cline{1-13}
$(\upsilon_1^2,\upsilon_2^2,\upsilon_3^2)=(0.5,0,0.3)$ & 1 (0.04) & 0.99 (0.03) & 0.91 (0.03) & 0.93 (0.04) &  0.9 (0.04) & 0.88 (0.04)  & 0.9 (0.01)  & 0.91 (0.03) & 0.9 (0.05) & 0.9 (0.02) & 0.89 (0.06) & 0.89 (0.05) \\
$(\upsilon_1^2,\upsilon_2^2,\upsilon_3^2)=(3,1,2)$ & 0.98 (0.01) & 0.96 (0.04) & 0.73 (0.02) & 0.67 (0.05) &  0.7 (0.03) & 0.71 (0.03)  & 0.8 (0.01)  & 0.76 (0.03) & 0.71 (0.05) & 0.75 (0.02) & 0.76 (0.06) & 0.82 (0.06) \\
$(\upsilon_1^2,\upsilon_2^2,\upsilon_3^2)=(9,5,4)$ & 0.93 (0.02) & 0.91 (0.05) & 0.61 (0.01) & 0.62 (0.03) &  0.62(0.04) & 0.68 (0.05)  & 0.68 (0.02)  & 0.75 (0.04) & 0.67 (0.05) & 0.72 (0.07) & 0.72 (0.03) & 0.78 (0.09) \\
$(\upsilon_1^2,\upsilon_2^2,\upsilon_3^2)=(10,12,8)$ & 0.88 (0.04) & 0.77 (0.02) & 0.65 (0.04) & 0.57 (0.07) & 0.6 (0.03) & 0.64 (0.03)  & 0.54 (0.03)  & 0.63 (0.02) & 0.64 (0.07) & 0.64 (0.08) & 0.65 (0.05) & 0.67 (0.02) \\
$(\upsilon_1^2,\upsilon_2^2,\upsilon_3^2)=(30,8,50)$ & 0.83 (0.04) & 0.74 (0.02) & 0.69 (0.04) & 0.59 (0.06) & 0.58 (0.05) & 0.63 (0.01)  & 0.48 (0.06)  & 0.67 (0.05) & 0.33 (0.05) & 0.51 (0.05) & 0.64 (0.04) & 0.55 (0.08) \\
$(\upsilon_1^2,\upsilon_2^2,\upsilon_3^2)=(50,15,70)$ & 0.78 (0.02) & 0.64 (0.03) & 0.55 (0.03) & 0.5 (0.04) & 0.42 (0.03) & 0.56 (0.04)  & 0.51 (0.02)  & 0.57 (0.01) & 0.37 (0.04) & 0.43 (0.05) & 0.57 (0.07) & 0.43 (0.06) \\
$(\upsilon_1^2,\upsilon_2^2,\upsilon_3^2)=(15,30,50)$ & 0.73 (0.01) & 0.65 (0.05) & 0.58 (0.03) & 0.54 (0.03) & 0.5 (0.03) & 0.6 (0.09)  & 0.5 (0.03)  & 0.63 (0.01) & 0.34 (0.02) & 0.51 (0.05) & 0.61 (0.02) & 0.53 (0.01) \\
\cline{1-13}
\end{tabular}
}
\caption{Comparison of Rand index across various methods. The reported values are the means with standard deviations (in parentheses), based on 1,000 replications.}
\vspace{0.5em}
\label{tab_comparison_3}
\end{table}

\begin{figure}[!htb]
    \centering
    \includegraphics[width=12cm, height=6cm]{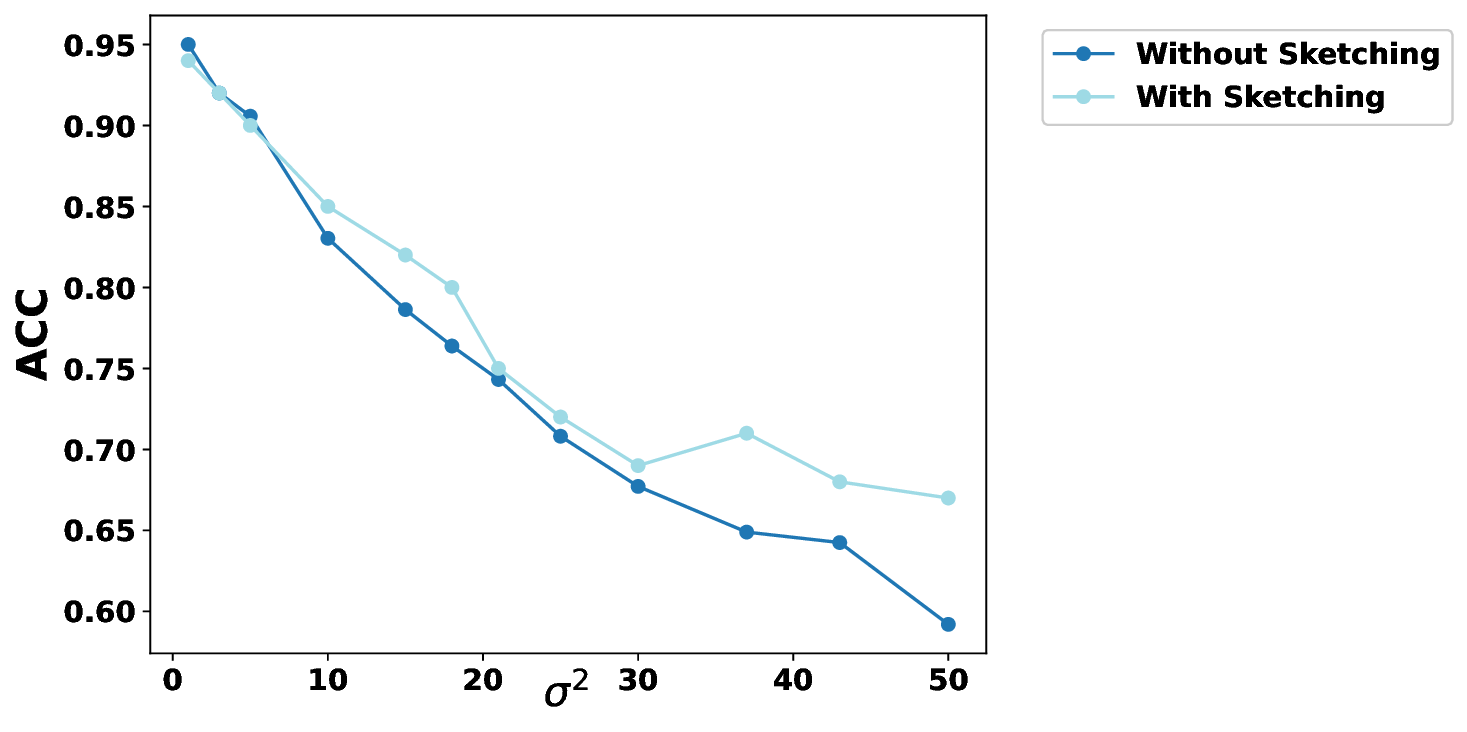}
    \caption{Comparison of the robustness of our proposed method with and without sketching. The detailed implementation is described in Appendix \ref{sec_practicalnoisereduction}, and the reported results are averaged over 1,000 repetitions.}
    \label{fig:resultssketching}
\end{figure}

\begin{figure}[!htb]
    \centering
    \begin{subfigure}[b]{0.45\textwidth}
        \centering
        \includegraphics[width=8cm, height=6cm]{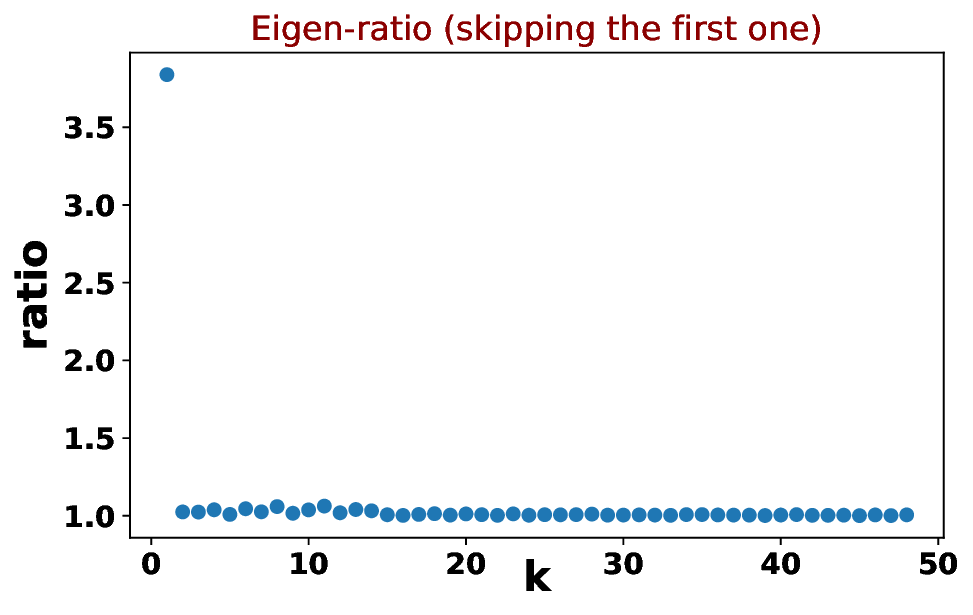}
        \caption{Eigen-ratio plot}
    \end{subfigure}
    \hfill
    \begin{subfigure}[b]{0.45\textwidth}
        \centering
        \includegraphics[width=8cm, height=5.7cm]{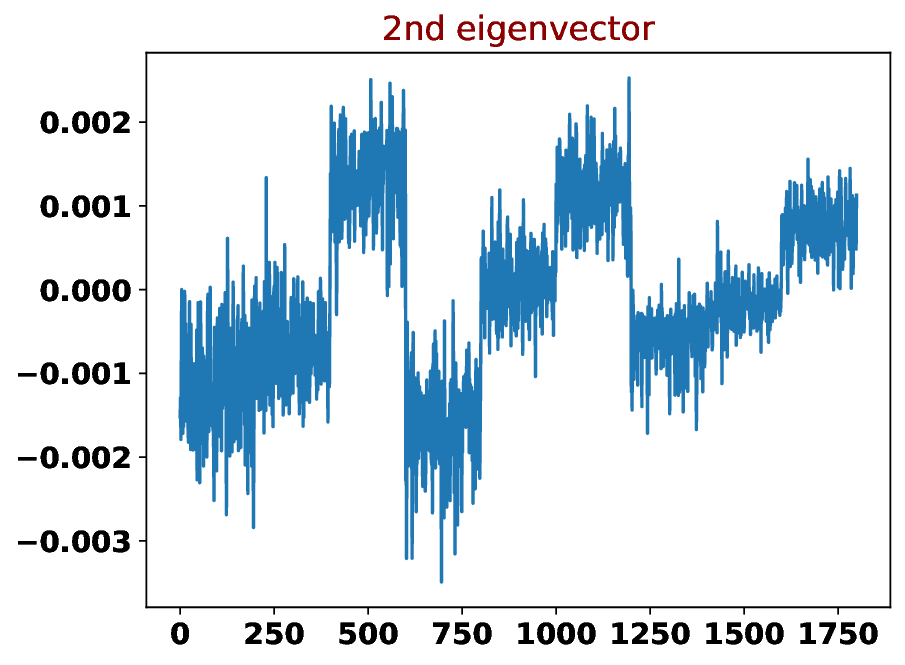}
        \caption{Second eigenvector}
            \hfill
            \end{subfigure}

    \begin{subfigure}[b]{0.45\textwidth}
        \centering
        \includegraphics[width=8cm, height=5.7cm]{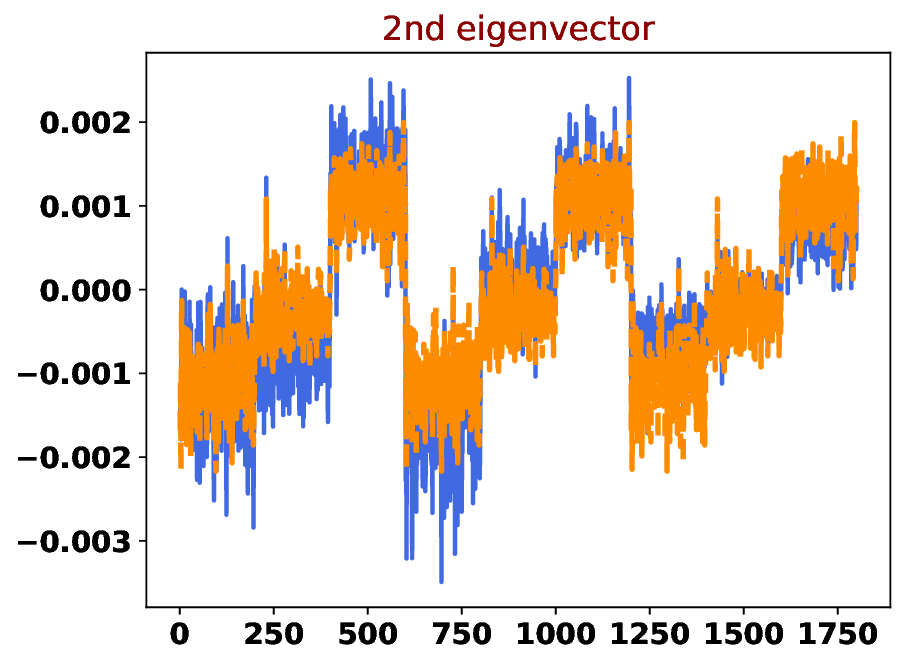}
        \caption{Second eigenvector and its average}
    \end{subfigure}
\hfill
    \caption{Eigen-ratio and eigenvector plots. We present a representative plot of the eigen-ratios and the second eigenvector of $\mathcal{A}$ in (\ref{eq_finaloperator}). The eigen-ratio plot indicates that $m=1$ is the appropriate choice, and accordingly only the second eigenvector and its average across three views are displayed. This also explains why our algorithm can efficiently identify the three distinct classes and why taking the average is beneficial.   The results are based on Setup (2) in Table \ref{tab_comparison_1}.}
    \label{fig:eigenplots}
\end{figure}

\bibliographystyle{plain}
\bibliography{ev}

\end{document}